%% file: main.tex
\documentclass[nohyperref]{article}

\usepackage{microtype}
\usepackage{graphicx}
\usepackage{subfigure}
\usepackage{booktabs} 
\usepackage{url}
\usepackage{algorithm}
\usepackage{algorithmic}
\usepackage{amsmath,amsfonts,bm}
\usepackage{amsthm,amssymb}
\usepackage{bbm}   
\usepackage{diagbox}
\usepackage{booktabs}  
\usepackage{color}
\usepackage{xcolor}
\definecolor{darkblue}{RGB}{25, 50, 112}
\usepackage[colorlinks=true, linkcolor=black, citecolor=darkblue]{hyperref}
\usepackage{float}
\usepackage{lipsum}  
\usepackage{multirow}
\usepackage{graphicx}  
\usepackage{caption}

\usepackage{enumerate}

\usepackage[accepted]{icml2022}

\usepackage{amsmath}
\usepackage{amssymb}
\usepackage{mathtools}
\usepackage{amsthm}

\usepackage{enumerate}
\usepackage[capitalize,noabbrev]{cleveref}  

\theoremstyle{plain}

\newtheorem{theorem}{Theorem}[]
\newtheorem{proposition}[theorem]{Proposition}
\newtheorem*{proposition*}{Proposition}

\theoremstyle{remark}

\newtheorem*{definition*}{Definition}

\DeclareMathOperator{\argmax}{arg\ max}

\usepackage[textsize=tiny]{todonotes}
\include{notations}

\icmltitlerunning{Generative Flow Networks for Discrete Probabilistic Modeling}

\begin{document}

\twocolumn[
\icmltitle{ Generative Flow Networks for Discrete Probabilistic Modeling }


\icmlsetsymbol{equal}{*}

\begin{icmlauthorlist}
\icmlauthor{Dinghuai Zhang}{milaudem}
\icmlauthor{Nikolay Malkin}{milaudem}
\icmlauthor{Zhen Liu}{milaudem}
\icmlauthor{Alexandra Volokhova}{milaudem}
\icmlauthor{Aaron Courville}{milaudem}
\icmlauthor{Yoshua Bengio}{milaudem}
\end{icmlauthorlist}

\icmlaffiliation{milaudem}{Mila - Quebec AI Institute and Universit\'e de Montr\'eal, Montreal, Quebec, Canada}

\icmlcorrespondingauthor{Dinghuai Zhang}{dinghuai.zhang@mila.quebec}

\icmlkeywords{Machine Learning, ICML}

\vskip 0.3in
]



\printAffiliationsAndNotice{}  

\begin{abstract}
We present energy-based generative flow networks (EB-GFN), a novel probabilistic modeling algorithm for high-dimensional discrete data. Building upon the theory of generative flow networks \citep[GFlowNets;][]{bengio2021foundations}, we model the generation process by a stochastic data construction policy and thus amortize expensive MCMC exploration into a fixed number of actions sampled from a GFlowNet.
We show how GFlowNets can approximately perform large-block Gibbs sampling to mix between modes. 
We propose a framework to
jointly train a GFlowNet with an energy function, so that the GFlowNet learns to sample from the energy distribution, while the energy learns with an approximate MLE objective with negative samples from the GFlowNet.
We demonstrate EB-GFN's effectiveness on various probabilistic modeling tasks.
Code is publicly available at 
\href{https://github.com/zdhNarsil/EB_GFN}{\tt github.com/zdhNarsil/EB\_GFN}.
\end{abstract}
\vspace{-4mm}

\section{Introduction}
\label{sec:introduction}

Probabilistic modeling in discrete spaces, especially those with compositional structure, is important due to the universality of applications of discrete data structures, such as in natural language processing \citep{Tai2015ImprovedSR} or in symbolic reasoning \citep{Besold2017NeuralSymbolicLA}.
However, distributions in high-dimensional discrete spaces are generally hard to model, as they may feature rapid combinatorial growth of modes. These modes can be well separated from each other, which poses a challenge for Markov Chain Monte Carlo (MCMC) methods \citep{Salakhutdinov2009LearningIM}.
Mixing between modes is generally slow without a priori knowledge of the specific latent structure of the distribution.

\begin{figure}[t]
\begin{minipage}{0.07\textwidth}
\hspace{0.cm}
\end{minipage}
\begin{minipage}{0.35\textwidth}
    \centering
    \includegraphics[width=\textwidth]{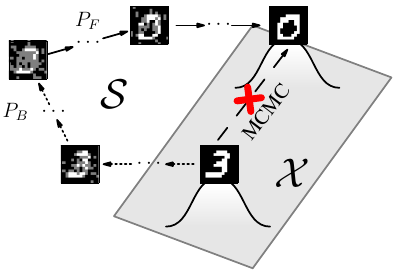}
    \end{minipage}
\caption{
    \textbf{The EB-GFN framework for learning an energy function on a discrete space $\gX$}, in this case the set of all binary images, 
    to maximize likelihood of a dataset 
    (\eg, MNIST). An energy landscape with well-separated modes  is difficult to explore with local search methods. A GFlowNet, parametrized with a pair of stochastic policies $P_F$ (painting) and $P_B$ (erasure), exploits a broader state space $\gS$ with partial choices, thus aiding the exploration of $\gX$ that is necessary for updating the energy function with contrastive-divergence-like objectives. 
    The trained GFlowNet is also a generative model on its own;
    see Figs.~\ref{fig:gfn_state_space}
     and \ref{fig:mnist}.
    }
    \label{fig:main_fig}
\vspace{-6mm}
\end{figure}

Furthermore, 
generative modeling methods such as energy-based models \citep[EBMs;][]{LeCun2006ATO} also suffer from these mixing issues when generating negative samples with MCMC
\citep{Tieleman2009UsingFW}.
The incapability of MCMC to capture the energy landscape results in the spurious mode problem \citep{Desjardins2010TemperedMC, Bengio2013BetterMV}, in which new modes that do not occur in the true data distribution appear in the learned energy distribution.

In this paper, we propose to take advantage of generative flow networks, or GFlowNets~\citep{bengio2021flow,bengio2021foundations}, instead of MCMC methods in order to simultaneously learn a sampler and train an energy function from data. GFlowNets are generative models for compositional objects, \ie, they learn a stochastic policy that iteratively constructs the sampled object through a sequence of simpler steps. In past work, GFlowNets were trained to query a given energy function, rather than from a dataset (like typical generative models). We propose to adapt the GFlowNet methodology to the problem of learning from data, and jointly train the GFlowNet sampler and the energy function.

Because a GFlowNet can {\em learn from the low-energy configurations it has already encountered}, it has a chance to guess and sample yet-unvisited modes if there are learnable regularities in the underlying data distribution. The compositional generative structure of GFlowNet further enhances its ability to discover regularities and thus jump between modes without having to travel long low-probability paths in between (Fig.~\ref{fig:main_fig}). Instead of having to exploit a priori structure to design long jumps in the MCMC, GFlowNets can discover that structure when it is not explicitly known.

This paper makes the following key contributions:
\vspace{-3mm}
\begin{enumerate}[(1)]
\item We cast a non-autoregressive sequential generation model for high-dimensional discrete data as a generative flow network (\S\ref{sec:gfn_setup}).\vspace{-2.5mm}
\item We introduce a GFlowNet-based MCMC proposal enabling efficient large jumps with low probability of rejection, taking advantage of the compositional structure learned by the GFlowNet generative policy (\S\ref{sec:joint_training}).\vspace{-2.5mm}
\item We describe a procedure based on such proposals for jointly training the GFlowNet sampler with an energy model given a dataset and state conditions under which this estimates the true data log-likelihood gradient with respect to the energy function’s parameters (\S\ref{sec:gfn_training}, \S\ref{sec:joint_training}).\vspace{-2.5mm}
\item We test the algorithm on a variety of synthetic and real tasks, achieving competitive results (\S\ref{sec:experiments}).
\end{enumerate}
\vspace{-3mm}



\vspace{-1.5mm}
\section{Preliminaries}
\vspace{-1.5mm}
\subsection{GFlowNets}
\vspace{-0.5mm}
\label{sec:prelim_gfn}

\begin{figure}[t]
    \centering
    \includegraphics[width=0.4\textwidth]{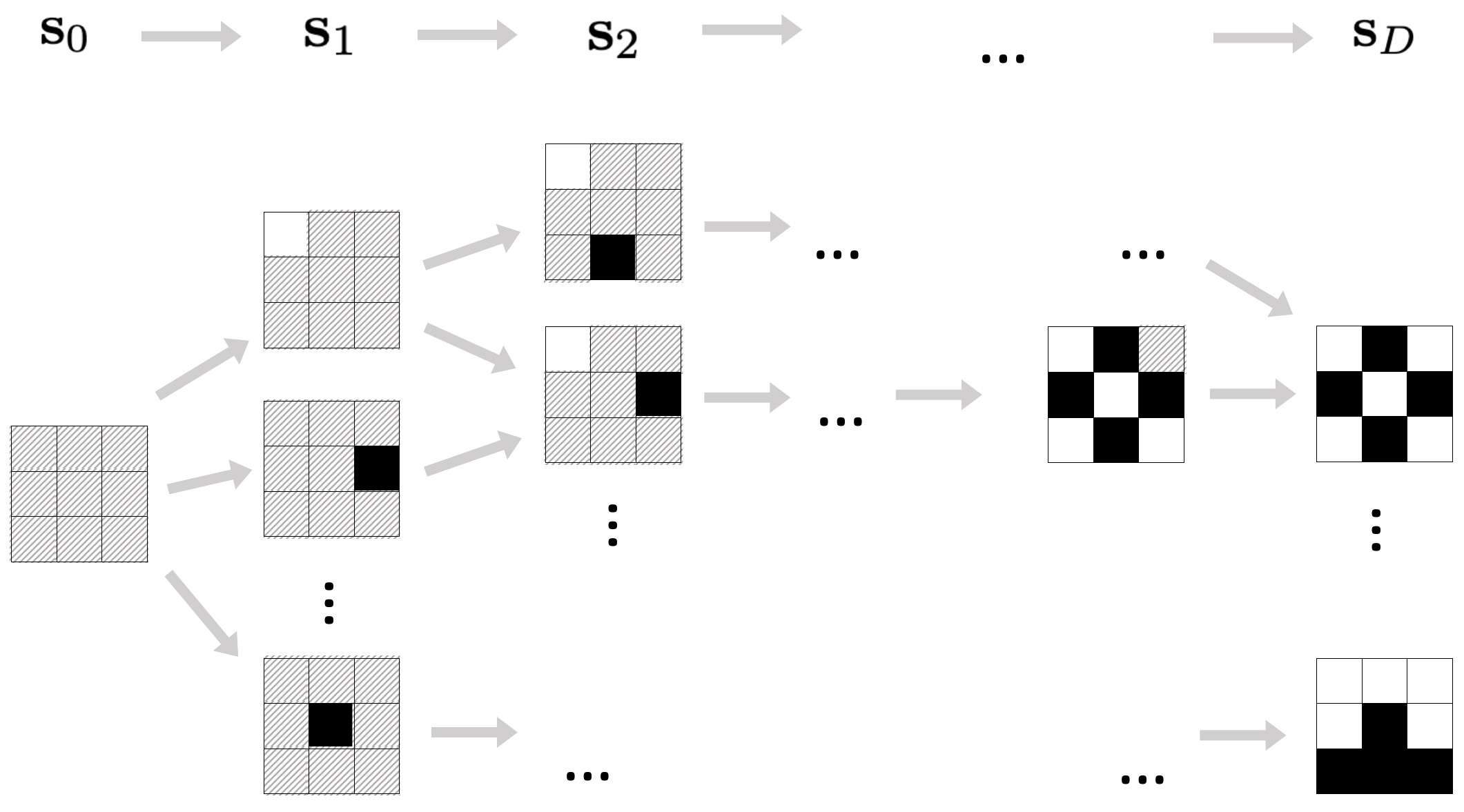}
    \vspace{-2mm}
    \caption{
    The state space $\gS$ and the GFlowNet's forward modeling process in a $9$-dimensional discrete data space.
    The states are the vertices of a DAG whose edges are the transitions -- actions of painting a grey pixel into black (1) or white (0). 
    }
    \label{fig:gfn_state_space}
\vspace{-5mm}
\end{figure}

Generative flow networks, or GFlowNets for short \citep{bengio2021flow,bengio2021foundations}, are trainable generative policies on which this paper is built. They model the generation process of objects $\x\in\gX$ by a sequence of discrete \emph{actions} that incrementally modify a partially constructed object (\emph{state}). The space of possible action sequences is represented by a directed acyclic graph (DAG, see Fig.~\ref{fig:gfn_state_space}) $G=(\gS,\gA)$, where the vertices in $\gS$ are states and the edges in $\gA$ are actions that modify one state to another. 
We use the terms \emph{parents} and \emph{children}, respectively, for the tails of incoming edges and the heads of outgoing edges of a state. Generation begins at a special \emph{initial state} $\s_0$ and terminates upon a transition to any \emph{terminal state}, which is a state with no outgoing actions. The set of terminal states is identified with the target space $\gX$. Note that multiple possible action sequences may lead to the same terminal state.

A \emph{complete trajectory} is a sequence of states $\s_0 \ra \s_1 \ra\dots\ra\s_n$, where each transition $\s_t\ra\s_{t+1}$ is an action in $\gA$ and $\s_n\in\gX$ is terminal. 
A \emph{trajectory flow} is a measure (unnormalized density) on the set of all complete trajectories $\gT$, \ie, a non-negative function $F:\gT\to\R_{\geq0}$, which can be thought of as a number of particles flowing from $\s_0$ to terminal states along each route. 
The flow is called \emph{Markovian} 
if there exist distributions $P_F(\cdot|\s)$ over the children of every non-terminal state $\s$, and a constant $Z$, such that for any complete trajectory $\tau=(\s_0\ra\s_1\ra\dots\ra\dots\ra\s_n)$, we have $P_F(\tau)=F(\tau)/Z$ with
\vspace{-1mm}
\begin{equation}
    P_F(\tau)=P_F(\s_1|\s_0)P_F(\s_2|\s_1)\dots P_F(\s_n|\s_{n-1}).
    \label{eq:markovian_factorization}
\vspace{-1mm}
\end{equation}
In this case, $P_F(\s_{t+1}|\s_t)$ is called a \emph{forward policy}, and can be used to sample complete trajectories from the density $F$ and thus also their terminal states, objects $\x\in\gX$.\footnote{It is helpful to think of the particle analogy. For a Markovian flow, the distribution over complete trajectories ($P_F(\tau)\propto F(\tau)$) satisfies a history-independence property: a particle's choice of route after reaching a state $\s$ is independent of how it reached $\s$. } 
We write $P_T(\x)$ for the probability that a trajectory sampled from $P_F$ terminates in $\x$.

Past work on GFlowNets \citep{bengio2021flow,bengio2021foundations,tbarxiv, Jain2022BiologicalSD} has considered the problem of fitting a Markovian flow to a fixed \emph{reward function} on $\gX$. Given a non-negative reward function $R:\gX\to\R_{\geq0}$, one seeks a Markovian flow $F$ such that the likelihood of a trajectory sampled from $F$ terminating in a given $\x\in\gX$ is proportional to $R(\x)$, \ie, $P_T(\x)\propto R(\x)$. This is accomplished by the reward matching condition:
\vspace{-1mm}
\begin{equation}
    R(\x)=\sum_{\tau=(\s_0\ra\ldots\ra\s_n), \s_n=\x} F(\tau).
    \label{eq:reward_matching}
\vspace{-1mm}
\end{equation}
By Eq.~(\ref{eq:markovian_factorization}), one can specify a Markovian flow by a learned scalar $\log Z_{\vtheta}$ and a neural net with parameters ${\vtheta}$ that outputs $P_F(\cdot|\s;{\vtheta})$ for any input state $\s$. Algorithms for learning $P_F$ from maximum-entropy reinforcement learning \citep{Haarnoja2017ReinforcementLW} would solve this problem if there was only one way to construct an object $\x$ (the DAG is a tree) while GFlowNets are applicable in the more general DAG setting \citep{bengio2021flow}. The solution originally proposed by \citet{bengio2021flow} recasts an \emph{unnormalized} forward policy as a network flow on $G$ (in the classical sense of \citet{ford-fulkerson}) and performs gradient descent on the error in a flow conservation constraint at states sampled from a training policy, amounting to a generalization of temporal difference objectives from \citet{Sutton2005LearningTP}.

Alternative objectives proposed by \citet{bengio2021foundations,tbarxiv} require a model to produce 3 outputs: the scalar $\log Z_{\vtheta}$, a forward policy $P_F(\cdot|\cdot;{\vtheta})$, and a \emph{backward policy} that produces distributions $P_B(\cdot|\s;{\vtheta})$ over the parents of an input state $\s$. (When action sequences are sampled in reverse from $P_B$, we will use dashed arrows $\s'\dra \s$ to indicate that the action $\s\ra\s'$ has been sampled \emph{against} the direction of the DAG edges. 
The $\vtheta$ will be dropped from $P_F$ and $P_B$ notation when it does not cause ambiguity.)

In modeling tasks with a fixed reward \citep{bengio2021flow,tbarxiv}, only the forward policy $P_F$ is used for sampling from the GFlowNet, and the backward policy $P_B$ is simply an training artifact. However, our approach makes use of the backward policy to perform local exploration in a contrastive divergence-like algorithm (\S\ref{sec:joint_training}, Figure~\ref{fig:main_fig}).






\vspace{-.5mm}
\subsection{Energy-based models}
\label{sec:prelim_ebm}
\vspace{-.5mm}

Energy-based models (EBMs) \citep{LeCun2006ATO, Song2021HowTT} are a popular approach for probabilistic inference and modeling.
An EBM specifies a distribution over a space $\gX$ by a density $p_{\vphi}(\x)=\frac{1}{Z_{\vphi}}\exp(-\gE_{\vphi}(\x))$, where $\gE_{\vphi}$ is the energy function, with model parameter $\vphi$, and $Z_{\vphi}$ is a normalizing constant independent of $\x$. The normalizing factor is not explicitly parametrized, but, for finite spaces, can theoretically be computed as $Z_{\vphi}=\sum_{\x\in\gX}\exp(-\gE_{\vphi}(\x))$. This summation is finite, and thus the energy function $\gE_{\vphi}(\x)$ can take an arbitrary form.
However, this flexibility comes with a price:
calculating $Z_{\vphi}$ can involve an exponentially large summation, 
making evaluation of the exact probability intractable. 
To avoid such expensive calculations, it is desirable to amortize sampling from such models by training a generative model.
 


We can train ${\cal E}_{\vphi}$ through maximum likelihood estimation (MLE), \ie\ seeking \ $\argmax_{\vphi}\E_{\x\sim p_{\text{data}}(\x)}\left[\log p_{\vphi}(\x)\right]$, 
where $p_{\text{data}}$ is the empirical distribution associated with the training data.
The gradient of negative log-likelihood (NLL) 
with respect to the model parameter ${\vphi}$  is given by
\vspace{-0.5mm}
\begin{align}
-\nabla_{\vphi} \log p_{\vphi}(\x) & =\nabla_{\vphi} \gE_{\vphi}(\x) + \nabla_{\vphi} \log Z_{\vphi}\nonumber\\
=
\nabla_{\vphi} &\gE_{\vphi}(\x)  - \E_{\x'\sim p_{\vphi}(\x')}[\nabla_{\vphi} \gE_{\vphi}(\x')].\label{eq:ebm_gradient}
\vspace{-0.5mm}
\end{align}
Evaluating or estimating the second term in (\ref{eq:ebm_gradient}) involves taking \textit{negative samples} $\x'$ from the EBM distribution, which can require an expensive search. 
The classical Contrastive Divergence (CD) algorithm \citep{Hinton2002TrainingPO} approximates this gradient update by changing the energy function parameter with the following stochastic approximation:
\vspace{-0.5mm}
\begin{align}
\label{eq:ebm_population_gradient}
\E_{\x\sim p_{\text{data}}(\x)}[\nabla_{\vphi} \gE_{\vphi}(\x) - \E_{\x'\sim q_{K}(\x'|\x)}\nabla_{\vphi} \gE_{\vphi}(\x')] 
\vspace{-0.5mm}
\end{align}
where $q_K(\x'|\x)$ is the distribution obtained by using a $K$-step MCMC initialized at $\x$ to \emph{approximately} sample from $p_{\vphi}$. For example, $\x'$ can be taken from a $K$-step Metropolis-Hastings chain starting at a true data sample $\x\sim p_{\rm data}(\x)$, and the parameters updated with $\nabla_{\vphi}(\gE_{\vphi}(\x)-\gE_{\vphi}(\x'))$.
As $K \to \infty$, assuming mixing of MCMC chains, 
the distribution over negative samples $q_K(\x'|\x)$ converges to $p_{\vphi}(\x')$, and we recover, in expectation, the true gradient (\ref{eq:ebm_gradient}).

\citet{Tieleman2008TrainingRB} later proposed persistent CD (PCD), where the MCMC chains that give negative samples do not restart at true data at every training step, but are initialized with a previous state.
While CD-type algorithms are efficient in computation and can accelerate learning, their gradient estimation is biased and thus may not model the true data distribution faithfully \citep{Nijkamp2020OnTA}.

\vspace{-1.5mm}
\section{Methodology}
\vspace{-1.5mm}

\begin{figure}[t]
    \centering
    \includegraphics[width=0.42\textwidth]{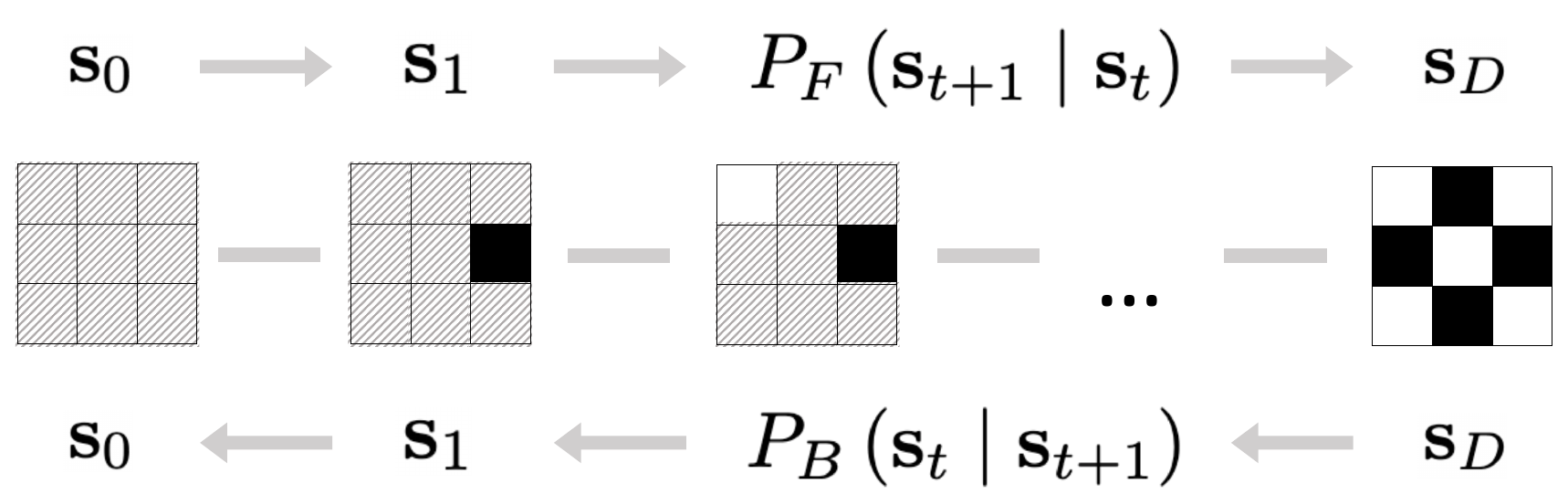}
    \vspace{-0.3cm}
    \caption{
    An illustration of the forward and backward GFlowNet policies in a $9$-dimensional discrete space of the kind studied here.
    The forward policy transforms a state $\s_t$ into $\s_{t+1}$, while the backward policy does the opposite operation.
    We represent $0 / 1$ with black / white patches, and use grey patches to denote unspecified entries $\oslash$ in incomplete (non-terminal) states.
    }
    \label{fig:gfn_pf_pb}
\vspace{-0.4cm}
\end{figure}

\subsection{GFlowNet generative process}
\vspace{-0.5mm}
\label{sec:gfn_setup} 


In this work, we aim to model a target distribution over discrete data with a GFlowNet.
In the domains we consider, the data is in $D$-dimensional binary space, \ie, $\x\in\gX\triangleq\{0, 1\}^D$.
As an example, $\x$ could be an image with $D$ pixels taking binary values.
We model the generation of vectors in $\gX$ by a GFlowNet. The state space $\gS$ of the GFlowNet consists of vectors of length $d$ with entries in $\{0,1,\oslash\}$, where the void symbol $\oslash$ represents a yet unspecified entry that may be turned to 0 or 1 by a future action. 
To be precise:
\vspace{-2mm}
\begin{equation}
\gS \triangleq \{(\s^1, \ldots, \s^D)\mid \s^d\in\{0, 1, \oslash\}, d=1,\ldots, D \}.
\vspace{-0.5mm}
\end{equation}
The DAG structure on $\gS$ is the $D$-th Cartesian power of the DAG with states $\{\oslash,0,1\}$, where 0 and 1 are children of $\oslash$. Concretely, the children of a state $\s=(\s^1,\dots,\s^D)$ are vectors that can be obtained from $\s$ by changing any one entry $\s^d$ from $\oslash$ to 0 or 1, and its parents are states that can be obtained by changing a single entry $\s^d\in\{0,1\}$ to $\oslash$.

We define $|\s| \triangleq \#\{\s^d\mid \s^d\in\{0, 1\}, d=1,\ldots, D\},$
the number of non-void entries in $\s$, so $\gX$ is naturally identified with $\{\s\in\gS:|\s|=D\}$.
There is an initial state 
$\s_0\triangleq (\oslash, \oslash, \ldots, \oslash)$.
Any trajectory from $\s_0$ to $\x\in\gX$ has exactly $D$ actions. A choice of trajectory from $\s_0$ to $\x$ amounts to a choice of the order in which the entries of $\x$ are assigned.

In this setting, the forward policy $P_F(\cdot|\s;{\vtheta})$ of a GFlowNet, introduced in \S\ref{sec:prelim_gfn}, is a distribution over all ways to select a position with a void entry in $\s$ and a value (0 or 1) to assign to this entry. Thus the action space for a state $\s$ has size \mbox{$2(D-|\s|)$}. Correspondingly, the backward policy $P_B(\cdot|\s;{\vtheta})$ is a distribution over the $|\s|$ ways to select a position with a nonvoid entry in $\s$. We illustrate the mechanism of the forward and backward policies in Figure~\ref{fig:gfn_pf_pb}.

In our experiments, we take $P_F(\cdot|\s;\vtheta)$ and $P_B(\cdot|\s;\vtheta)$ to be neural networks with a multilayer perceptron (MLP) architecture, where the input is a vector $\s\in\{\oslash,0,1\}^D$ encoded using a value of $-1$ for $\oslash$, and $P_F$ and $P_B$ share all weights except the final weight matrix that outputs logits for the forward and backward actions. The scalar $Z_{\vtheta}$ is parametrized in the log domain, as suggested by \citet{tbarxiv}. 
Specific implementation details are given in \S\ref{sec:experiments}.

\vspace{-0.5mm}
\subsection{GFlowNet training towards a target distribution}
\vspace{-0.5mm}
\label{sec:gfn_training}

Recall from \S\ref{sec:prelim_gfn} that, given a non-negative reward function $R:\gX\to\R_{\geq0}$, a GFlowNet can be trained so that its terminating probability distribution matches the reward distribution. To be precise, the marginal likelihood that a trajectory sampled from the GFlowNet's forward policy $P_F(\cdot|\cdot;{\vtheta})$ terminates at a given state is proportional to the state's reward, $P_T(\x) \propto R(\x)$. 
We now describe how 
GFlowNets could be trained towards matching a given reward.


\vspace{-0.3cm}
\paragraph{Trajectory balance.} To train the parameters ${\vtheta}$ of the GFlowNet, we use the trajectory balance objective proposed by \citet{tbarxiv}. Trajectory balance optimizes the following objective along complete trajectories $\tau=(\s_0\ra\s_1\ra\dots\ra\dots\ra\s_n)$:
\begin{equation}
    \gL_{\vtheta}(\tau)=\left[\log\frac{ Z_{\vtheta}\prod_{t=0}^{n-1}P_F(\s_{t+1}\mid s_t;\vtheta)}{ R(\s_n)\prod_{t=0}^{n-1}P_B(\s_t\mid \s_{t+1};\vtheta)}\right]^2.
    \label{eq:tb_objective}
\end{equation}
Proposition 1 of \citet{tbarxiv} shows that if this objective is globally minimized (\ie, zeroed out) for all complete trajectories $\tau$, then $P_T(\x)\propto R(\x)$, \ie, the forward policy samples proportionally to the reward. 
Trajectory balance improves training of GFlowNets under various metrics and characteristics of the reward landscape \citep{tbarxiv} relative to previously proposed objectives.

\vspace{-0.3cm}
\paragraph{Training policy.} 
With the trajectory balance objective, we train the GFlowNet with stochastic gradient
\vspace{-1mm}
\begin{equation}
\E_{\tau\sim\pi_{\vtheta}(\tau)}\left[\nabla_{\vtheta}\mathcal{L}_{\vtheta}(\tau)\right]
\label{eqn:tb_stochastic_gradient}
\vspace{-1mm}
\end{equation}
with some training trajectory distribution $\pi_{\vtheta}(\tau)$. Akin to on-policy RL settings, \citet{tbarxiv} took $\pi_{\vtheta}$ to be the distribution over trajectories sampled from the current policy $P_F(\cdot\mid\cdot;\vtheta)$, or a perturbed / tempered version of it. 
That is, $\tau$ is sampled with $\s_{t+1}\sim P_F(\cdot|\s_t;\vtheta)$ starting from $\s_0$, perhaps raised to a power or mixed with a uniform action policy to ensure $\pi_{\vtheta}$ has full support, 
which
is a condition for obtaining the desired distribution~\citep{tbarxiv}.


In addition to this forward sampling approach, we propose a complementary strategy to benefit from the circumstances where we are given some terminating states (data examples $\x\in\gX$). For a terminating state $\x$, one can sample a reverse trajectory $\tau=(\x=\s_D\dra \s_{D-1}\dra\dots\dra \s_0)$, where \mbox{$\s_t \sim P_B(\cdot|\s_{t+1};\vtheta)$}. Empirically, this backward trajectory sampling technique enables us to obtain a different trajectory distribution from the forward sampling distribution, as such backward trajectories visit regions of $\gS$ near the true data samples that may be poorly explored by $P_F$, and could thus stabilize the optimization.
Hereafter, we use $P_F(\tau)$ and $P_B(\tau|\x)$ to denote the trajectory distributions that sample forward from $\s_0$ using $P_F$ and backward from $\x$ using $P_B$, respectively.
In experiments, we take the training trajectory distribution $\pi_{\vtheta}$ to be a mixture of these two sampling methods (see steps 3-9 of Algorithm~\ref{alg:gfn_training}).



\vspace{-0.4cm}
\paragraph{Canonical design of $P_B$.} \citet{bengio2021foundations} noted that while there may be multiple Markovian flows satisfying (\ref{eq:reward_matching}), for any choice of a \emph{fixed} backward policy $P_B$, there is a unique forward policy $P_F$ such that the corresponding $P_T(\x)$ is proportional to the reward. \citet{tbarxiv} suggested fixing $P_B(\cdot|\s)$ to be uniform over the parents of every state $\s$ as a canonical choice. We find this to be beneficial in some domains.
Furthermore, in our setting, this choice also enforces a maximum-entropy property on the forward policy, as the following proposition shows.

\begin{definition*}
The \emph{entropy} of a Markovian flow $F$, denoted $\gH[F]$, is the expected total entropy of its forward policy distributions along a complete trajectory:\vspace{-1mm}
\begin{equation}
\vspace{-1mm}
    \gH[F]=\E_{(\s_0\ra\dots\ra \s_n)\sim P_F(\cdot|\cdot)}\left[\sum_{t=0}^{n-1}\gH[P_F(\cdot|\s_t)]\right].
    \label{eqn:entropy}
\end{equation}
\end{definition*}

\def\uniformprop{
Suppose $G$ is the DAG defined in \S\ref{sec:gfn_setup}. Let $R$ be a nonnegative reward function on $\gX$ and let $P_B^\circ$ be the uniform backward policy on $G$. Let $F^\circ$ be the Markovian flow uniquely determined by $P_B^\circ$ and $R$ subject to the reward matching constraint (\ref{eq:reward_matching}). Then $F^\circ$ has maximal entropy among all Markovian flows satisfying (\ref{eq:reward_matching}).
}
\begin{proposition}
\uniformprop
\label{prop:uniform}
\end{proposition}

\vspace{-0.4cm}
\paragraph{Estimating GFlowNet data likelihood.}
The most commonly used metric in probabilistic modeling is the model's likelihood on a test set.  
A well-trained model should assign a high likelihood to data from the same underlying distribution as the training data, that is,
the terminating probability distribution 
$P_T(\x) = \sum_{\tau=(\s_0\ra\ldots\ra\s_D), \s_D=\x} P_F(\tau) $ 
would be close to the true data distribution. 
We overcome the intractability of the sum defining $P_T(\x)$ (the number of terms is factorial in $D$) 
by importance sampling:
\vspace{-1mm}
\begin{equation}
P_T(\x) = \E_{ P_B(\tau\mid\x)}\frac{P_F(\tau)}{P_B(\tau|\x)}
\approx \frac1M\sum_{
j=1
}^M\frac{P_F(\tau^j)}{P_B(\tau^j|\x)},
\label{eqn:ll_estimate}
\vspace{-1mm}
\end{equation}
where $\tau^j\sim P_B(\tau|\x)$ are trajectories sampled backward from $\x$ using $P_B$. 
We can then use the average GFlowNet log likelihood on test set, estimated using (\ref{eqn:ll_estimate}) with $M$ large enough, as an evaluation metric.


\begin{algorithm}[t]
\begin{algorithmic}[1]
\INPUT Training dataset $\{\x_i\}_i$, hyperparameter $\alpha\in[0, 1]$
\STATE Initialize GFlowNet's $P_F,P_B,Z$ with parameters $\vtheta$, and the energy function $\gE_{\vphi}$ with parameters $\vphi$.
\REPEAT
\STATE $X\sim\rm{Bernoulli}(\alpha)$
\IF {$X=1$}{}
\STATE Sample forward trajectory $\tau\sim P_F(\tau)$.
\ELSE 
\STATE Uniformly sample $\x_i$ from dataset.
\STATE Sample backward trajectory $\tau\sim P_B(\tau|\x_i)$.
\ENDIF
\STATE Update the GFlowNet via gradient step on $\gL_{\vtheta}(\tau)$ with reward $R(\x)=e^{-\gE(\x;\vphi)}$ (Eq.~(\ref{eq:tb_objective})).
\STATE Update 
energy function 
with Algorithm~\ref{alg:ebm_training}.
\UNTIL 
some convergence condition
\end{algorithmic}
\caption{EB-GFN joint training framework}
\label{alg:gfn_training}
\vspace{-0.5mm}
\end{algorithm}

\vspace{-0.5mm}
\subsection{Interleaved updates of GFlowNet and energy}
\vspace{-0.5mm}
\label{sec:joint_training}

The training of GFlowNets relies on a given function $R(\x)$ to provide reward signals.
In generative modeling, we typically set $R(\x)$ to be the unnormalized target probability.
However, in many settings, we do not have access to this exact quantity, but only to a collection of data samples from a target distribution.

To address this issue, we propose to introduce an energy-based model $\gE_{\vphi}(\x)$ as an intermediate object between the data and the GFlowNet, to serve as the reward function with which the GFlowNet can be trained. The ``final products'' of training are then twofold: the trained energy model $\gE_{\vphi}(\x)$ and the GFlowNet sampling policy $P_F$.




We train the EBM associated with the GFlowNet in an approximate MLE manner similar to Eq.~(\ref{eq:ebm_population_gradient}), but using the GFlowNet 
policy
to generate negative examples, $\x'\sim P_T(\x')$. In the basic form of the energy function training procedure,
updates to $\vphi$ are made proportionally to 
\vspace{-1mm}
\begin{equation}
\E_{\x\sim p_{\text{data}}(\x)}\nabla_{\vphi} \gE_{\vphi}(\x) - \E_{\x'\sim P_{T}(\x')}\nabla_{\vphi} \gE_{\vphi}(\x'),
\label{eq:basic_ebm_gfn_update}
\vspace{-1mm}
\end{equation}
where the GFlowNet terminating probability distribution $P_T(\x)$ corresponds to marginalizing the
forward trajectory distribution $P_F(\tau)$ on its non-terminating states.
If the GFlowNet is perfectly trained (with zero training loss), its terminating probability distribution will be equal to the energy distribution, and thus this approximate MLE objective is an unbiased estimate of the maximum likelihood training of EBM.
In this way, we amortize the MCMC sampling computation into the GFlowNet training process.

\vspace{-0.3cm}
\paragraph{GFlowNet as an MCMC transition kernel.}
In contrastive divergence training of EBMs (\ref{eq:ebm_population_gradient}), $K$-step MCMC is used to generate the negative samples.
For each step of traditional MCMC methods, one first proposes a local random perturbation of the current sample $\x$ (seen as a state of the Markov chain) to some nearby point $\x'$, and then decides whether to accept this transition\footnote{This should be distinguished from the `transition' (action) in GFlowNets, which happens between two states $\s_t$ and $\s_{t+1}$ in $\gS$.} according to the Metropolis-Hastings (MH) rejection rule \citep{metropolis-hastings}. 



Here, we show that a GFlowNet can also be used for generating a proposal $\x'$ from a given point $\x$ (Fig.~\ref{fig:main_fig}). 
Given $\x$ and a fixed number of steps $K$ ($1\leq K\leq D$), we first sample a $K$-step trajectory with the backward policy $P_B$:
\vspace{-1mm}
\[
\tau=
(\x=\s_D\dra \s_{D-1}\dra \ldots\dra \s_{D-K})
, \s_{t}\sim P_B(\s_t|\s_{t+1}),
\vspace{-1mm}
\]
and then sample with the forward policy $P_F$, starting at $\s_{D-K}$ until a new terminal state $\x'$ is reached:
\vspace{-1mm}
\[
\tau'=(\s'_{D-K}\ra\ldots\ra \s'_D=\x'), \quad\s'_{t+1}\sim P_F(\s'_{t+1}|\s'_t),
\vspace{-1mm}
\]
where we have $\s'_{D-K}=\s_{D-K}$. For convenience, we denote this back-and-forth trajectory by $(\tau, \tau')$ and the reverse trajectory by
\vspace{-1mm}
\[
(\tau'_{-}, \tau_{-}) = (\s'_{D}\dra\ldots\dra \s'_{D-K} = \s_{D-K} \ra \ldots\ra \s_D).
\vspace{-1mm}
\]
Similar to $K$-step MCMC proposals, this GFlowNet proposal only changes the values of at most $K$ different entries.

\begin{algorithm}[t]
\begin{algorithmic}[1]
\INPUT Training dataset $\{\x_i\}_i$, GFlowNet providing $\{P_F,P_B,Z\}$, energy function $\gE_{\vphi}$, horizon $K$.
\STATE Uniformly sample $\x$ from dataset.
\STATE Sample a $K$-step backward trajectory from $P_B(\cdot|\cdot;\vtheta)$: \mbox{$\tau=(\x=\s_D\dra\s_{D-1}\dra\dots\dra\s_{D-K})$}.
\STATE Sample a $K$-step forward trajectory from $P_F(\cdot|\cdot;\vtheta)$: \mbox{$\tau'=(\s_{D-K}\ra\s_{D-K+1}'\ra\dots\ra\s_D'=\x')$}.
\STATE Accept or reject $\x'$ via Eq. (\ref{eq:mh_prob}); set $\x'\,\la\,\x$ if reject.
\STATE Update $\vphi$ with gradient of $\gE_{\vphi}(\x)-\gE_{\vphi}(\x')$.
\end{algorithmic}
\caption{GFlowNet-guided energy function update}
\label{alg:ebm_training}
\vspace{-0.5mm}
\end{algorithm}

With $\tau,\tau'$ as above, we extend the definitions 
for complete trajectories (Eq.~\ref{eq:markovian_factorization}) to
$P_F(\tau)=\prod_{t=D-K}^{D-1}P_F(\s_{t+1}|\s_t)$ and $P_B(\tau|\x) = \prod_{t=D-K}^{D-1}P_B(\s_t|\s_{t+1})$. The probability of a transition from $\x$ to $\x'$ along the back-and-forth trajectory $(\tau, \tau')$ is $P_B(\tau|\x)P_F(\tau')$. Similarly, the probability of going $\x'$ to $\x$ along the reverse trajectory $(\tau_{-}', \tau_{-})$ is $P_B(\tau'|\x')P_F(\tau)$. 
With the MH rule, if the move from $\x$ to $\x'$ is accepted with probability
\vspace{-2mm}
\begin{align}
\hspace{-4mm}
A_{\tau,\tau'}(\x\to\x')\triangleq\min\left[1, \frac{e^{-\gE_{\vphi}(\x')}}{e^{-\gE_{\vphi}(\x)}}\frac{P_B(\tau|\x)P_F(\tau')}{P_B(\tau'|\x')P_F(\tau)}\right]\hspace{-1mm},
\label{eq:mh_prob}
\vspace{-3mm}
\end{align}
then the stationary distribution of an iterated application of such steps is equal to the desired reward distribution.

The following proposition shows that with a perfectly trained GFlowNet, we can cheaply obtain a high-dimensional form of Gibbs sampling, where $K$ variables are updated at each step\footnote{Vanilla block Gibbs sampling would require computation exponential in $K$ in order to compute $2^K$ possible $K$-bit changes, with their energies and their normalizing constant.}:
\def\mhstepprop{If a GFlowNet fits the reward perfectly, \ie, satisfies (\ref{eq:reward_matching}), then $A_{\tau,\tau'}(\x\to\x')=1$, so the MH rejection step will always accept the proposal.}
\begin{proposition}
\label{prop:mh_step_noneffective}
\mhstepprop
\end{proposition}
\vspace{-2mm}
The proof is in \S\ref{app:proofs}. 
Notice that if $K=D$, the proposed transition is equivalent to directly sampling from $\s_{D-K}=\s_0$ with $P_F$, which is independent of the backward policy and of the starting sample $\x$.

As a relaxation of (\ref{eq:basic_ebm_gfn_update}) that can accelerate learning, we propose to generate negative samples using this back-and-forth GFlowNet transition proposal. 
Unlike MCMC methods, we do not iterate this kernel, but perform only one single
step to generate negative samples $\x'$ for each EBM parameter update, similarly to one-step contrastive divergence. 
As with MCMC-based contrastive divergence, it may be beneficial to begin with a small $K$ and gradually increase it over the course of training. In all of our experiments, we either use a constant $K=D$ (unconditional samples from the GFlowNet are used as negative examples, as in Eq.~(\ref{eq:basic_ebm_gfn_update})) or gradually increase $K$ from 1 to $D$. We summarize our use of the GFlowNet proposal in EBM training in Algorithm~\ref{alg:ebm_training}.

\vspace{-3mm}
\paragraph{Summary.} 

We propose a joint training framework (Algorithm~\ref{alg:gfn_training}), 
where the EBM and the GFlowNet are optimized alternately:
the energy function serves as the (negative log-) reward function for the GFlowNet, which is trained with the trajectory balance objective to sample from the evolving energy model, while the energy function is trained with an approximate MLE gradient, where the GFlowNet provides negative samples $\x'$ through an MCMC transition proposal that approximates 
block
Gibbs sampling. 


\begin{table}[t]
\vspace{-4mm}
    \centering
        \caption{Mean negative log-RMSE (higher is better) between data-generating matrix $J$ and learned matrix $J_{\vphi}$ for different values of $\sigma$. We find standard deviation $<0.1$ between runs for each setting.}
    \label{tab:ising_results}
    \resizebox{\linewidth}{!}{
   \begin{tabular}{lccccccc}
\toprule
 & \multicolumn{5}{c}{$D=10^2$} & \multicolumn{2}{c}{$D=9^2$} \\
\cmidrule(lr){2-6}\cmidrule(lr){7-8}
 Method $\backslash$ $\sigma$ & $0.1$ & $0.2$ & $0.3$ & $0.4$ & $0.5$ & $-0.1$ & $-0.2$  \\ 
 \midrule
Gibbs & $4.8$ & $4.7$ & $\bf 3.4$ & $\bf 2.6$ & $\bf 2.3$ & $4.8$ & $4.7$ \\
GWG & $4.8$ & $4.7$ & $\bf 3.4$ & $\bf 2.6$ & $\bf 2.3$ & $4.8$ & $4.7$ \\
EB-GFN & $\bf 6.1$ & $\bf 5.1$ & $3.3$ & $\bf 2.6$ & $\bf 2.3$  & $\bf 5.7$ & $\bf 5.1$ \\
\bottomrule
    \end{tabular}
    }
\vspace{-5mm}
\end{table}

\vspace{-2mm}
\section{Experiments}
\vspace{-1mm}
\label{sec:experiments}

\subsection{Ising models}
\vspace{-0.5mm}

\begin{figure}[t]
    \centering
    \includegraphics[width=0.49\linewidth,trim=0 200 0 8,clip]{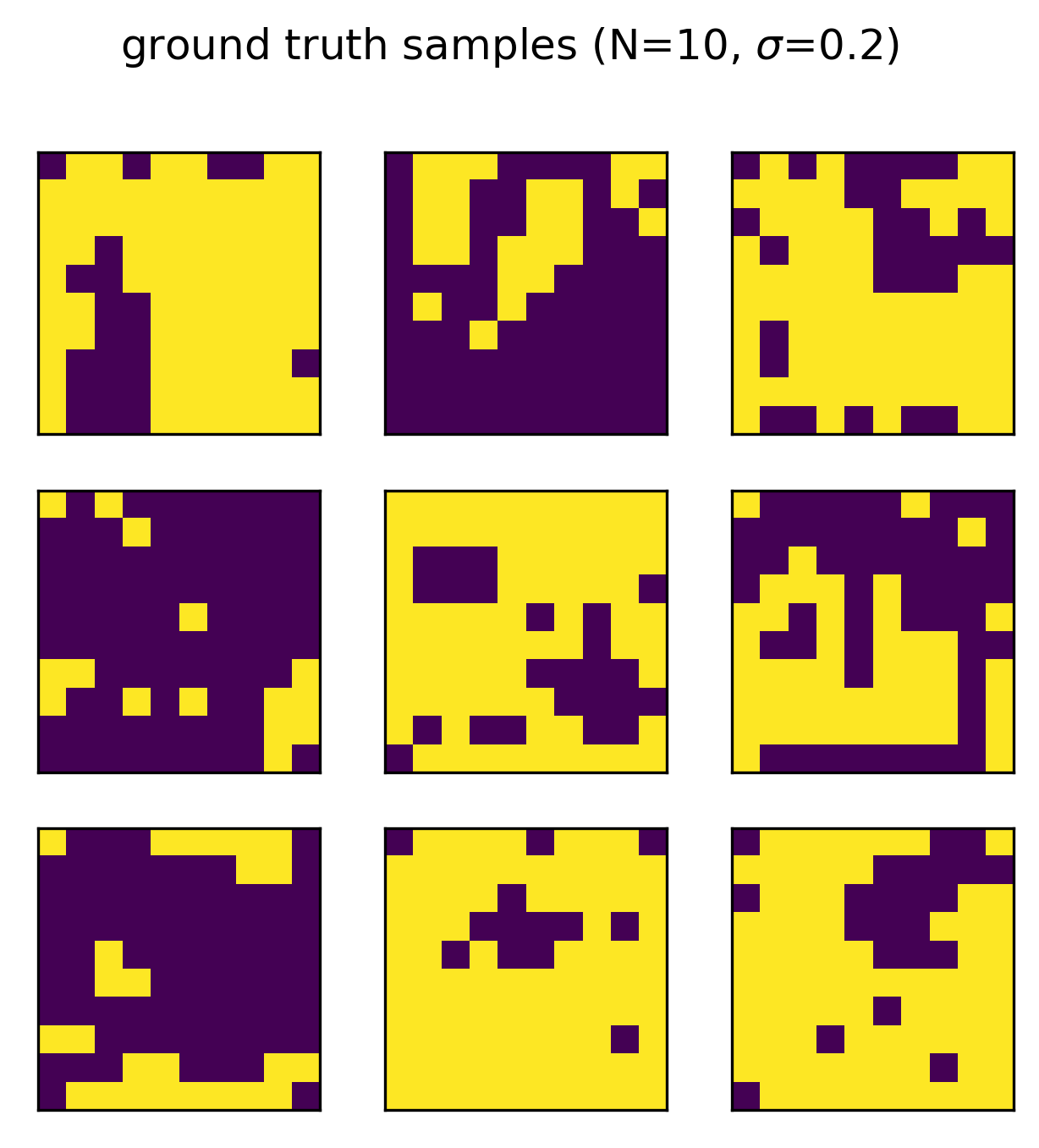}
    \hfill
    \includegraphics[width=0.49\linewidth,trim=0 200 0 8,clip]{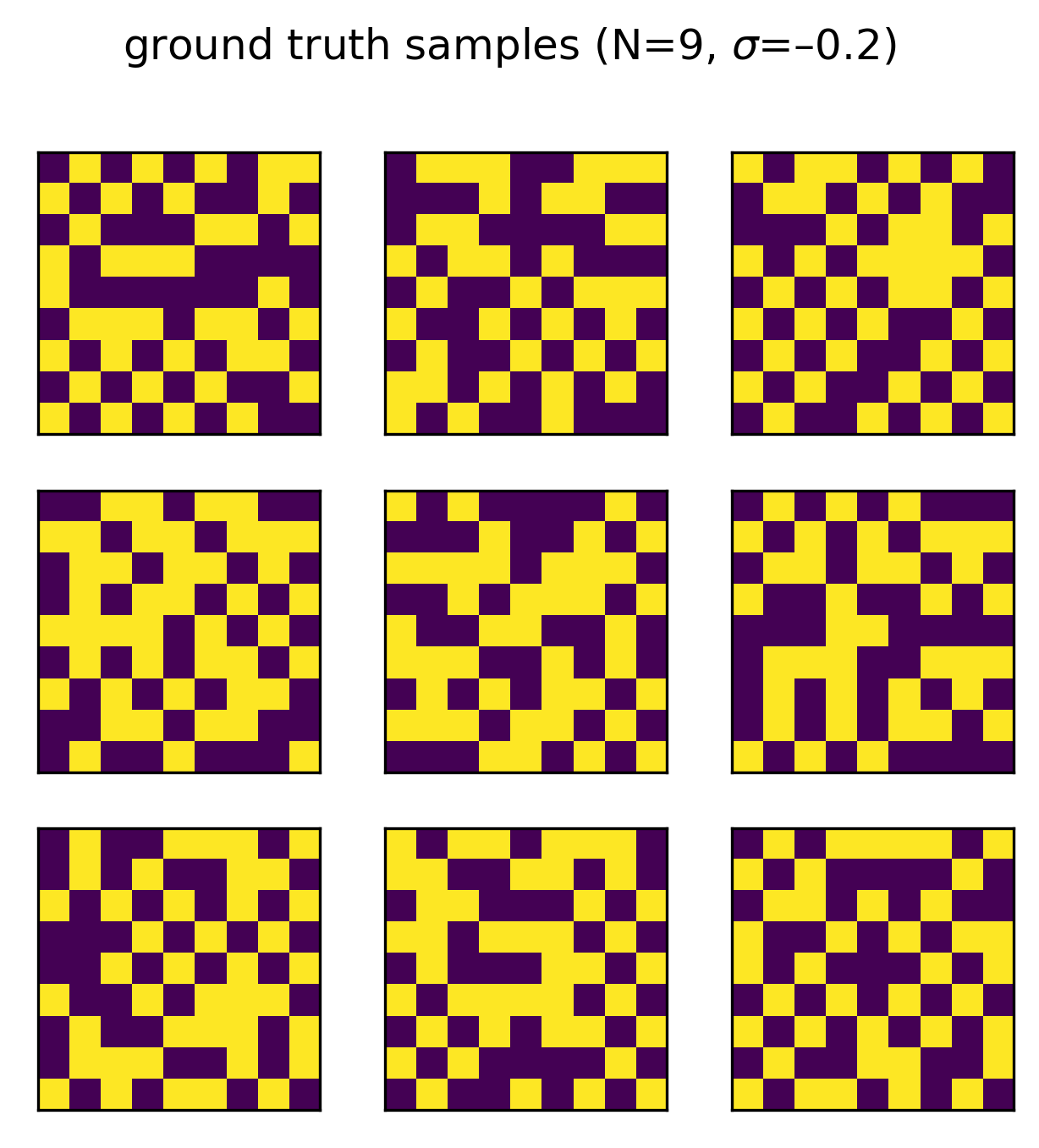}
    \\[0.2cm]
    \includegraphics[width=0.49\linewidth,trim=0 200 0 0,clip]{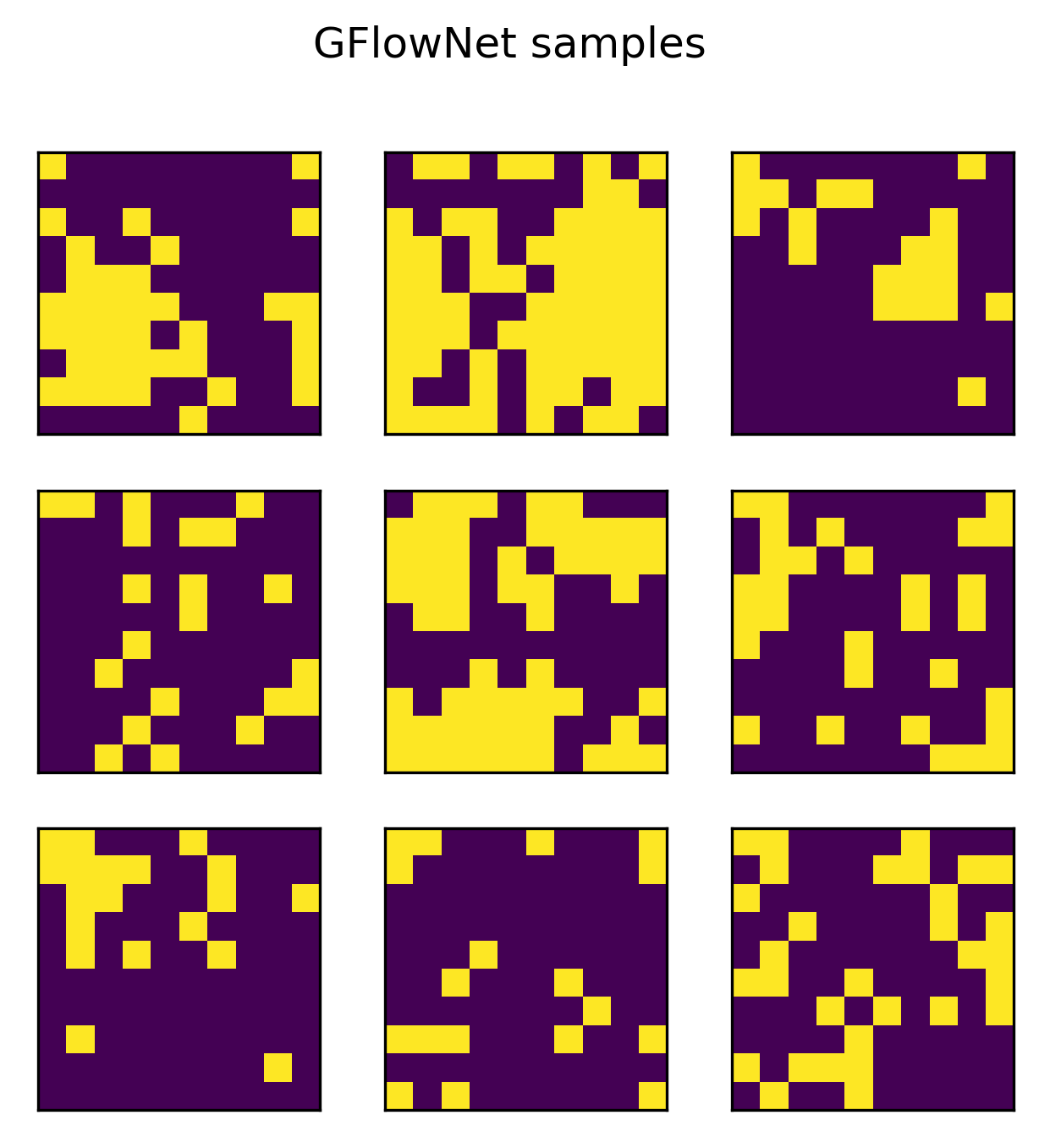}
    \hfill
    \includegraphics[width=0.49\linewidth,trim=0 200 0 0,clip]{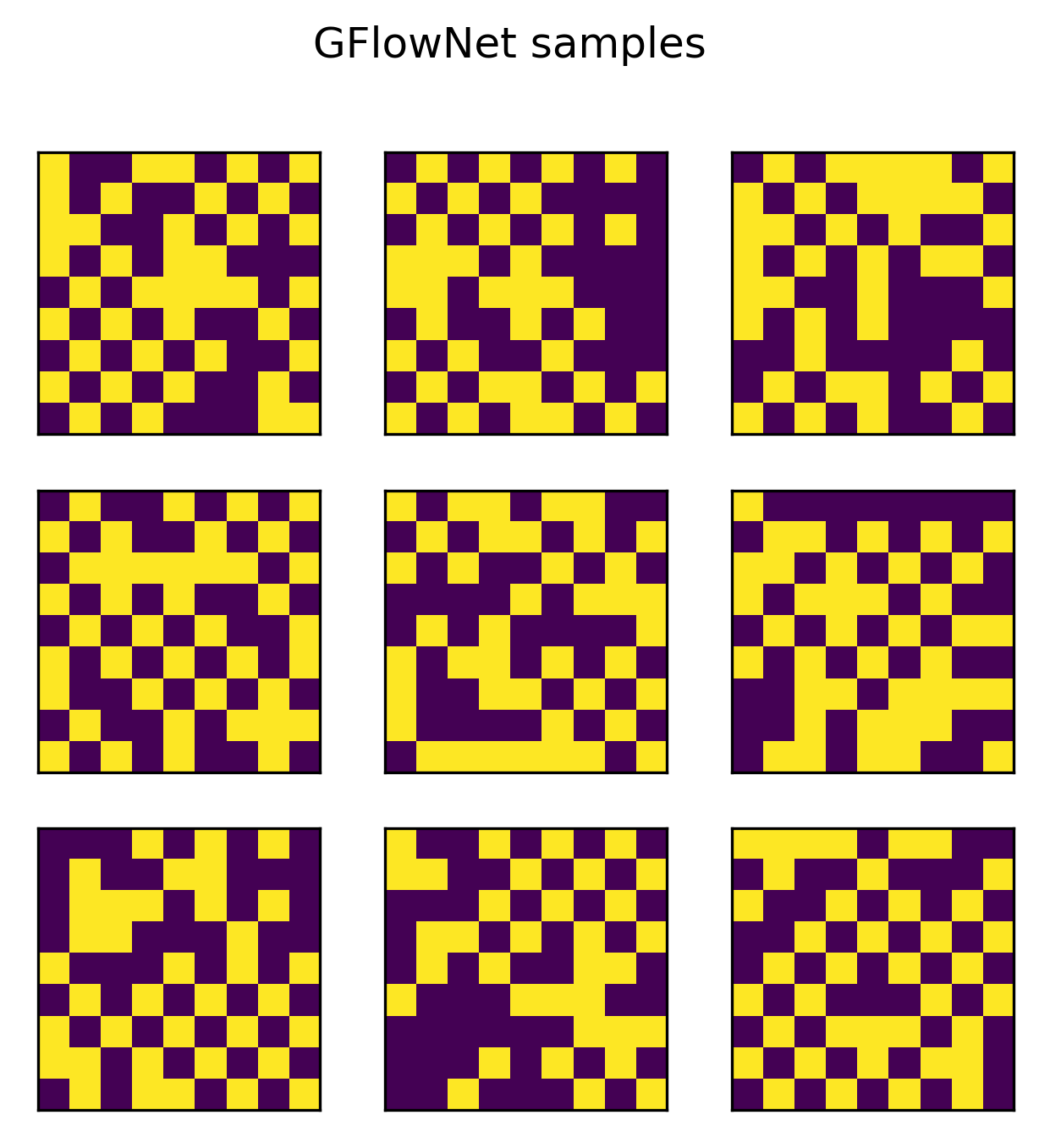}
    \\
    \includegraphics[width=0.49\linewidth,trim=20 40 20 20,clip]{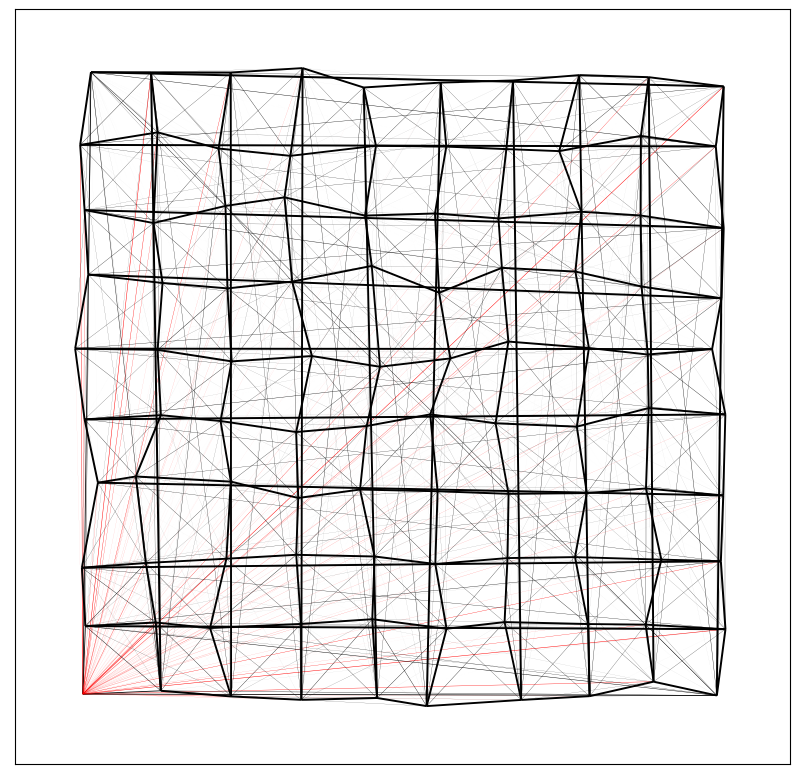}
    \hfill
    \includegraphics[width=0.49\linewidth,trim=20 40 20 20,clip]{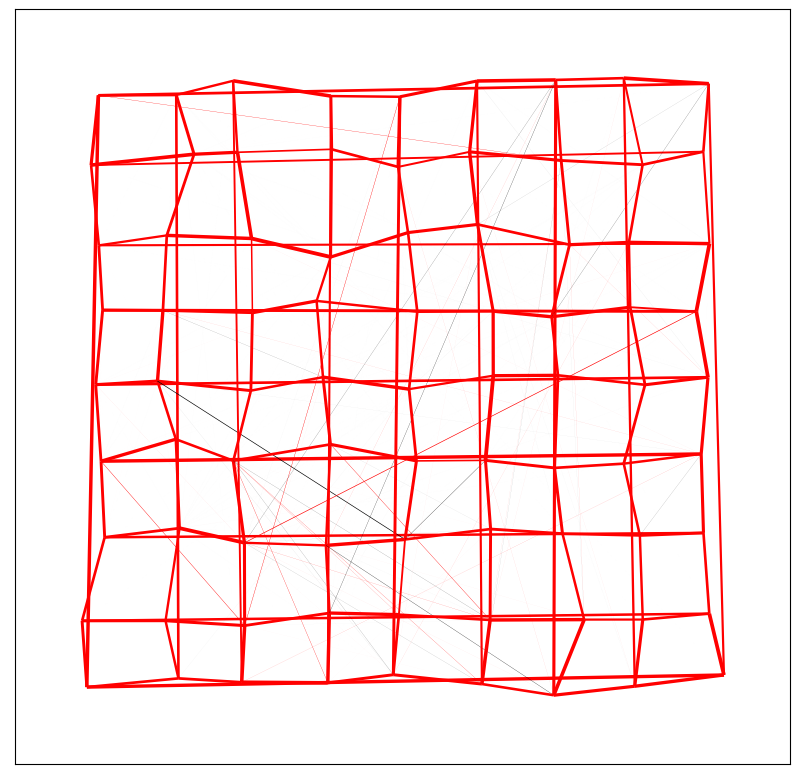}
    \\
    \vspace{-4mm}
    \caption{
    \textit{Top:} Samples from an Ising model and samples from a GFlowNet trained jointly with an energy function on 2000 such samples. The generated samples display similar characteristics to the ground truth samples: contiguous regions with the same spin for $\sigma>0$ and checkerboard texture for $\sigma<0$.
    \textit{Bottom:} A visualization of the learned matrix $J_{\vphi}$ as a graph, where the thickness of the edge between nodes $i$ and $j$ is proportional to $(J_{\vphi})_{ij}$ and red/black edges represent negative/positive entries. Despite EB-GFN having no a priori knowledge of the grid structure, it almost perfectly recovers the full $D\times D$ matrix -- 4950 (left) or 4851 (right) degrees of freedom -- from 2000 discrete samples.
    }
    \label{fig:ising_vis}
    \vspace{-6mm}
\end{figure}

\newcommand\mpwid{0.16}
\newcommand\hinterval{0.5cm}
\newcommand\samplempwid{0.146}
\newcommand\samplehinterval{0.23cm}
\newcommand\samplefigwid{\textwidth}
\begin{figure*}[t]
\centering
\begin{minipage}{\textwidth}
    \begin{minipage}[t]{\samplempwid\textwidth}
    \centering
    \small{2spirals}\\
    \includegraphics[width=\samplefigwid,
    trim=35 35 10 10,clip
    ]{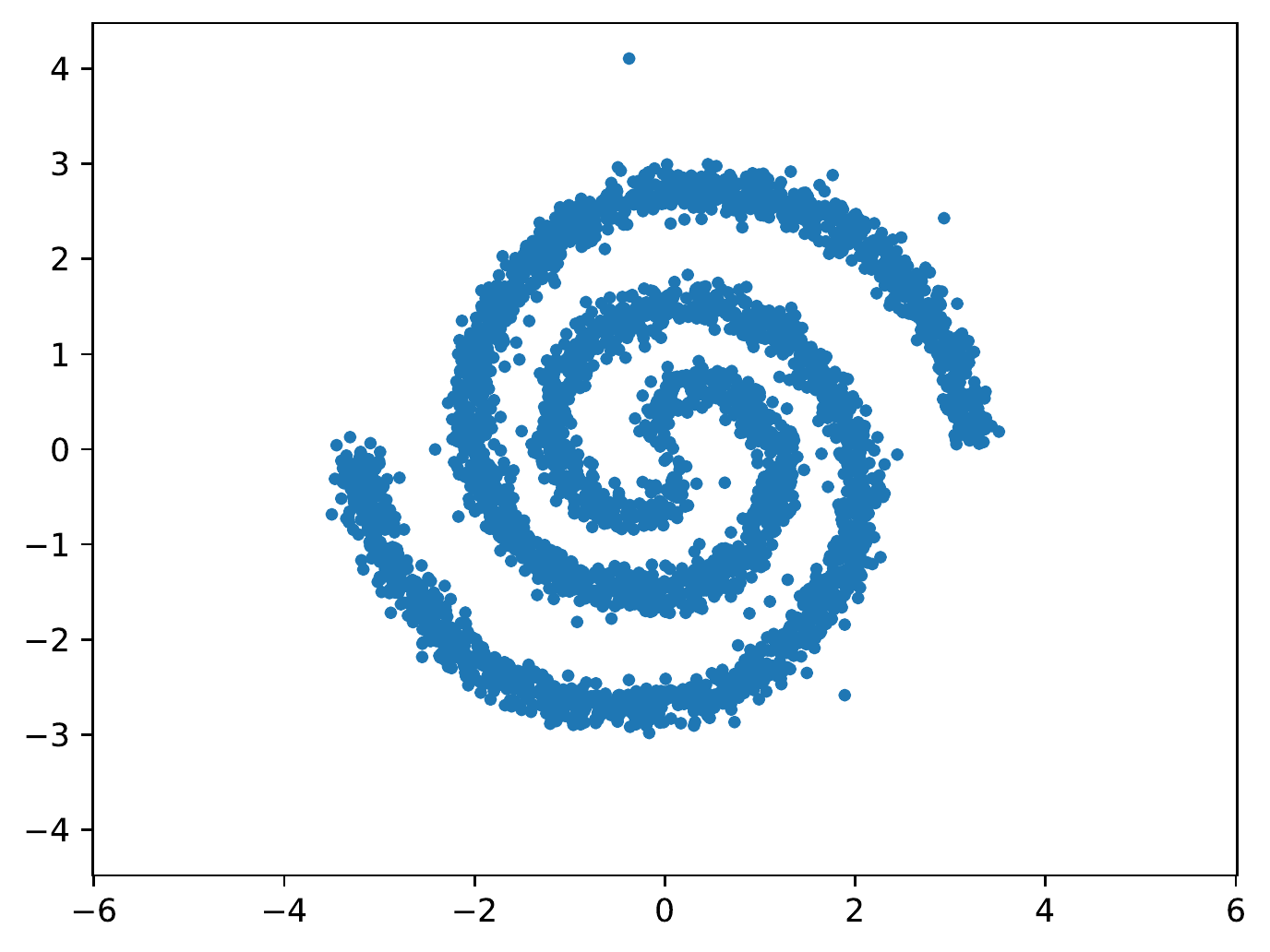}
    \end{minipage}
\hspace{-\samplehinterval}
    \begin{minipage}[t]{\samplempwid\textwidth}
    \centering
    \small{8gaussians}\\
    \includegraphics[width=\samplefigwid
    ,trim=35 35 10 10,clip
    ]{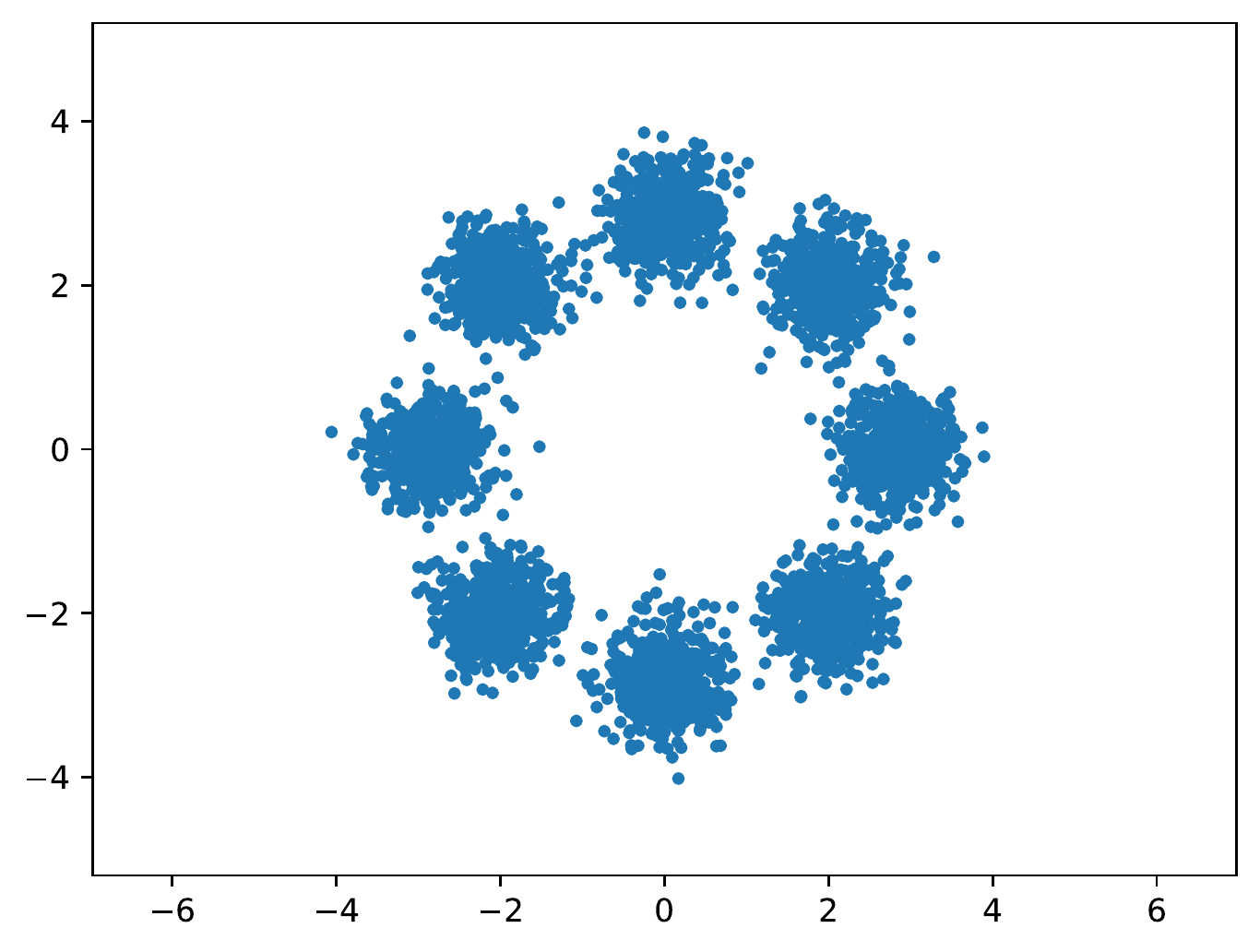}
    \end{minipage}
\hspace{-\samplehinterval}
    \begin{minipage}[t]{\samplempwid\textwidth}
    \centering
    \small{circles}\\
    \includegraphics[width=\samplefigwid
    ,trim=35 35 10 10,clip
    ]{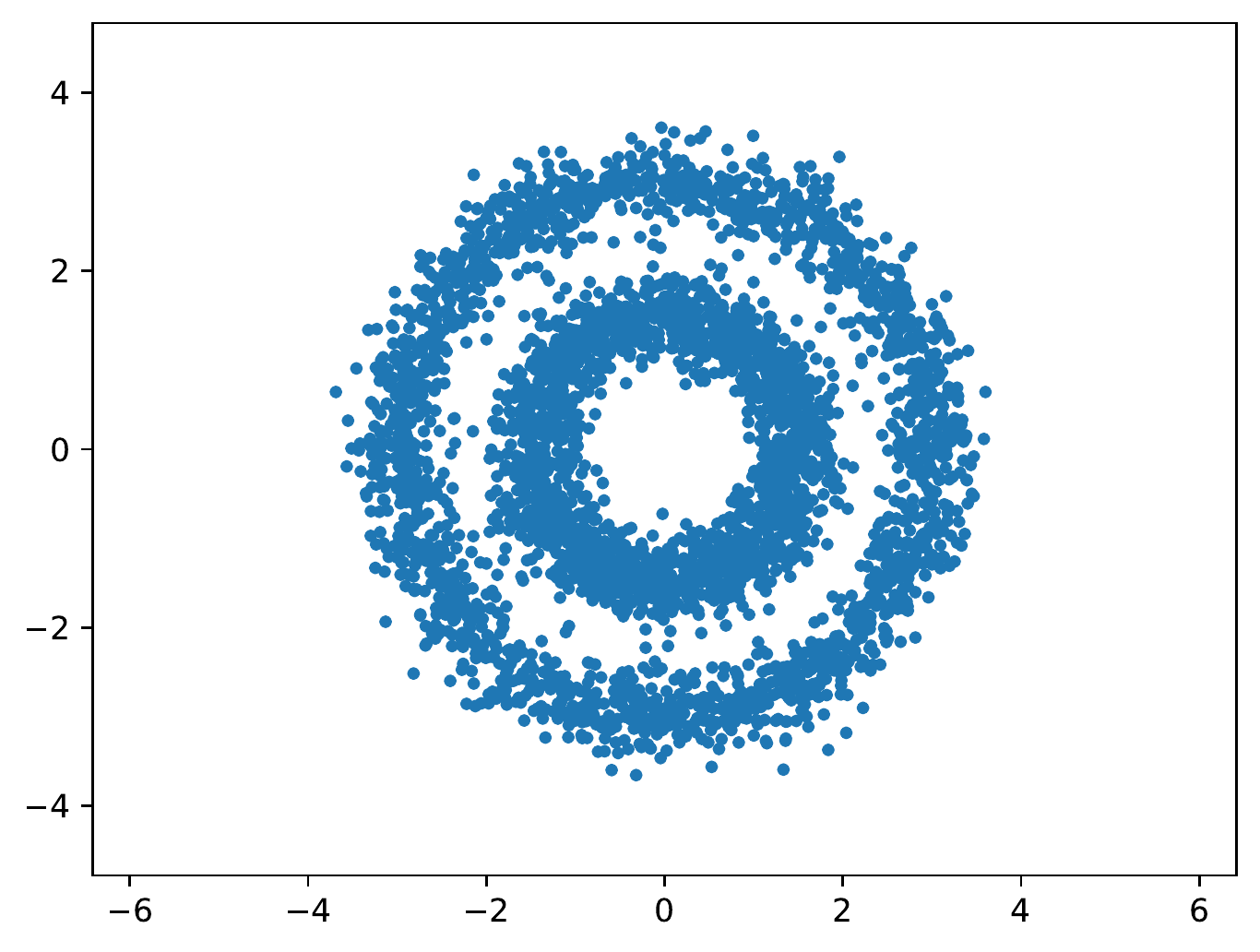}
    \end{minipage}
\hspace{-\samplehinterval}
    \begin{minipage}[t]{\samplempwid\textwidth}
    \centering
    \small{moons}\\
    \includegraphics[width=\samplefigwid
    ,trim=35 35 10 10,clip
    ]{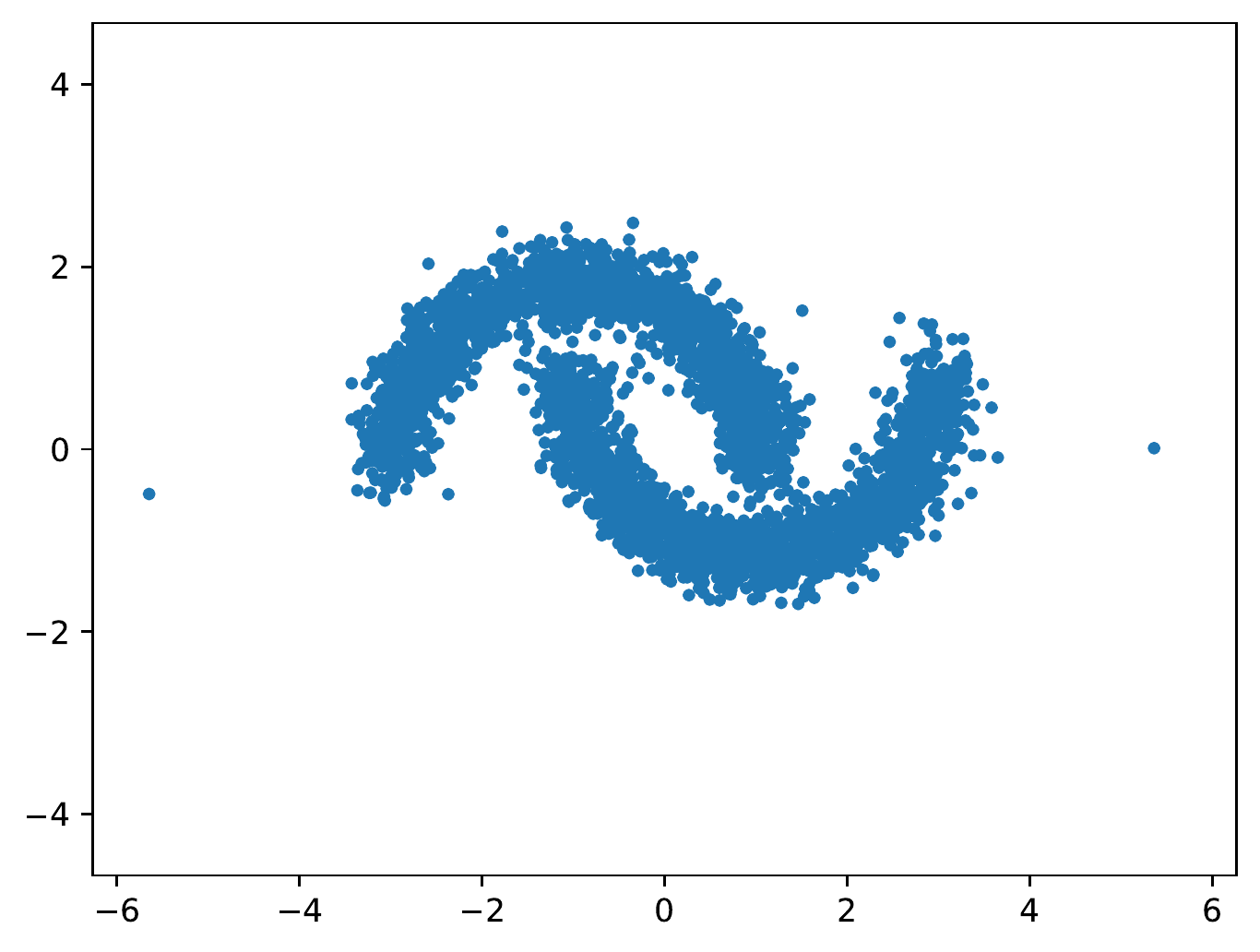}
    \end{minipage}
\hspace{-\samplehinterval}
    \begin{minipage}[t]{\samplempwid\textwidth}
    \centering
    \small{pinwheel}\\
    \includegraphics[width=\samplefigwid
    ,trim=35 35 10 10,clip
    ]{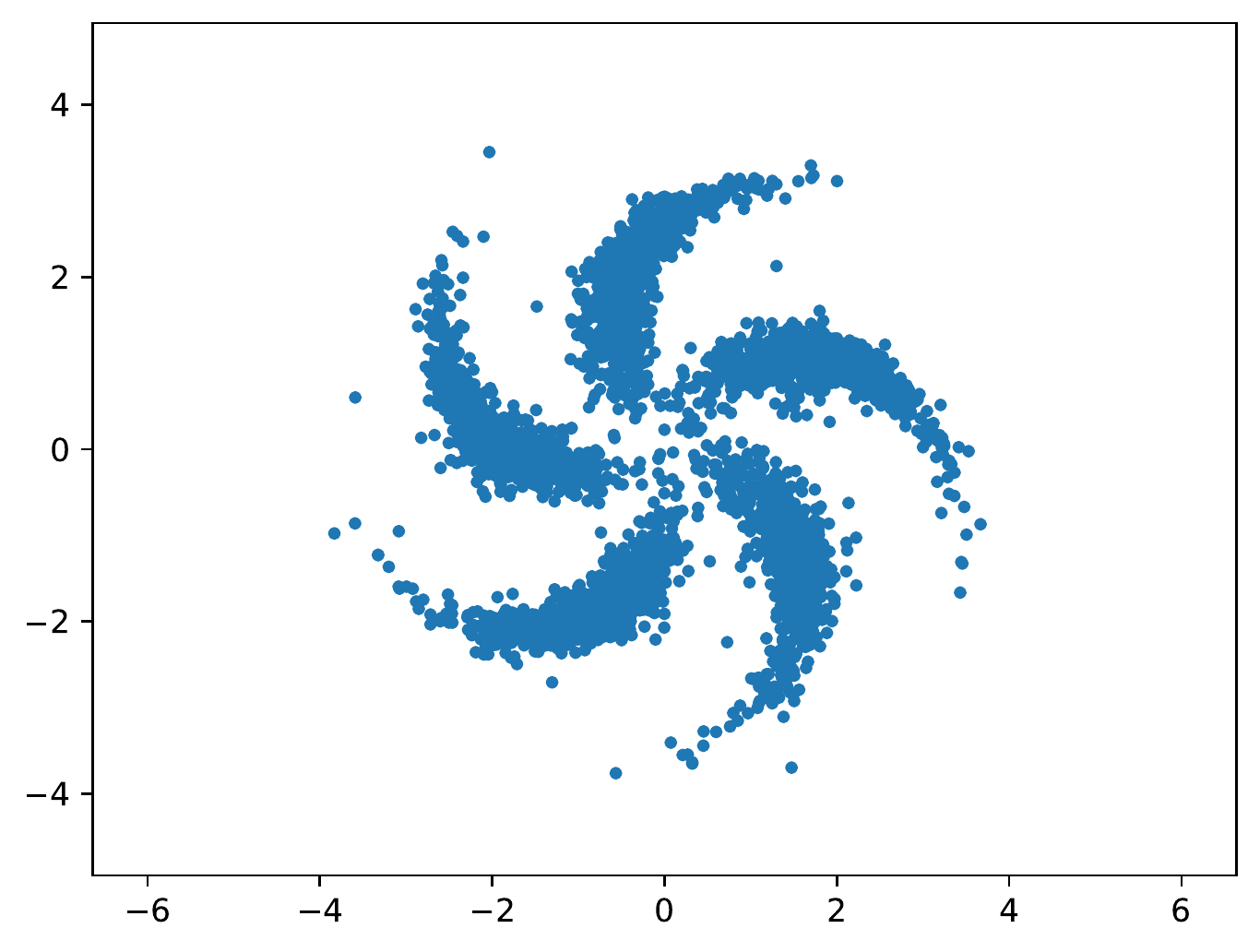}
    \end{minipage}
\hspace{-\samplehinterval}
    \begin{minipage}[t]{\samplempwid\textwidth}
    \centering
    \small{swissroll}\\
    \includegraphics[width=\samplefigwid
    ,trim=35 35 10 10,clip
    ]{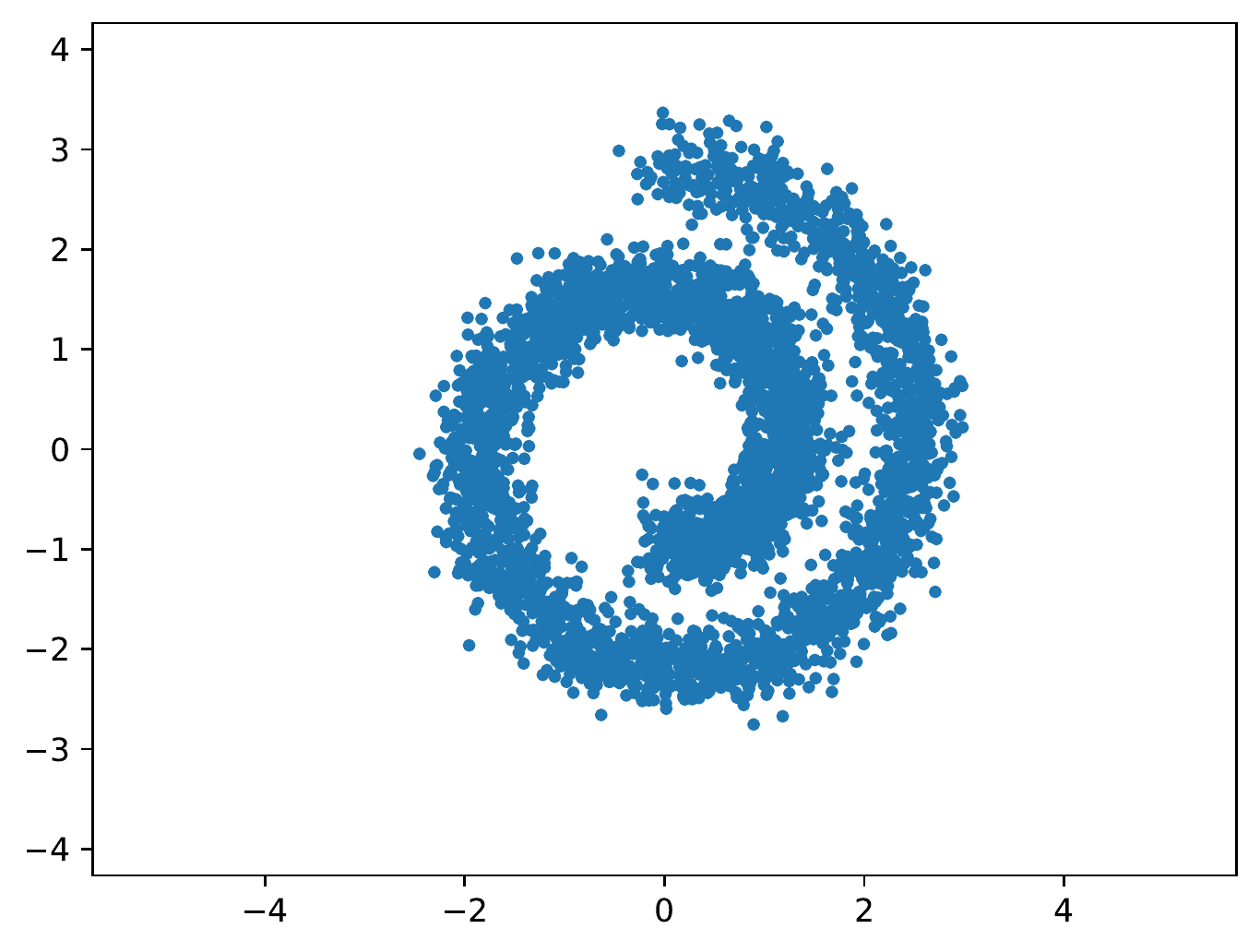}
    \end{minipage}
\hspace{-\samplehinterval}
    \begin{minipage}[t]{\samplempwid\textwidth}
    \centering
    \small{checkerboard}\\
    \includegraphics[width=\samplefigwid
    ,trim=35 35 10 10,clip
    ]{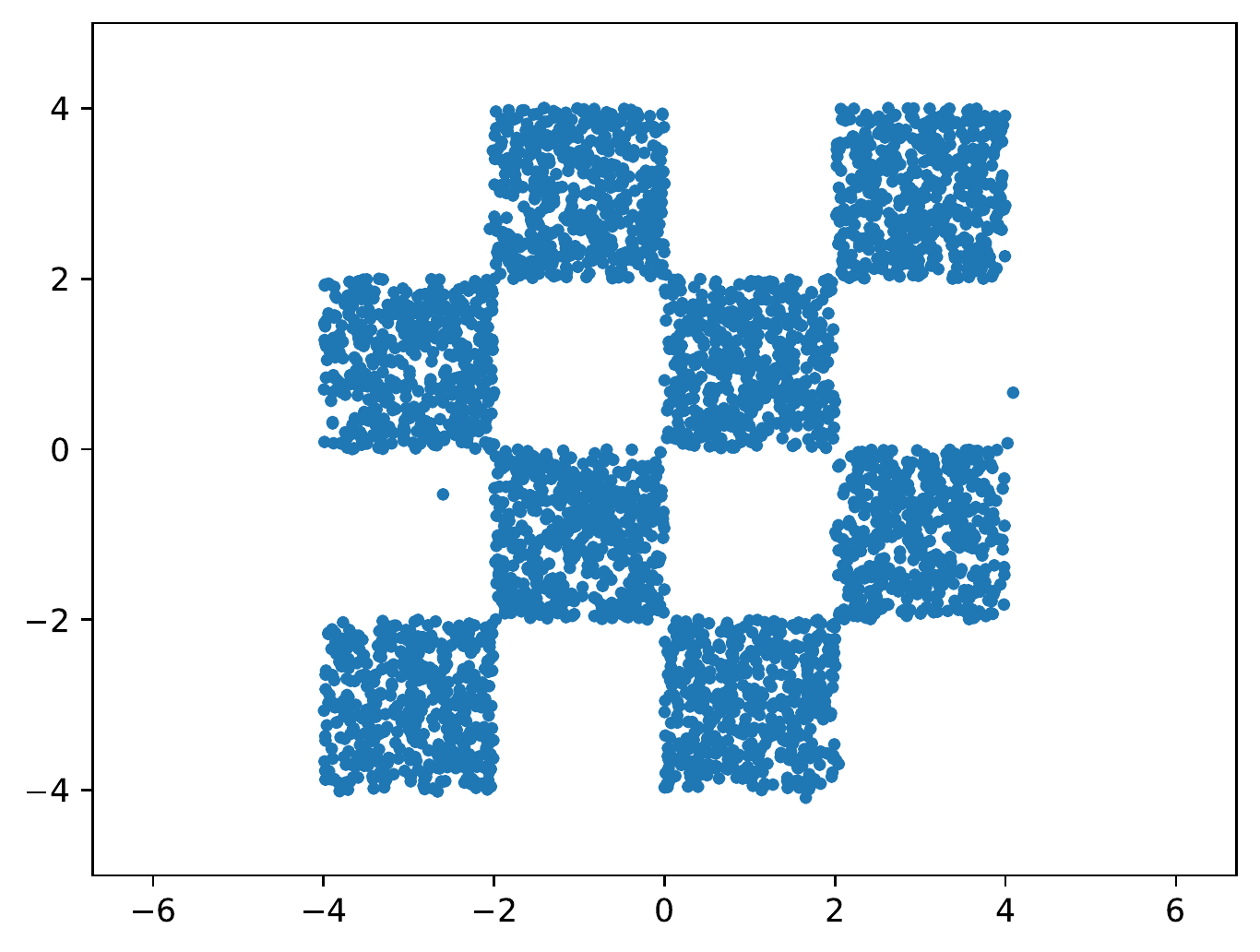}
    \end{minipage}
\hspace{-\samplehinterval}
\end{minipage}
\begin{minipage}{\textwidth}
    \begin{minipage}{\mpwid\textwidth}
    \centering
    \includegraphics[width=\textwidth,trim=0 30 0 30,clip
    ]{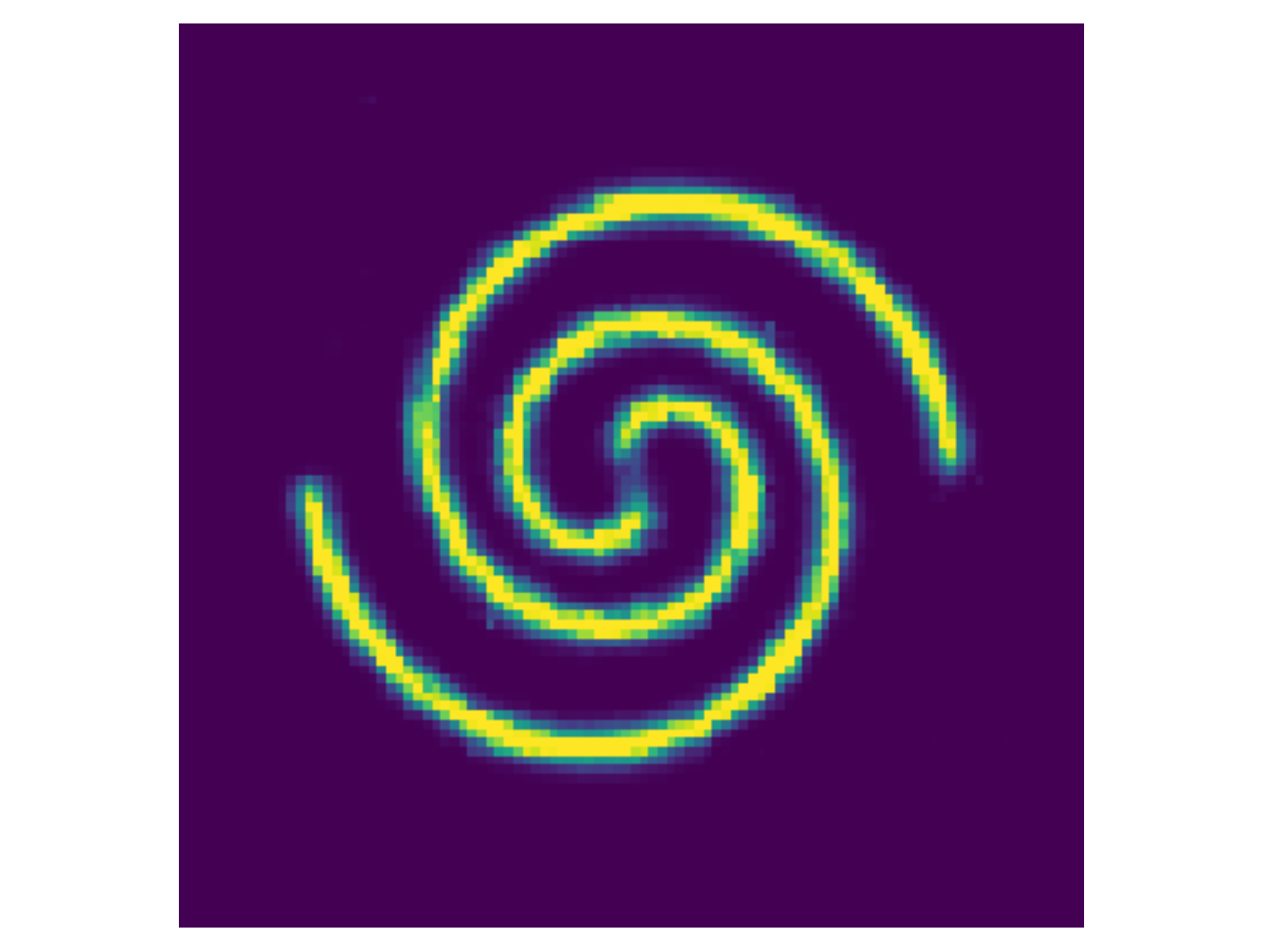}
    \end{minipage}
\hspace{-\hinterval}
    \begin{minipage}{\mpwid\textwidth}
    \centering
    \includegraphics[width=\textwidth,trim=0 30 0 30,clip
    ]{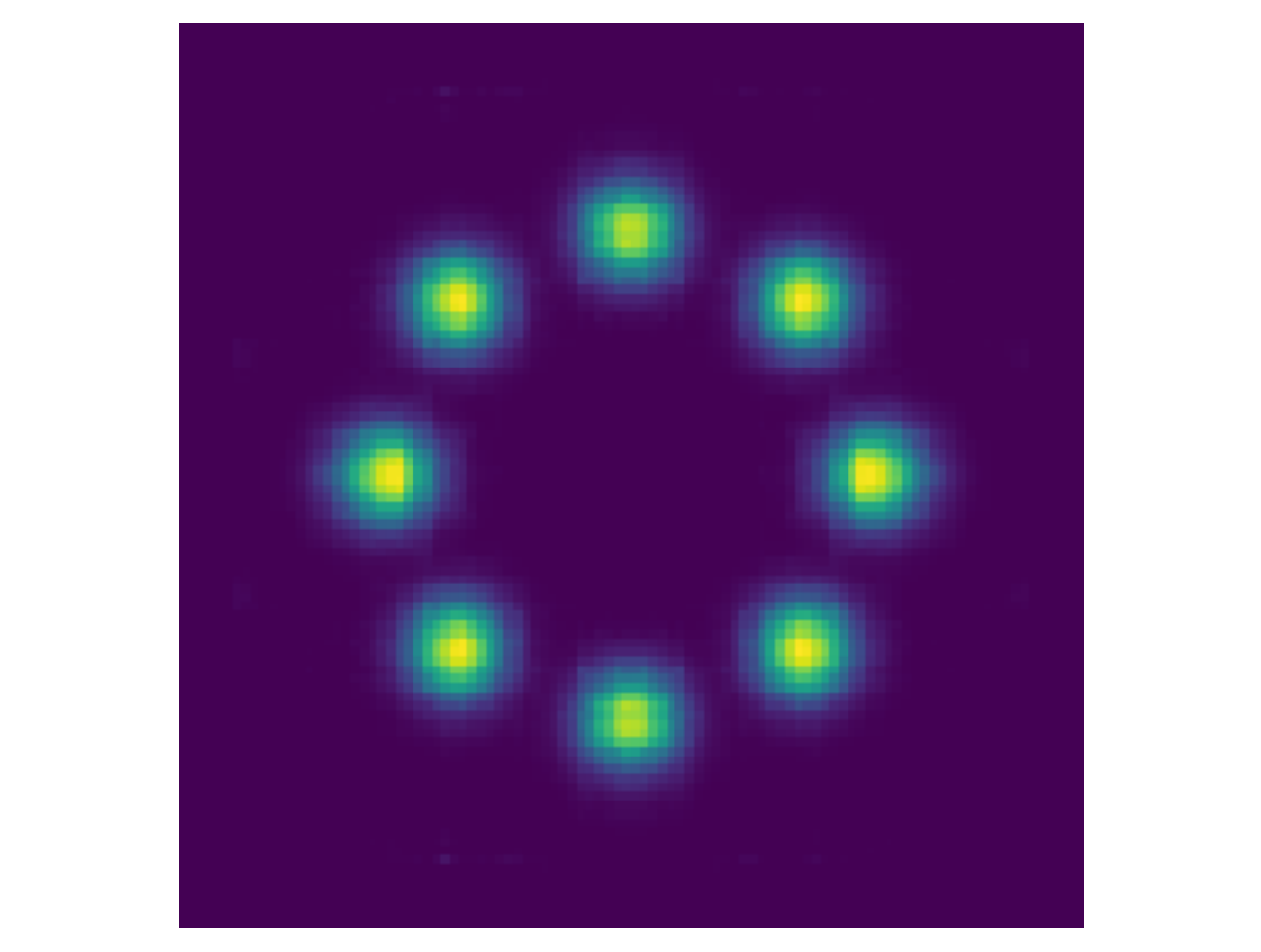}
    \end{minipage}
\hspace{-\hinterval}
    \begin{minipage}{\mpwid\textwidth}
    \centering
    \includegraphics[width=\textwidth,trim=0 30 0 30,clip
    ]{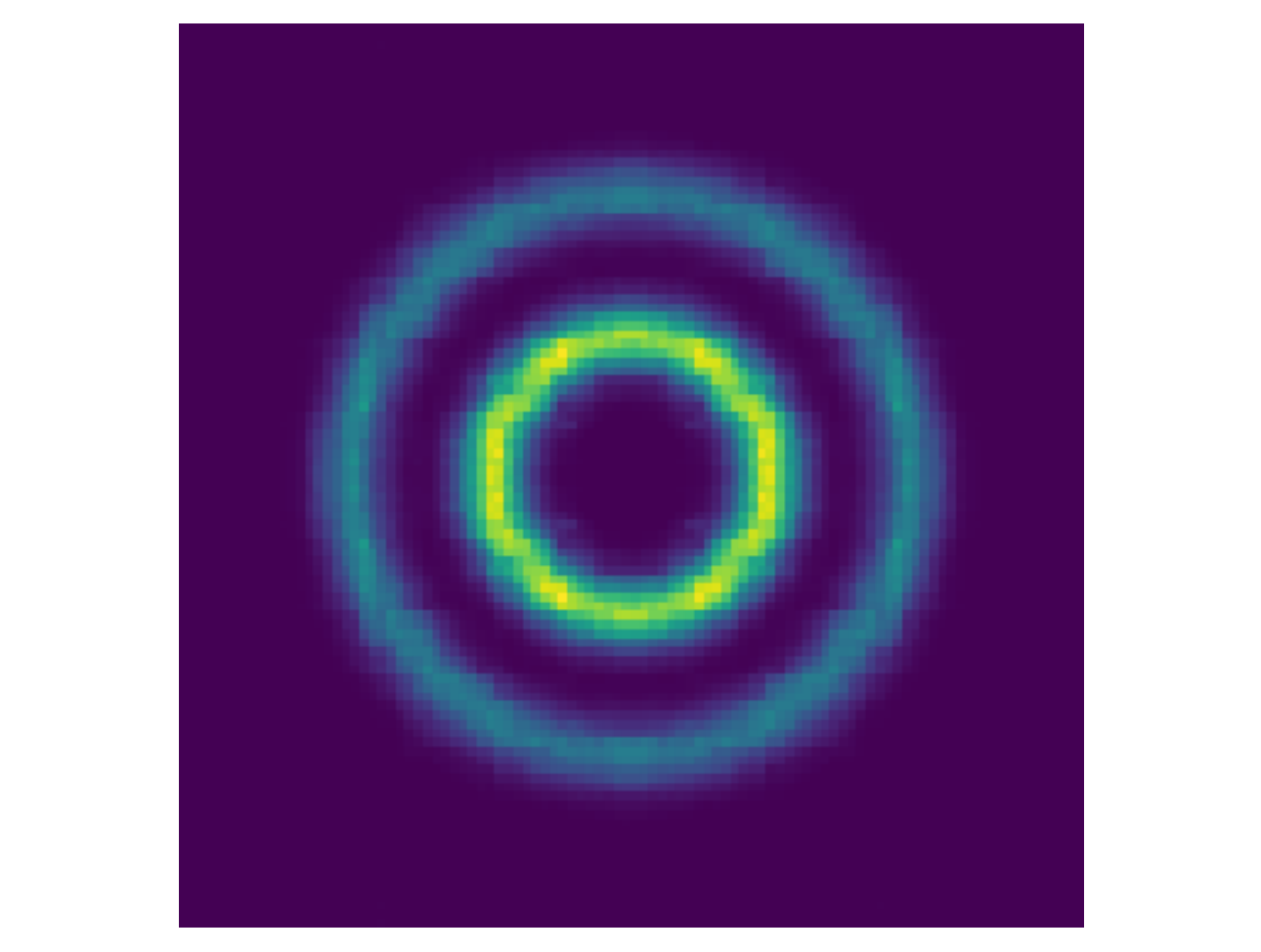}
    \end{minipage}
\hspace{-\hinterval}
    \begin{minipage}{\mpwid\textwidth}
    \centering
    \includegraphics[width=\textwidth,trim=0 30 0 30,clip
    ]{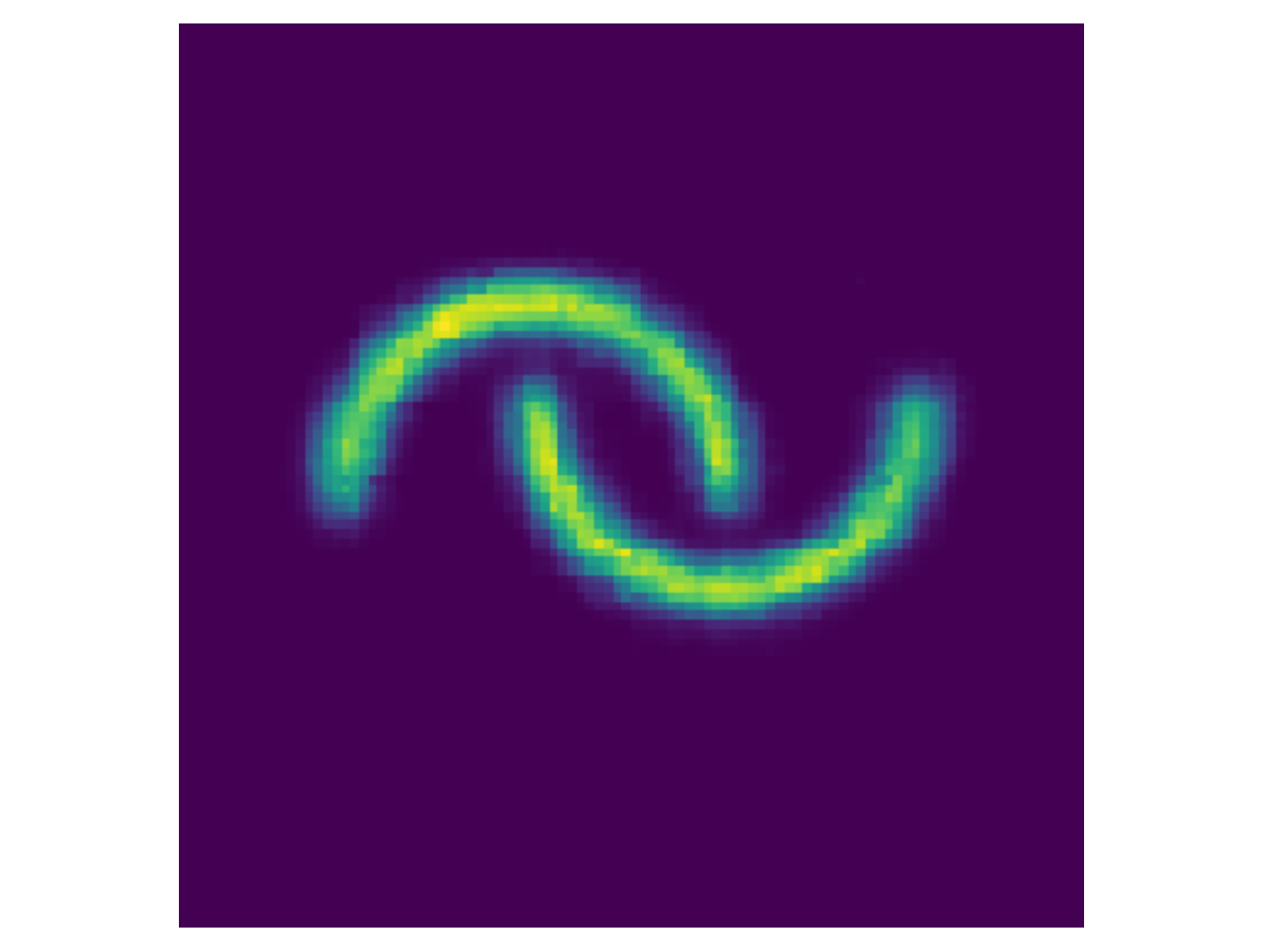}
    \end{minipage}
\hspace{-\hinterval}
    \begin{minipage}{\mpwid\textwidth}
    \centering
    \includegraphics[width=\textwidth,trim=0 30 0 30,clip
    ]{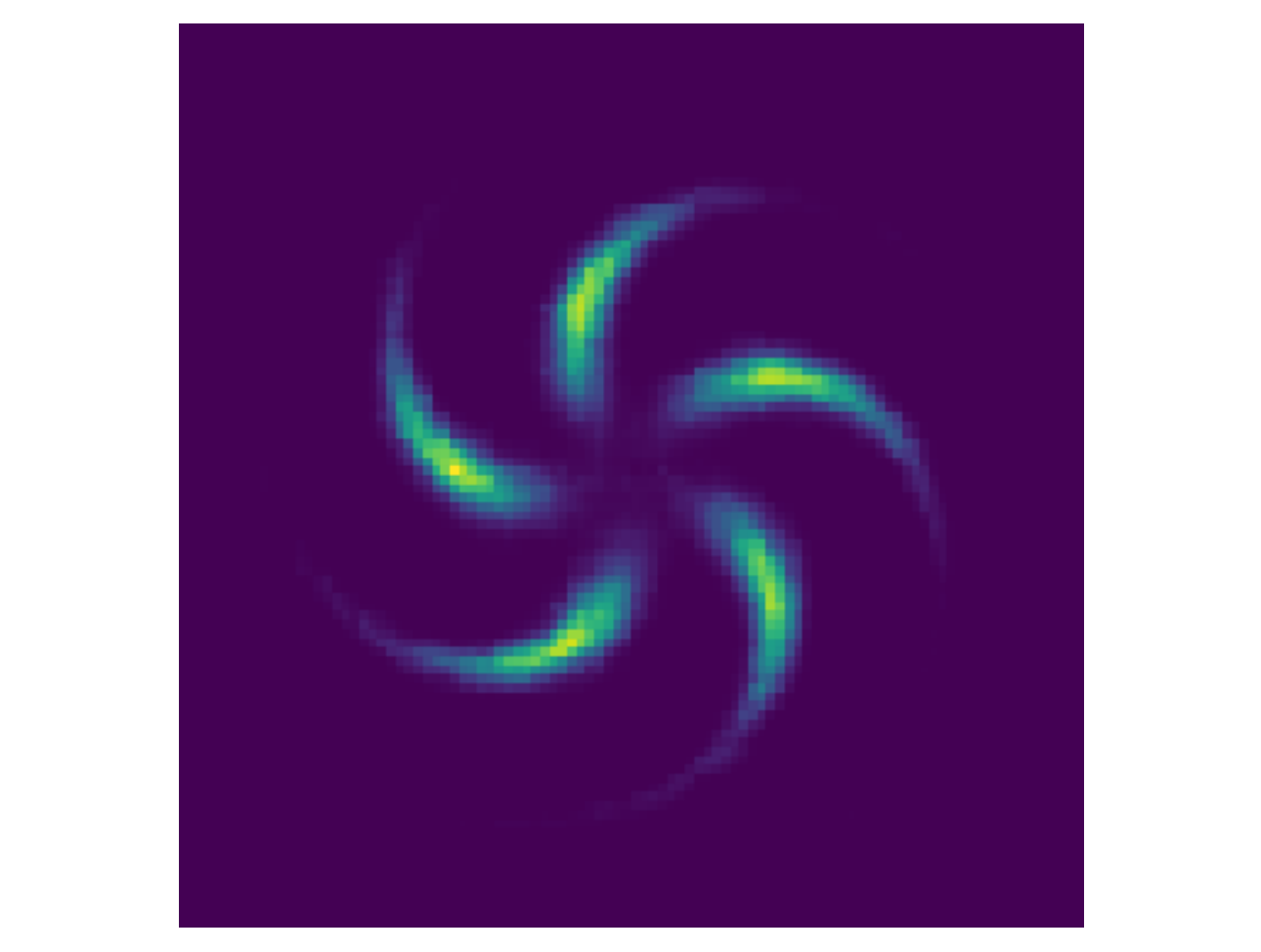}
    \end{minipage}
\hspace{-\hinterval}
    \begin{minipage}{\mpwid\textwidth}
    \centering
    \includegraphics[width=\textwidth,trim=0 30 0 30,clip
    ]{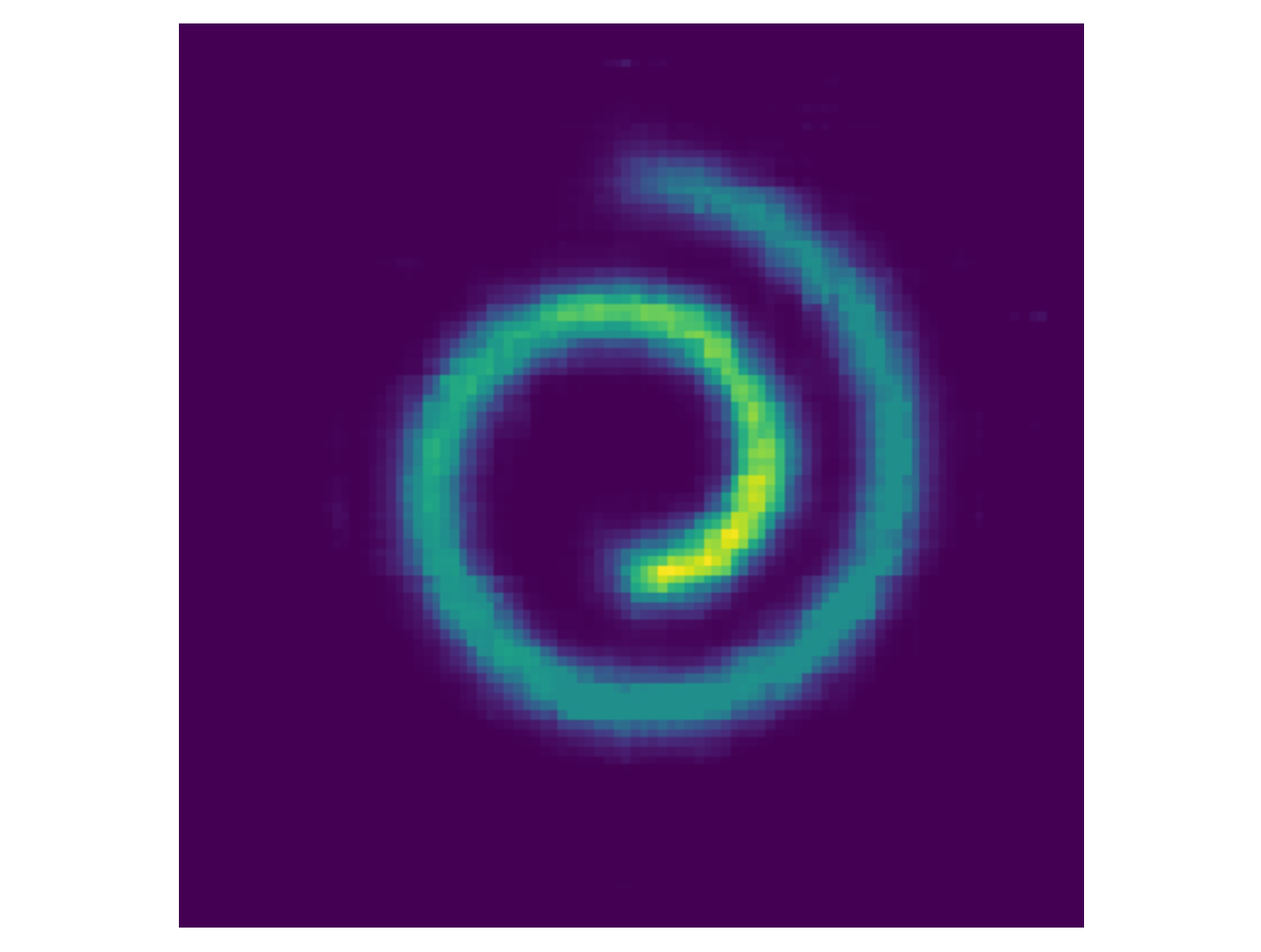}
    \end{minipage}
\hspace{-\hinterval}
    \begin{minipage}{\mpwid\textwidth}
    \centering
    \includegraphics[width=\textwidth,trim=0 30 0 30,clip
    ]{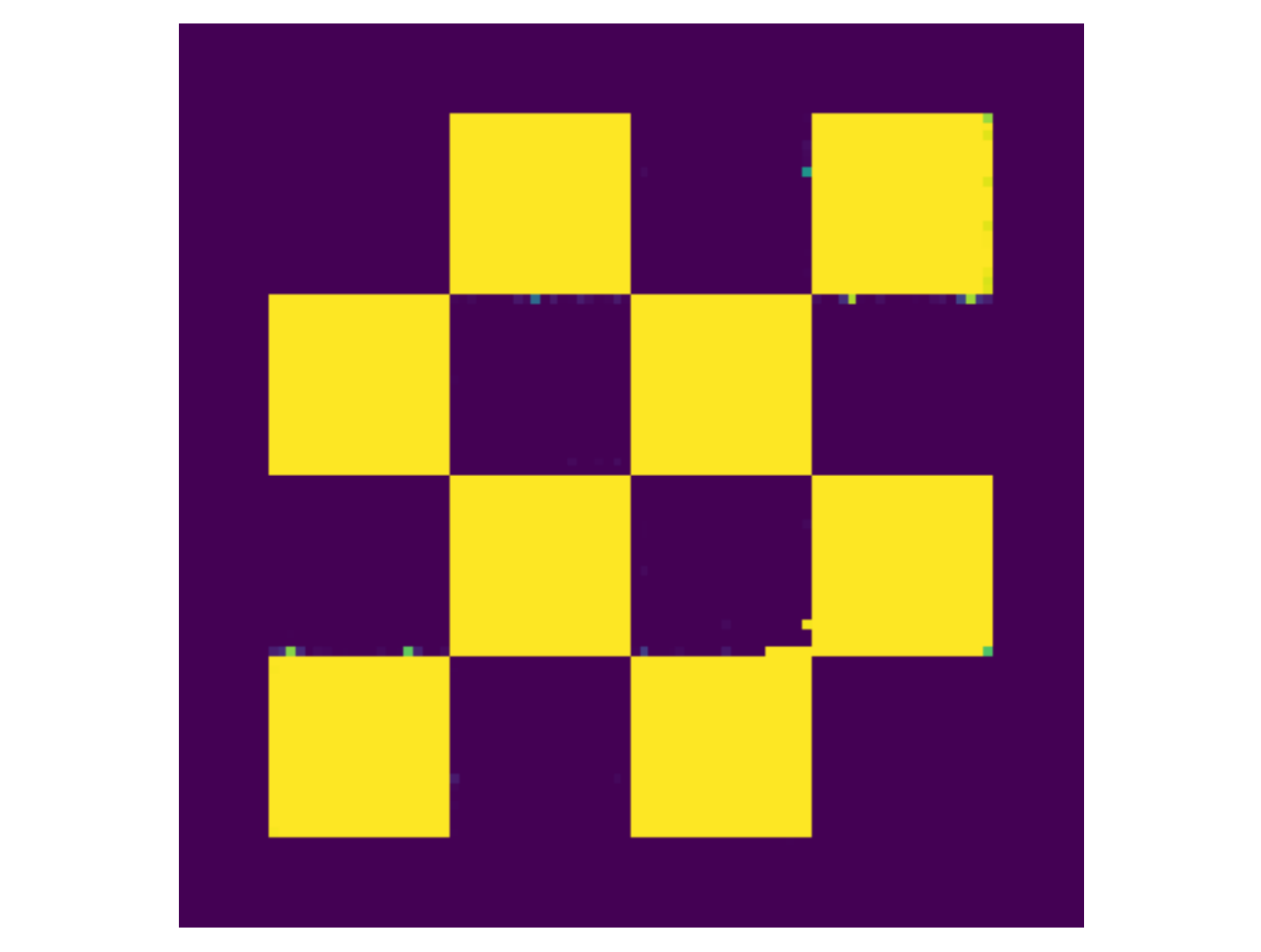}
    \end{minipage}
\hspace{-\hinterval}
\end{minipage}
\centering
\vspace{-2mm}
\caption{
\emph{Top:} Visualization of the samples generated with a learned GFlowNet. 
\emph{Bottom:} Visualization of the energy function learned jointly with the GFlowNet.
Due to limited space, we defer the visualization of baselines to Fig.~\ref{fig:baseline_energy}.
}
\vspace{-3mm}
\label{fig:gfn_synthetic_result}
\end{figure*}

\begin{table*}[t]
\setlength{\tabcolsep}{2.5mm}
\centering
\caption{
Experiment results with seven 2D synthetic problems.
We display the negative log-likelihood (NLL) and MMD (in units of $1\times 10^{-4}$).
Note that ALOE+ uses a thirty times larger parametrization than ALOE and EB-GFN.
}
\label{tab:synthetic_mmd}
\begin{tabular}{c|l|ccccccc}
\toprule
Metric & Method & 2spirals & 8gaussians & circles & moons & pinwheel & swissroll & checkerboard \\
\midrule
\multirow{4}{*}{NLL$\downarrow$} 
& PCD    &  $20.094$ & $19.991$ &$20.565$&$19.763$&$19.593$&$20.172$&$21.214$ \\
& ALOE   &  $20.295$ & $20.350$ & $20.565$ & $19.287$ & $19.821$ & $20.160$ & $54.653$ \\
& ALOE+  &  $20.062$ & $19.984$ & $20.570$ & $19.743$ & $19.576$ & $20.170$ & $21.142$\\
& EB-GFN & $\textbf{20.050}$ & $\textbf{19.982}$ & $\textbf{20.546}$ & $\textbf{19.732}$ & $\textbf{19.554}$ & $\textbf{20.146}$ & $\textbf{20.696}$ \\
\midrule
\multirow{4}{*}{MMD$\downarrow$} 
& PCD    &  $2.160$&$0.954$&$0.188$&$0.962$&$0.505$&$1.382$& $2.831$ \\
& ALOE   &  $21.926$ & $107.320$ & $0.497$ & $26.894$ & $39.091$ & $0.471$ & $61.562$ \\
& ALOE+  &  $ \textbf{0.149} $ & $\textbf{0.078}$ & $0.636$ & $0.516$ & $1.746$ & $0.718$ & $12.138$ \\
& EB-GFN &  $0.583$ & $0.531$ & $\textbf{0.305}$ & $\textbf{0.121}$ & $\textbf{0.492}$ & $\textbf{0.274}$ & $\textbf{1.206}$\\ 
\bottomrule
\end{tabular}
\vspace{-3mm}
\end{table*}

We validate the EB-GFN algorithm on 
the Ising model \citep{ising}. 
The Ising model is an elementary example of a Markov random field and is widely studied in mathematics and physics \citep[see][\S31]{mackay}.
A $D$-particle Ising model is a distribution over $D$-dimensional binary vectors, with entries called \emph{spins}. To keep with established conventions, we call the two possible values of each spin $\{+1,-1\}$ rather than $\{0,1\}$. 
The distribution is given by an energy model, where the energy is a quadratic form (with symmetric $D\times D$ matrix $J$) evaluated on the spin vector:
\vspace{-1mm}
\begin{align*}
    &P(\x)\propto\exp(-\gE_J(\x)),\quad\x\in\{\pm1\}^D,\nonumber\\
     &\gE_J(\x)\triangleq-\x^\top J\x=-\sum_{i=1}^D\sum_{j=1}^DJ_{ij}\x^i\x^j.
     \label{eq:ising_model}
\vspace{-2mm}
\end{align*}
Positive entries $J_{ij}$ encourage $\x^i$ and $\x^j$ to have the same spin, while negative $J_{ij}$ push them to have opposite spins. 
The matrix $J$ is often taken to be a multiple of the adjacency matrix of a graph.
Here we consider models of the form $J=\sigma A_N$, where $\sigma\in\R$ and $A_N$ is the adjacency matrix of a $N\times N$ grid with toroidal wrap-around ($D=N^2$).

We consider $10\times10$ grids with $\sigma=0.1,0.2,\dots,0.5$ and $9\times9$ grids with $\sigma=-0.1,-0.2$.\footnote{We use an odd grid size when $\sigma<0$ because such a model has many modes, each resembling a checkerboard with `seams' where the checkerboard pattern is violated.
} 
For each setting of $\sigma$ and $N$, we use standard methods \citep{swendsen-wang}
to generate 2000 samples $\{\x_i\}_{i=1}^{2000}$ from the Ising model with energy $\gE_J$. We then use the EB-GFN algorithm to jointly fit a symmetric matrix $J_{\vphi}$, giving an estimated Ising EBM $\gE_{J_{\vphi}}$, and a GFlowNet that samples from this Ising EBM.
Note that the EB-GFN algorithm does not have access to the true data-generating matrix $J$, only to the collection of samples $\{\x_i\}$. This is a simple test case for EB-GFN, since the energy is parametrized by a single matrix $J_{\vphi}$, not by a deep model. We evaluate the discrepancy (RMSE) between the true matrix $J$ and the learned matrix $J_{\vphi}$.


For simplicity, we set $K=D$ for the negative sampling step and use a training policy with $\alpha=1$ (no backward paths from training examples).  Details can be found in \S\ref{app:appendix_ising}. We compare EB-GFN with the Gibbs and Gibbs-With-Gradients (GWG) 
PCD algorithms \citep{Grathwohl2021OopsIT}. 
Table~\ref{tab:ising_results} shows the advantage of EB-GFN. Fig.~\ref{fig:ising_vis} shows how faithful both the samples and the energy function obtained by the GFlowNet are, 
suggesting that the GFlowNet is able to discover generalizable structure from the data.




\vspace{-0.5mm}
\subsection{Synthetic tasks}
\vspace{-0.5mm}
\label{sec:synthetic}

The experiment described in this subsection follows the setup of \citet{Dai2020ALOE}. 
We aim to model seven different distributions over 32-dimensional binary data that are discretizations of continuous distributions over the plane (shown in Fig.~\ref{fig:true_synthetic_samples}). To convert planar data $(x,y)\in\R^2$ to 32-dimensional binary data, we quantize both $x$ and $y$ into $2^{16}$ equal-width buckets, then remap them to 16-bit representations via a Gray code \citep{gray}, so that any two neighbouring buckets differ in exactly one bit. 
As a result, the methods are required to model data in $\{0, 1\}^{32}$.

We compare with PCD-10 and ALOE \citep{Dai2020ALOE}, two baselines for energy modeling in discrete spaces.
We use the same energy function architecture 
and training protocol as \citet{Dai2020ALOE}.
The PCD-10 baseline utilizes Gibbs sampling with a replay buffer and the random restart mechanism \citep{Du2019ImplicitGA} for negative sample generation. 
ALOE learns three neural networks for negative sampling: a proposal model to provide initial samples and a pair of models (local search policy and stop policy) used to iteratively refine these samples.
The initial proposal model could be either a simple multinomial distribution or a large autoregressive network.
We name them ALOE and ALOE+ respectively, as the former has a similar total number of parameters to our GFlowNet,
while ALOE+ is $32\times$ larger.
For GFlowNet training, we use a mixed training policy with $\alpha=0.5$ and a schedule with linearly increasing $K$ for the back-and-forth proposal.
See \S\ref{sec:appendix_synthetic} for details.

For qualitative evaluation, we first visualize in Fig.~\ref{fig:gfn_synthetic_result} the heatmaps of the learned energy functions and some GFlowNet-generated samples by remapping the Gray code representations of samples back to 2-D space.
As a comparison, we also visualize the energy model baselines in Fig.~\ref{fig:baseline_energy}.
We observe that for multimodal tasks such as {\em checkerboard} and {\em 8gaussians}, GFlowNets are much better at capturing the modes and their structure than the baselines (as illustrated schematically in Fig.~\ref{fig:main_fig}).
We hypothesize that {\em multimodality is easier to handle with GFlowNets than with MCMC if there are generalizable regularities that make it possible for the GFlowNet to guess new modes from those already visited and from which it has learned. }

We quantitatively evaluate the algorithms in Table~\ref{tab:synthetic_mmd} by showing for each method the NLL of a large independent sample of ground truth data and the exponential Hamming MMD \citep{Gretton2012AKT} between ground truth data and generated samples as performance metrics. 
Our method outperforms the baselines on all datasets and metrics except MMD on {\em 2spirals} and {\em 8gaussians}, where it still exceeds comparable-size baselines. 
Of note, ALOE does not surpass PCD for many tasks without a large initial proposal network, which shows a potential weakness of its local search strategy.
We defer more results and related details to \S\ref{sec:appendix_synthetic}.

\vspace{-0.5mm}
\subsection{Discrete image modeling}
\vspace{-0.5mm}
\label{sec:image}

\begin{table}[t]
\vspace{-3mm}
\caption{Experiments on discrete image modeling. We display the negative log likelihood (NLL) per sample for different algorithms.}
\label{tab:image_logll}
\begin{center}
\begin{tabular}{lccc}
\toprule 
Dataset $\backslash$ Method & Gibbs & GWG & EB-GFN \\
\midrule 
Omniglot & $133.92$ & $114.96$ & $\textbf{112.59}$ \\
Silhouettes & $475.55$ & $188.82$ &  $\textbf{185.57}$ \\
Static MNIST & $173.61$ & $\textbf{99.36}$ & $102.43$ \\
Dynamic MNIST & $162.25$ & $108.29$ & $\textbf{105.75}$ \\
\bottomrule
\end{tabular}
\end{center}
\vspace{-3mm}
\end{table}

Here, we aim to generatively model 
previously studied image datasets in discrete high-dimensional spaces. 
These are generally hard problems as we lose the information of continuous pixel values, and prevalent scalable methods \citep{Welling2011BayesianLV, Ma2015ACR}
are not applicable to training deep EBMs on discrete data.
An MLP is utilized as the energy function, and
the baseline training methods are PCD with Gibbs sampling and the Gibbs-With-Gradients sampling method \citep[GWG;][]{Grathwohl2021OopsIT}.
Experiments are performed on four different binary image datasets.
Following the experimental settings in \citet{Grathwohl2021OopsIT}, the EBM is trained via PCD-100 with different negative sampling methods, and a replay buffer is adopted as in \citet{Du2019ImplicitGA}.
The validation and evaluation protocol is also kept aligned with \citet{Grathwohl2021OopsIT},
where the checkpoint with the best negative log-likelihood on the validation set is reported.

The test set NLLs are displayed in Table~\ref{tab:image_logll}.
The results indicates that
EB-GFN reaches state-of-the-art energy modeling performance, surpassing GWG, on three of the four datasets.
As a supplement, 
a visualization of Dynamic MNIST data is shown in Figure \ref{fig:mnist}. 
More details are in \S\ref{sec:appendix_image}.

\vspace{-1.5mm}
\section{Related Work}
\vspace{-1.5mm}

\textbf{Energy-based models.} EBM, one of the central methods in generative modeling, has proved effective with energy functions parametrized by deep nets \citep{Hinton2006AFL, Salakhutdinov2009DeepBM}. To avoid costly MCMC simulation with deep models, contrastive divergence-type methods \citep{Hinton2002TrainingPO,Tieleman2008TrainingRB,Xie2016ATO,Du2021ImprovedCD} were proposed to approximate the energy gradient. It has also been shown that better objectives beyond vanilla CD helps EBM training \citep{Yu2020TrainingDE}. Training methods have been proposed for better stability, shorter mixing time, faster training \citep{Nijkamp2019LearningNN, Du2019ImplicitGA, Grathwohl2021NoMF, Gao2021LearningEM}. Recent work shows that it can be beneficial to learn the sampler or the proposal distribution as well \citep{Dai2019ExponentialFE, Arbel2021GeneralizedEB}, a finding that this work extends to discrete spaces. 


\begin{figure}[t]
\centering
\vspace{-1mm}
\begin{minipage}{0.45\columnwidth}
\begin{minipage}{\textwidth}
\includegraphics[width=\textwidth]{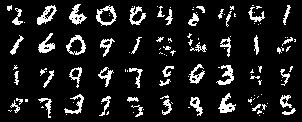}
\end{minipage}
\vspace{0.3cm}

\begin{minipage}{\textwidth}
\includegraphics[width=\textwidth]{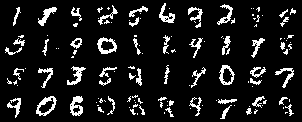}
\end{minipage}
\end{minipage}
\hspace{0.2cm}
\begin{minipage}{0.45\columnwidth}
\centering
\includegraphics[width=\textwidth]{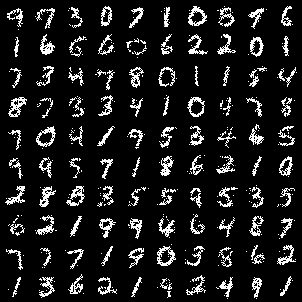}
\end{minipage}
\vspace{-1.5mm}
\caption{Visualization of the Dynamic MNIST samples generated from Gibbs \textit{(top left)}, GWG \textit{(bottom left)} and GFlowNet \textit{(right)}. 
EB-GFN models the details of certain modes better (see, \eg, the bottom right corner of the GWG samples). Note that the original images were binarized, which explains in part why all samples look noisier than the familiar gray-level MNIST images.}
\label{fig:mnist}
\vspace{-4.5mm}
\end{figure}

\textbf{Related methods.}
Autoregressive models \citep{Sutskever2011GeneratingTW, Graves2013GeneratingSW, Germain2015MADEMA}, like our GFlowNets, generate each entry of a data vector sequentially, but in a fixed order. 
Autoregressive models can also be used for data without a natural order \citep{ Uria2016NeuralAD, Oord2016PixelRN, Oord2016ConditionalIG, Meng2021ImprovedAM}. 
Some recent work 
\citep{Emelianenko2019SequenceMW, Li2021DiscoveringNA} allows generation order to be learned. We are the first to interpret learning of generation order as a joint inference of forward (construction) and reverse (erasure) Markov processes.


\textbf{Discrete inference.}
Probabilistic inference in discrete settings is generally harder than in continuous spaces. Many optimized sampling methods for discrete spaces have been developed \citep{Titsias2017TheHB,Zanella2019InformedPF,Han2020SteinVI,zhang2022langevin}.
Most applicable to our experiment domains, \citet{Grathwohl2021OopsIT} uses a continuous relaxation to approximate the local energy landscape, while
\citet{Dai2020ALOE} introduces a local search strategy in the variational distribution to initialize negative sample generation.


\vspace{-2mm}
\section{Conclusion}
\vspace{-1.5mm}

We have extended GFlowNets to the setting where one is given a dataset rather than a fixed energy function, and we learn both the GFlowNet sampler and the energy function. In doing so, we introduced a new proposal for MH MCMC that approaches block Gibbs sampling as the GFlowNet training converges. The main advantage of this proposal is that it can perform large jumps in the state space, unlike simple Gibbs sampling, taking advantage of the compositional structure that the GFlowNet may have uncovered that allows it to generalize across modes of the distribution and more easily jump between them (Fig.~\ref{fig:main_fig}). Future work can consider iterating such proposals from a trained GFlowNet for efficient exploration of the space, rather than just sampling complete trajectories ($K=D$) from the GFlowNet.

The cost of such an approach is that in addition to training the energy function, we also have to train the sampler. However, this cost can be amortized if we intend to use the sampler later, since generating samples from the GFlowNet is often much cheaper than from a MCMC (if we want to make sure to cover the modes well). 
We hypothesize that the above advantages explain the good comparative results obtained here, and expect that the proposed approach may be extended to many other types of generative tasks.


\section*{Acknowledgement}
We thank Yilun Du, Ricky T. Q. Chen, and Mila GFlowNet group for helpful discussion.
Dinghuai Zhang thanks the never-ending snow storm in Montreal for preventing him from any form of outdoor activity :).
Zhen Liu thanks miHoYo
for the joy from their awesome games in the tough winters.
Aaron Courville thanks the support of Microsoft Research, Hitachi and CIFAR.
Yoshua Bengio acknowledges the funding from CIFAR, Samsung, IBM and Microsoft.


\bibliography{ref}
\bibliographystyle{icml2022}

\newpage 
\onecolumn
\appendix

\counterwithin{figure}{section}
\counterwithin{table}{section}

\section{Summary of GFlowNet notation}

\begin{table}[h]
\centering
\begin{tabular}{ll}
 \toprule
 Symbol & Description \\ 
 \midrule
 $D$ & data dimension \\
 $\gS$ & state space $\{0,1,\oslash\}^D$ \\
 $\gX$ & data space $\{0,1\}^D$, identified with the set of terminal states in $\gS$\\
 $\gA$ & action / transition space (edges $\s\to\s'$)\\
 $\gT$ & set of complete trajectories \\
 $G$ & directed acyclic graph $(\gS,\gA)$ \\
 $\s$ & state in $\gS$\\
 $\s_0$ & initial state $(\oslash,\dots,\oslash)\in\gS$ \\
 $\x$ & terminal state in $\gX$ \\
 $F: \gT \to \R$ & Markovian flow\\
 $F: \gS \to \R$ & state flow, used in the proof of Proposition~\ref{prop:uniform}\\
 $F: \gA \to \R$ & edge flow, used in the proof of Proposition~\ref{prop:uniform}\\
 $P_F$ & forward policy (distribution over children) \\
 $P_B$ & backward policy (distribution over parents) \\
 $Z$ & scalar, equal to $\sum_{\tau\in\gT}F(\tau)$ for a Markovian flow
 \\
 \bottomrule
\end{tabular}
\end{table}

\section{Proofs of propositions}
\label{app:proofs}


Recall Proposition~\ref{prop:uniform}:
\begin{proposition*}
\uniformprop
\end{proposition*}
\begin{proof}
We use the following definitions from \citet{bengio2021foundations}:

For a trajectory flow $F$ and for any state $\s$, define the state flow $F(\s)=\sum_{\s\in\tau}F(\tau)$, and, for any edge $\s\ra\s'$, the edge flow \[F(\s\ra\s')=\sum_{\tau=(\dots\ra \s\ra \s'\ra\dots)}F(\tau).\] Notice that $Z=F(s_0)$ immediately from (\ref{eq:markovian_factorization}).

If $F$ is a Markovian flow, then $P_F$ and $P_B$ can be computed in terms of state and edge flows:
\begin{equation}
    P_F(\s'|\s)=\frac{F(\s\ra \s')}{F(\s)},\quad P_B(\s|\s')=\frac{F(\s\ra \s')}{F(\s')},
    \label{eqn:flow_to_policy}
\end{equation}
and we have 
\begin{equation}
F(\s)=\sum_{(\s''\ra \s)\in \gA}F(\s''\ra \s)=\sum_{(\s\ra \s')\in \gA}F(\s\ra \s').
\label{eqn:fm_constraint}
\end{equation}

The following computation shows that the entropy of the forward policy, defined by (\ref{eqn:entropy}), equals a similar expression for the backward policy:
\begin{align*}
    \gH[F]
    &=\E_{(\s_0\ra\dots\ra \s_n)\sim P_F}\left[\sum_{t=0}^{n-1}\gH[P_F(\cdot|\s_t)]\right] & \text{(definition (\ref{eqn:entropy}))}\\
    &=\sum_{\s\in\gS\text{ nonterminal}}\P[\tau=(\dots\ra \s\ra\dots)]\gH[P_F(\cdot|\s)]&\text{(linearity of expectation)}\\
    &=-\sum_{\s\in\gS\text{ nonterminal}}\frac{F(\s)}{Z}\sum_{\s':(\s\ra \s')\in\gA}P_F(\s'|\s)\log P_F(\s'|\s) & \text{(by definition of state flow)}\\
    &=-\sum_{(\s,\s')\in\gA}\frac{F(\s)}{Z}\frac{F(\s\ra \s')}{F(\s)}\log\frac{F(\s\ra \s')}{F(\s)}&\text{(grouping terms and (\ref{eqn:flow_to_policy}))}\\
    &=\frac{-1}{Z}\left(\sum_{(\s\ra \s')\in\gA}F(\s\ra \s')\log F(\s\ra \s')-\sum_{\s\in\gS\text{ nonterminal}}F(\s) \log F(\s)\right)&\text{(rearrangement and (\ref{eqn:fm_constraint}))}\\
    &=-\sum_{(\s,\s')\in\gA}\frac{F(\s')}{Z}\frac{F(\s\ra \s')}{F(\s')}\log\frac{F(\s\ra \s')}{F(\s')}+\underbrace{\frac{1}{Z}\left(Z\log Z-\sum_{x\text{ terminal}}F(x)\log F(x)\right)}_{\Delta}&\text{(rearrangement and (\ref{eqn:fm_constraint}))}\\
    &=-\sum_{\s'\in\gS\text{ noninitial}}\frac{F(\s')}{Z}\sum_{\s:(\s\ra \s')\in\gA}P_B(\s|\s')\log P_B(\s|\s')+\Delta&\text{(grouping terms and (\ref{eqn:flow_to_policy}))}\\
    &=\sum_{\s'\in\gS\text{ noninitial}}\P[\tau=(\dots\ra \s'\ra\dots)]\gH[P_B(\cdot|\s')]+\Delta&\text{(by definition of state flow)}\\
    &=\E_{(\s_0\ra\dots\ra \s_n)\sim P_F}\left[\sum_{t=1}^n\gH[P_B(\cdot|\s_t)]\right]+\Delta.&\text{(linearity of expectation)}
\end{align*}
Because we have assumed $F(x)=R(x)$ for all $x$ terminal (condition (\ref{eq:reward_matching})), and we have $Z=\sum_{\x\in\gX}F(\x)$ clearly from the definitions, the quantity $\Delta$ is independent of the choice of Markovian flow. Therefore, maximizing $\gH[F]$ is equivalent to maximizing the expected entropy of $P_B$.

Finally, notice that every complete trajectory $s_0\ra\dots\ra s_n$ passes through exactly one state $\s_d$ with $\gH[P_B^\circ(\cdot|\s_d)]=\log d$ for each $d=1,\dots,D$, and that 
\[\gH[P_B(\cdot|\s_d)]\leq\gH[P_B^\circ(\cdot|\s_d)]\] with equality if $P_B$ is uniform over the parents of $\s_d$. Thus $\gH[F]$ is maximized when $P_B(\cdot|\s)=P_B^\circ(\cdot|\s)$ for all $\s$.
\end{proof}

Then we prove the Proposition \ref{prop:mh_step_noneffective}:
\begin{proposition*}
\mhstepprop
\end{proposition*}
\begin{proof}
Recall that the acceptance probability for a move from $\x$ to $\x'$ along a reverse trajectory $\tau$ and a forward trajectory $\tau'$ is given by
\[
A_{\tau,\tau'}(\x\to\x')\triangleq\min\left(1, e^{\gE_{\vphi}(\x)-\gE_{\vphi}(\x')}\frac{P_B(\tau|\x)P_F(\tau')}{P_B(\tau'|\x')P_F(\tau)}\right).
\]
According to Eq. (21) 
of \citet{tbarxiv}, for a GFlowNet satisfying the reward matching constraint (\ref{eq:reward_matching}) with respect to a reward function $R$, we have $R(\x)P_B(\tau|\x)P_F(\tau')=R(\x')P_B(\tau'|\x')P_F(\tau)$. Elementary algebraic manipulation, and substituting $R(\x)=e^{-\gE(\x)}$, yields that $A_{\tau,\tau'}(\x\to\x')=1$.
\end{proof}

\section{More about experiments}
\subsection{Ising models}\label{app:appendix_ising}

For the architecture of the GFlowNet, we use a four-layer MLP with 256 units in each hidden layer. 
We used Adam optimizer and batch size 256 to train
EB-GFN
and the baselines. 
For baselines, we use $100$ steps of MCMC computation, \ie, PCD-100.
For all methods, we stopped the training when RMSE between ground truth and learnt $J$ reached its minimal value. 
We report the best result for each setting with a same hyperparameter searching protocol for all three methods.
We search the learning rate of energy function in $\{1\times 10^{-4}, 5\times 10^{-4}
\}$, the learning rate of GFlowNet in $\{1\times 10^{-3}, 1\times 10^{-2}\}$,
the coefficient $\ell_1$ regularization of $J$ in $\{0.01, 0.02, 0.05, 0.1\}$.
We keep other hyperparameters to be consistent with \citet{Grathwohl2021OopsIT}.
Notice that the benefit of EB-GFN is most clear when $\sigma<0$.
This matches our hypothesis that GFlowNet has a good inductive bias for modeling multimode distributions (when $\sigma>0$, the lattice Ising model only has two modes, namely when the configuration is all $+1$ or all $-1$).

\subsection{Synthetic tasks}
\label{sec:appendix_synthetic}

\begin{figure}[t]
\begin{minipage}{\textwidth}
    \begin{minipage}{\samplempwid\textwidth}
    \centering
    \includegraphics[width=\textwidth]{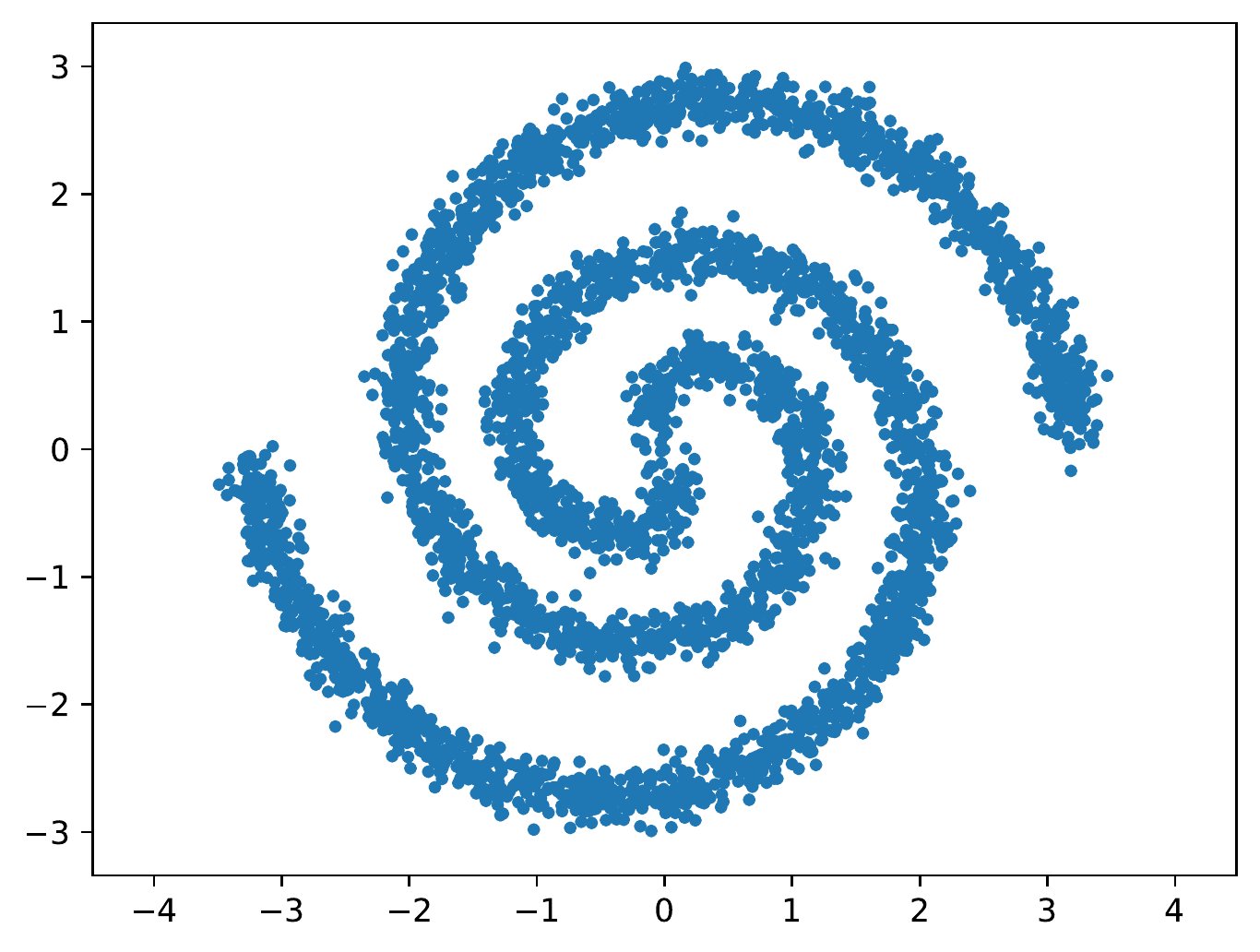}
    2spirals
    \end{minipage}
\hspace{-\samplehinterval}
    \begin{minipage}{\samplempwid\textwidth}
    \centering
    \includegraphics[width=\textwidth]{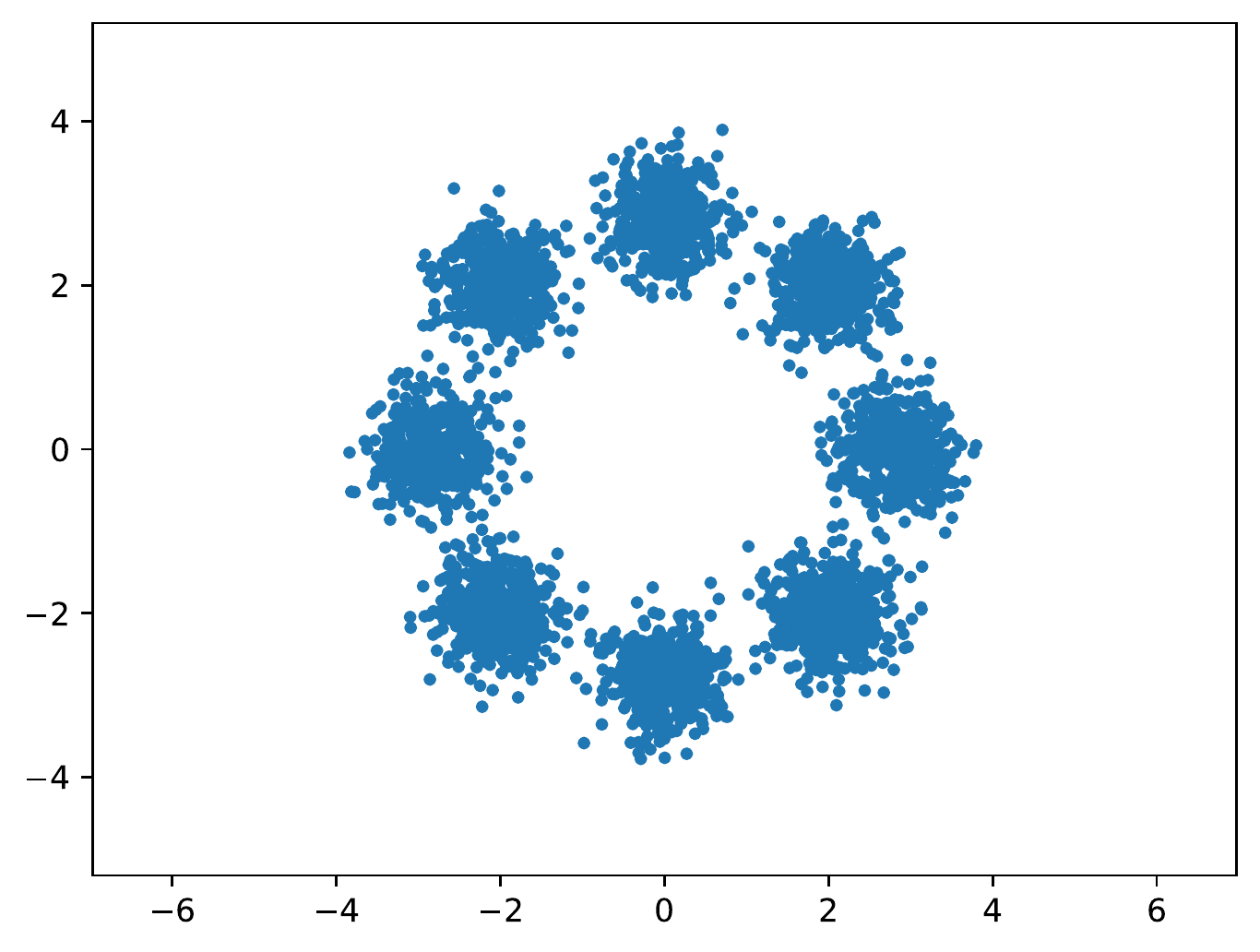}
    8gaussians
    \end{minipage}
\hspace{-\samplehinterval}
    \begin{minipage}{\samplempwid\textwidth}
    \centering
    \includegraphics[width=\textwidth]{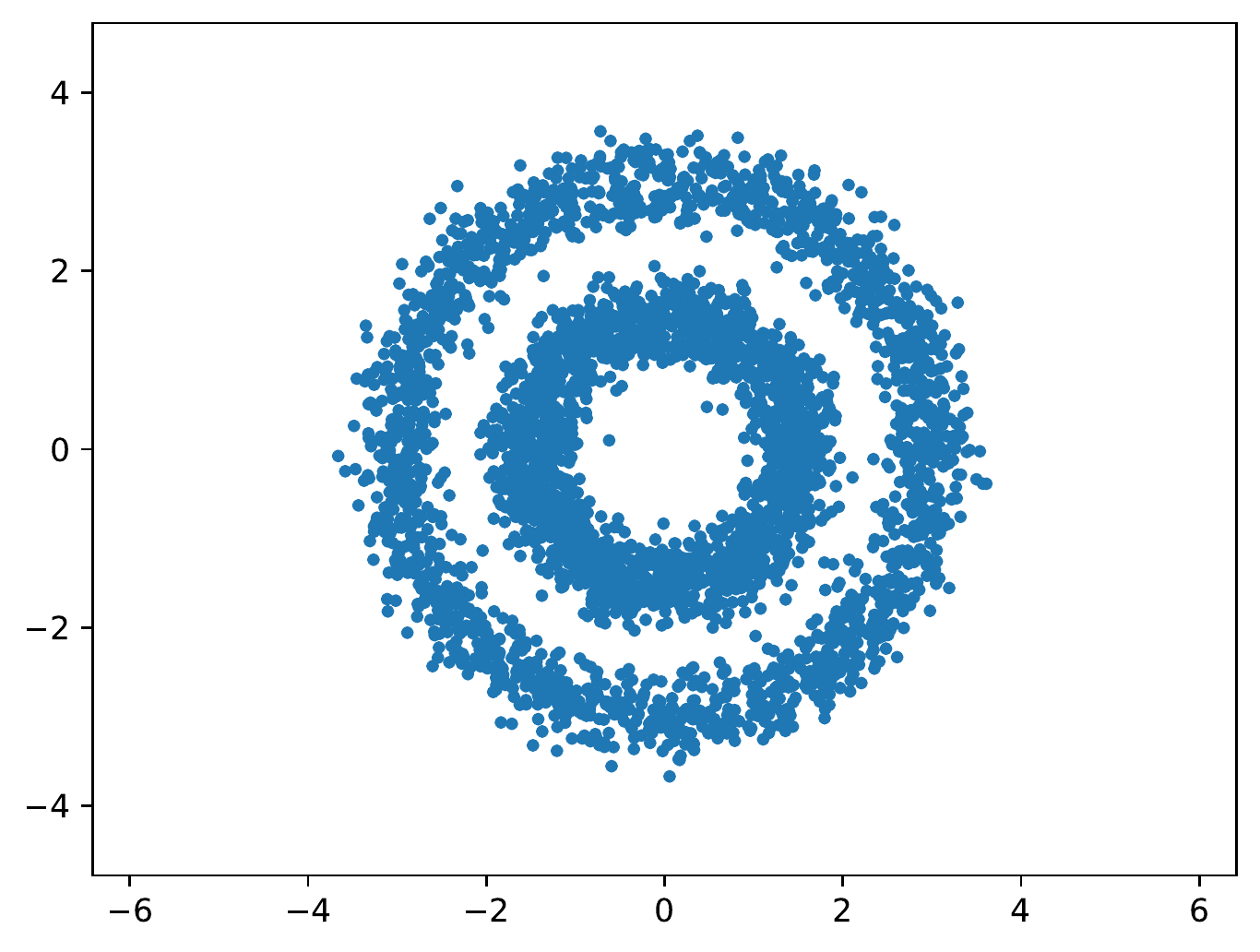}
    circles
    \end{minipage}
\hspace{-\samplehinterval}
    \begin{minipage}{\samplempwid\textwidth}
    \centering
    \includegraphics[width=\textwidth]{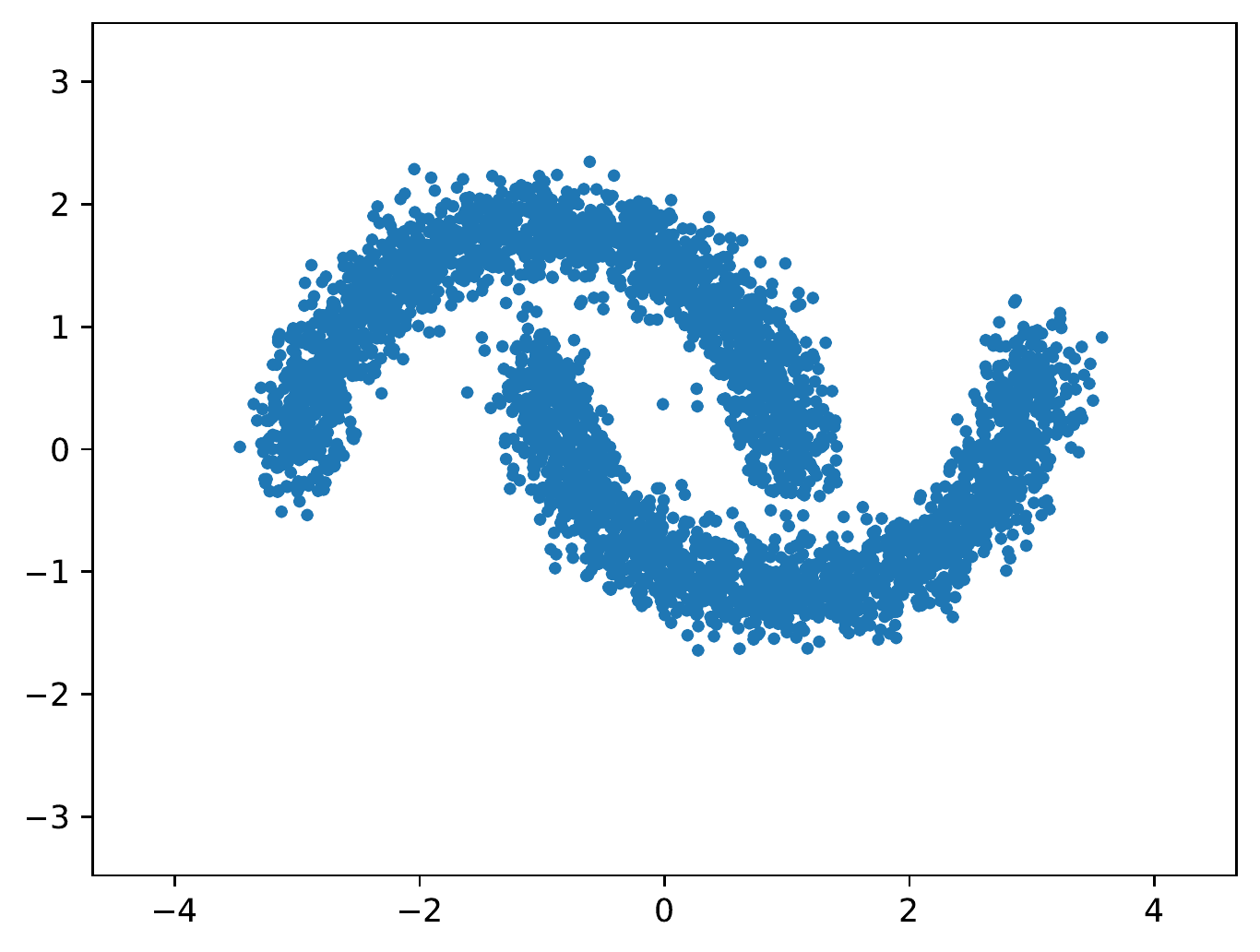}
    moons
    \end{minipage}
\hspace{-\samplehinterval}
    \begin{minipage}{\samplempwid\textwidth}
    \centering
    \includegraphics[width=\textwidth]{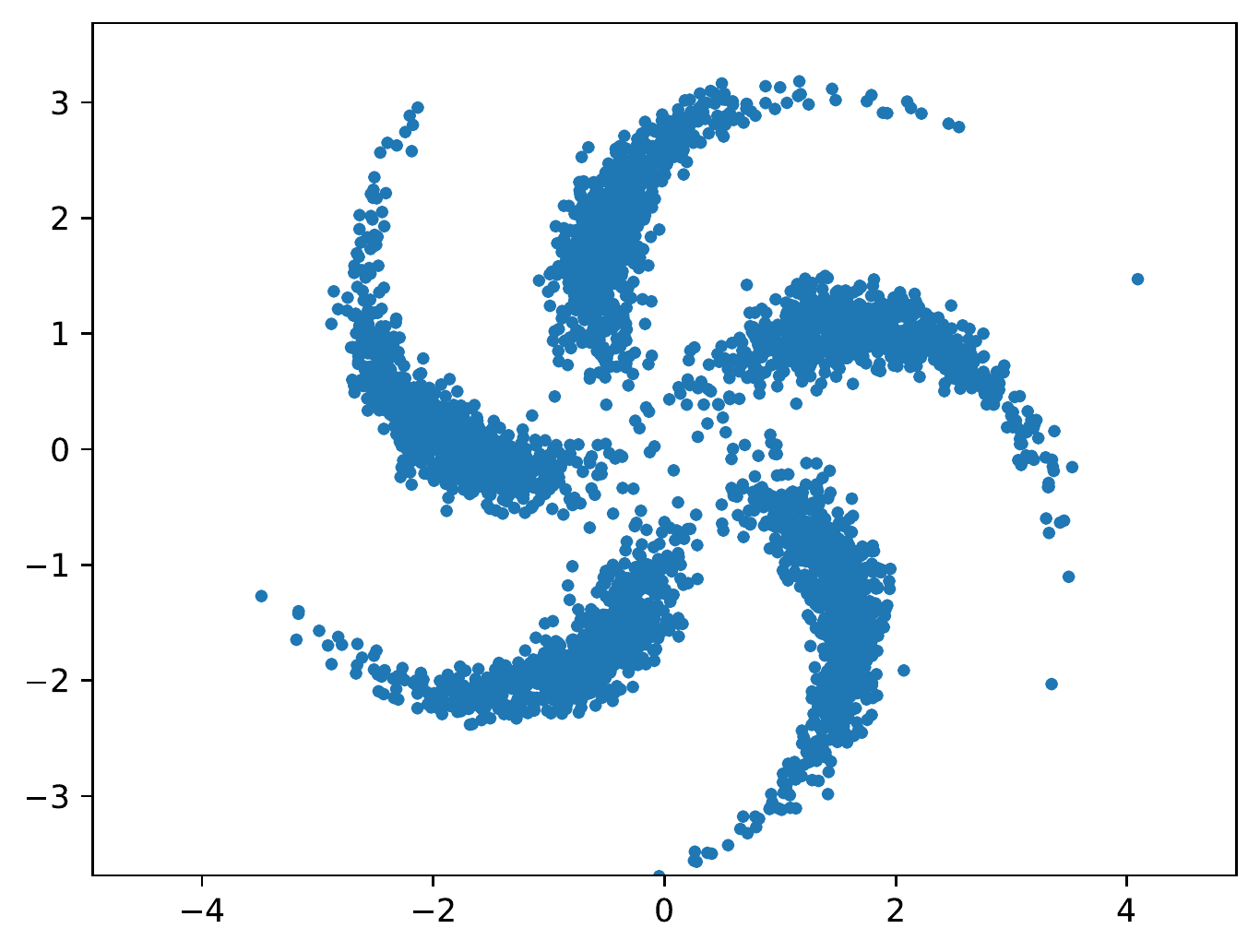}
    pinwheel
    \end{minipage}
\hspace{-\samplehinterval}
    \begin{minipage}{\samplempwid\textwidth}
    \centering
    \includegraphics[width=\textwidth]{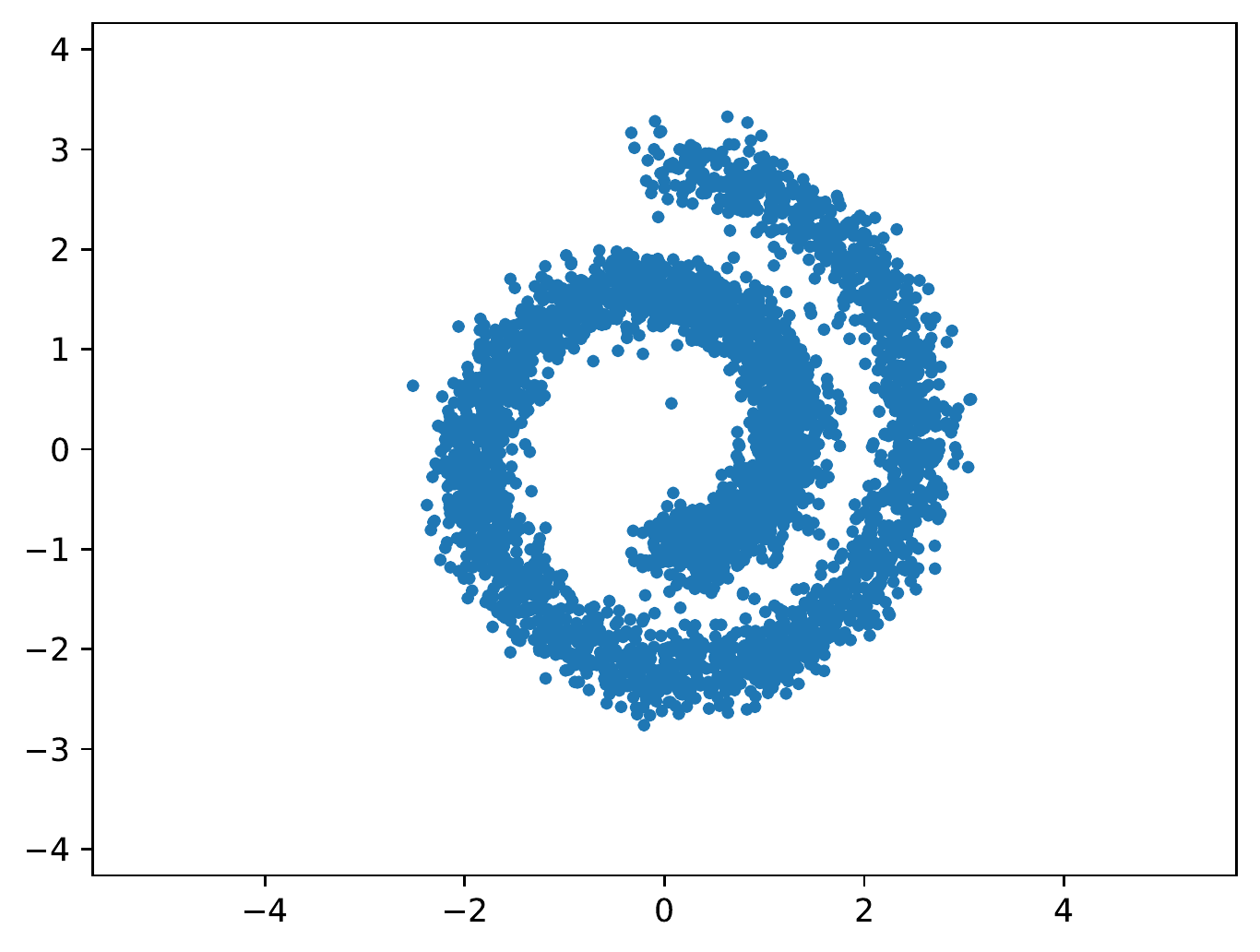}
    swissroll
    \end{minipage}
\hspace{-\samplehinterval}
    \begin{minipage}{\samplempwid\textwidth}
    \centering
    \includegraphics[width=\textwidth]{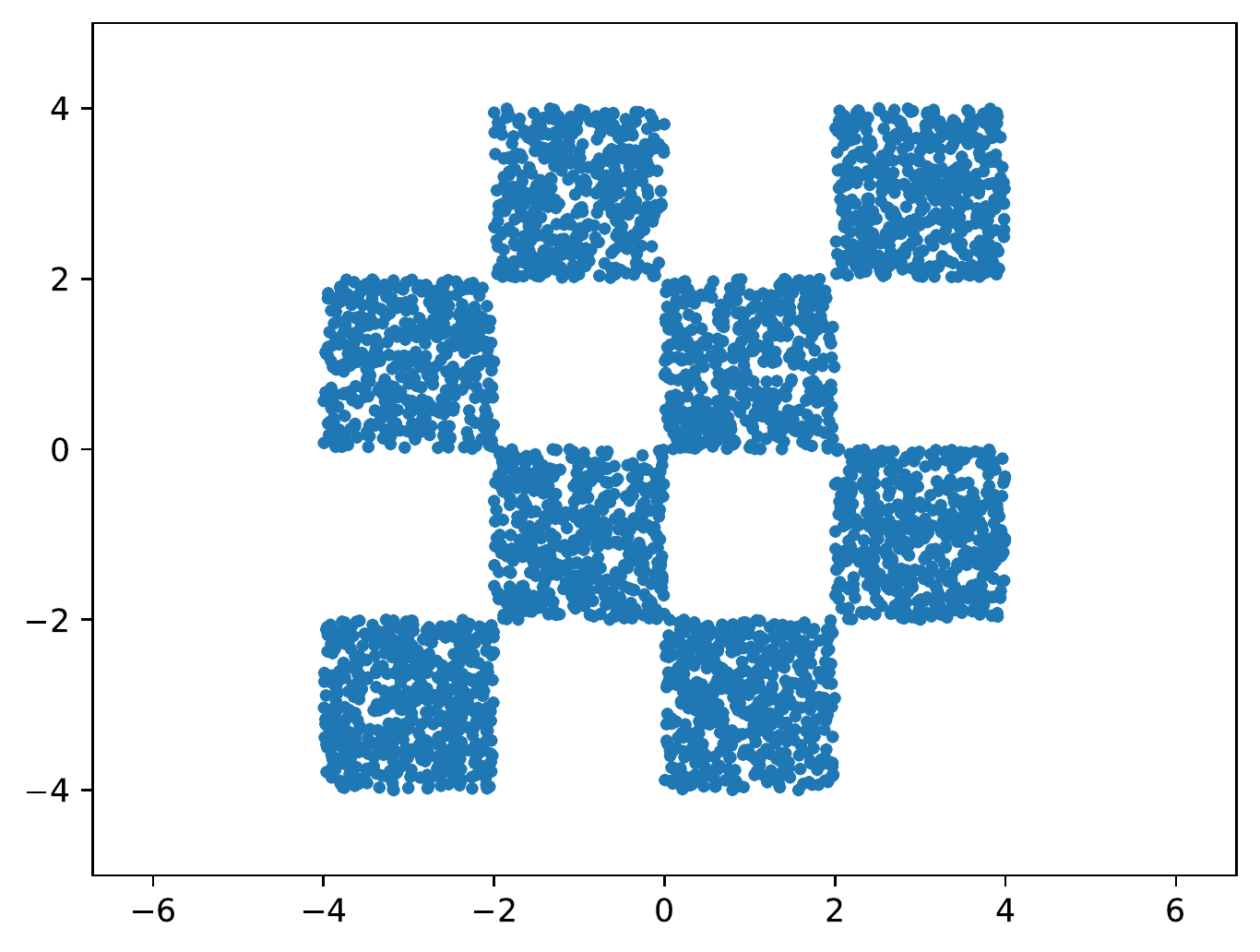}
    checkerboard
    \end{minipage}
\hspace{-\samplehinterval}
\end{minipage}
\caption{Visualization of samples for synthetic problems from ground truth.}
\label{fig:true_synthetic_samples}
\end{figure}

\begin{figure}[t]
\centering
\begin{minipage}{\textwidth}
\centerline{PCD}
    \begin{minipage}{\mpwid\textwidth}
    \centering
    \includegraphics[width=\textwidth]{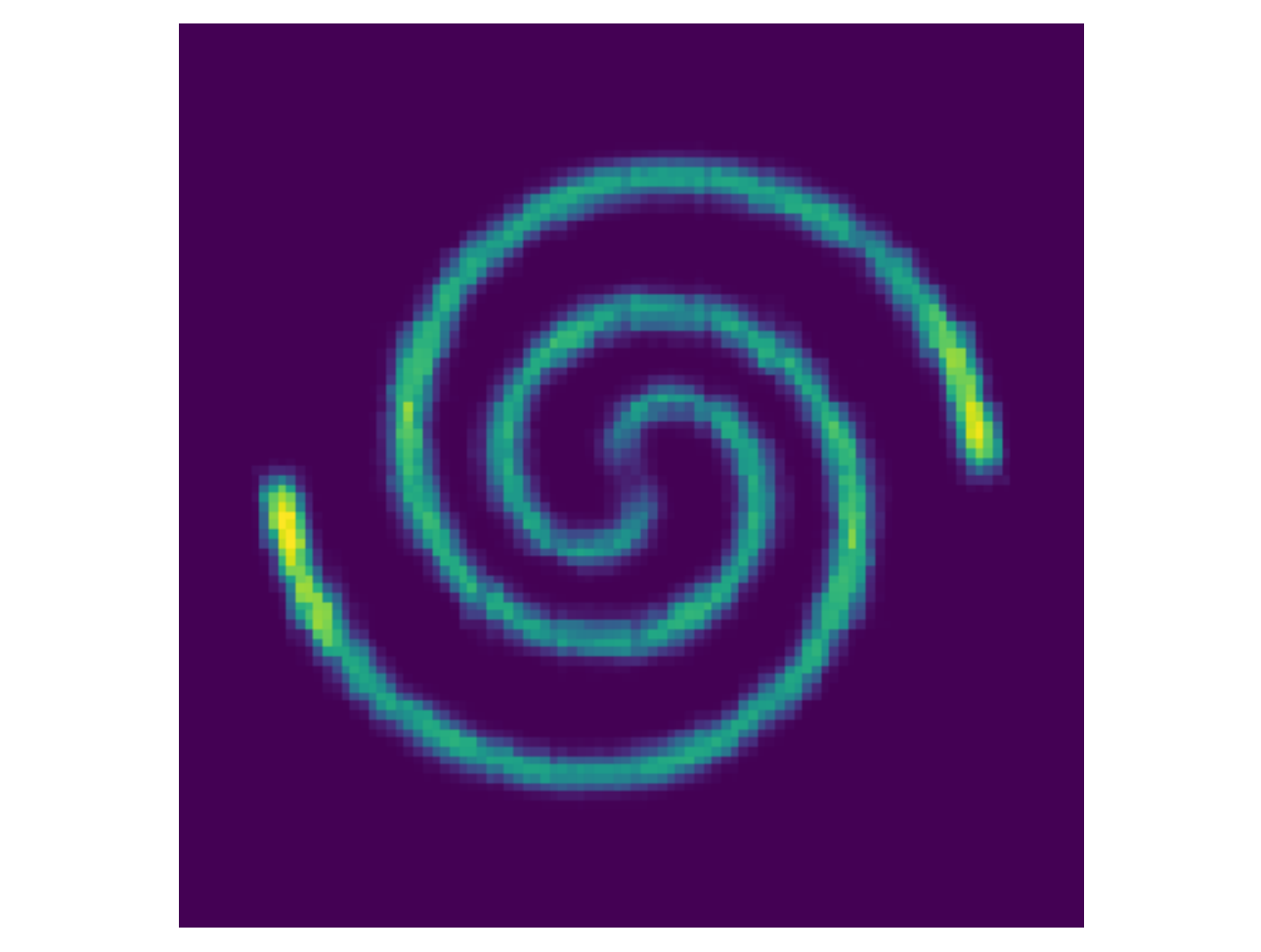}
    \end{minipage}
\hspace{-\hinterval}
    \begin{minipage}{\mpwid\textwidth}
    \centering
    \includegraphics[width=\textwidth]{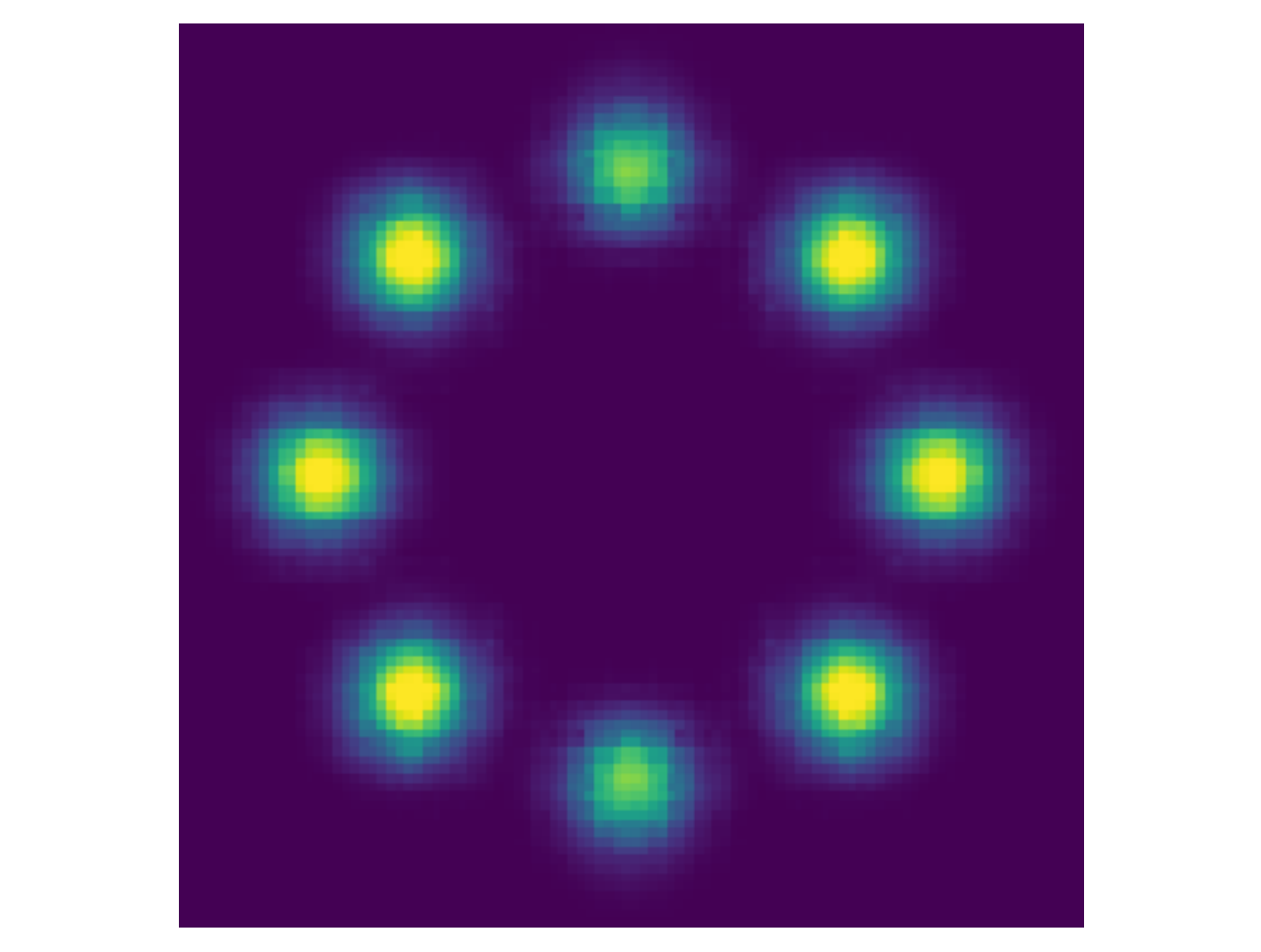}
    \end{minipage}
\hspace{-\hinterval}
    \begin{minipage}{\mpwid\textwidth}
    \centering
    \includegraphics[width=\textwidth]{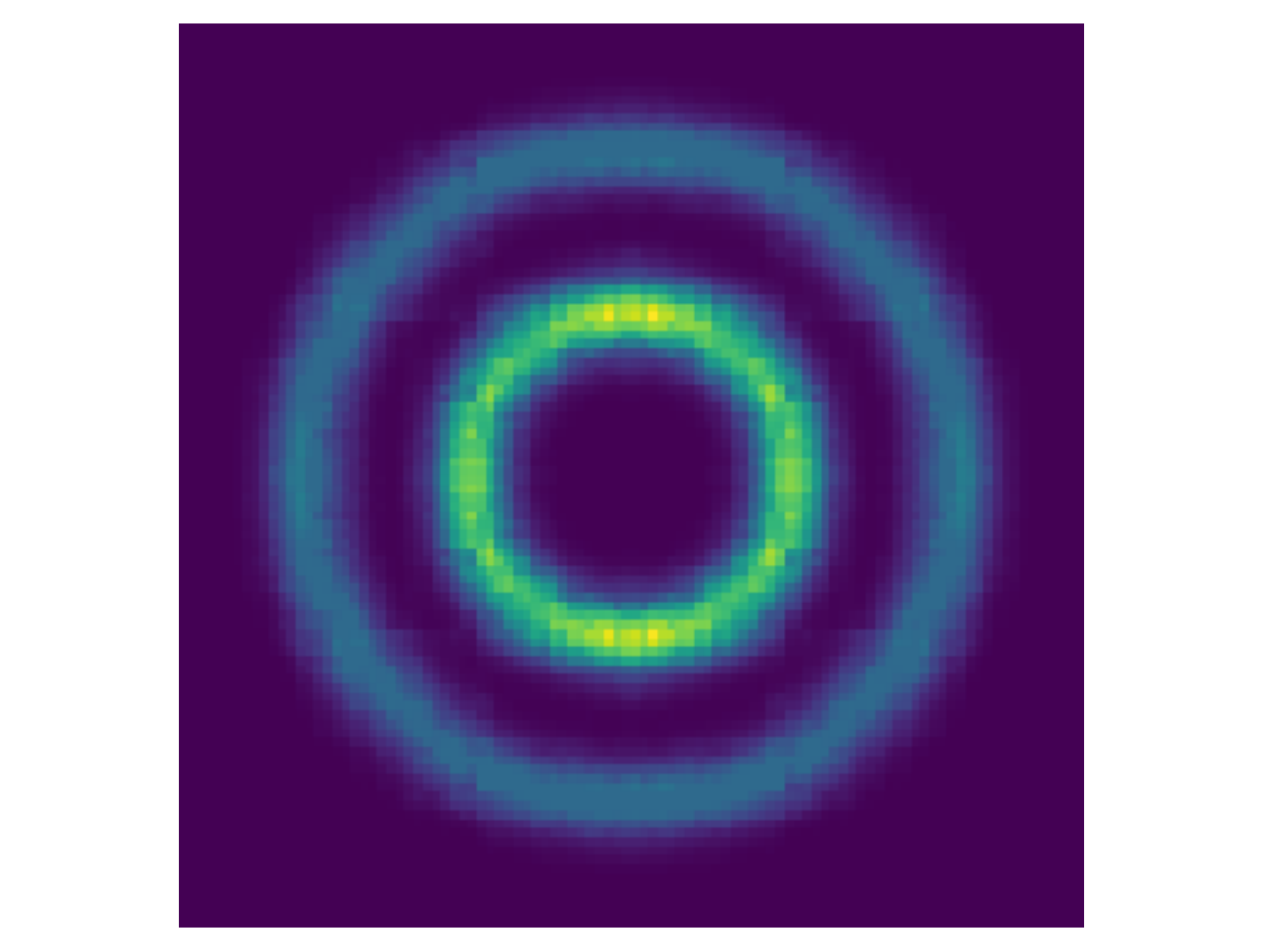}
    \end{minipage}
\hspace{-\hinterval}
    \begin{minipage}{\mpwid\textwidth}
    \centering
    \includegraphics[width=\textwidth]{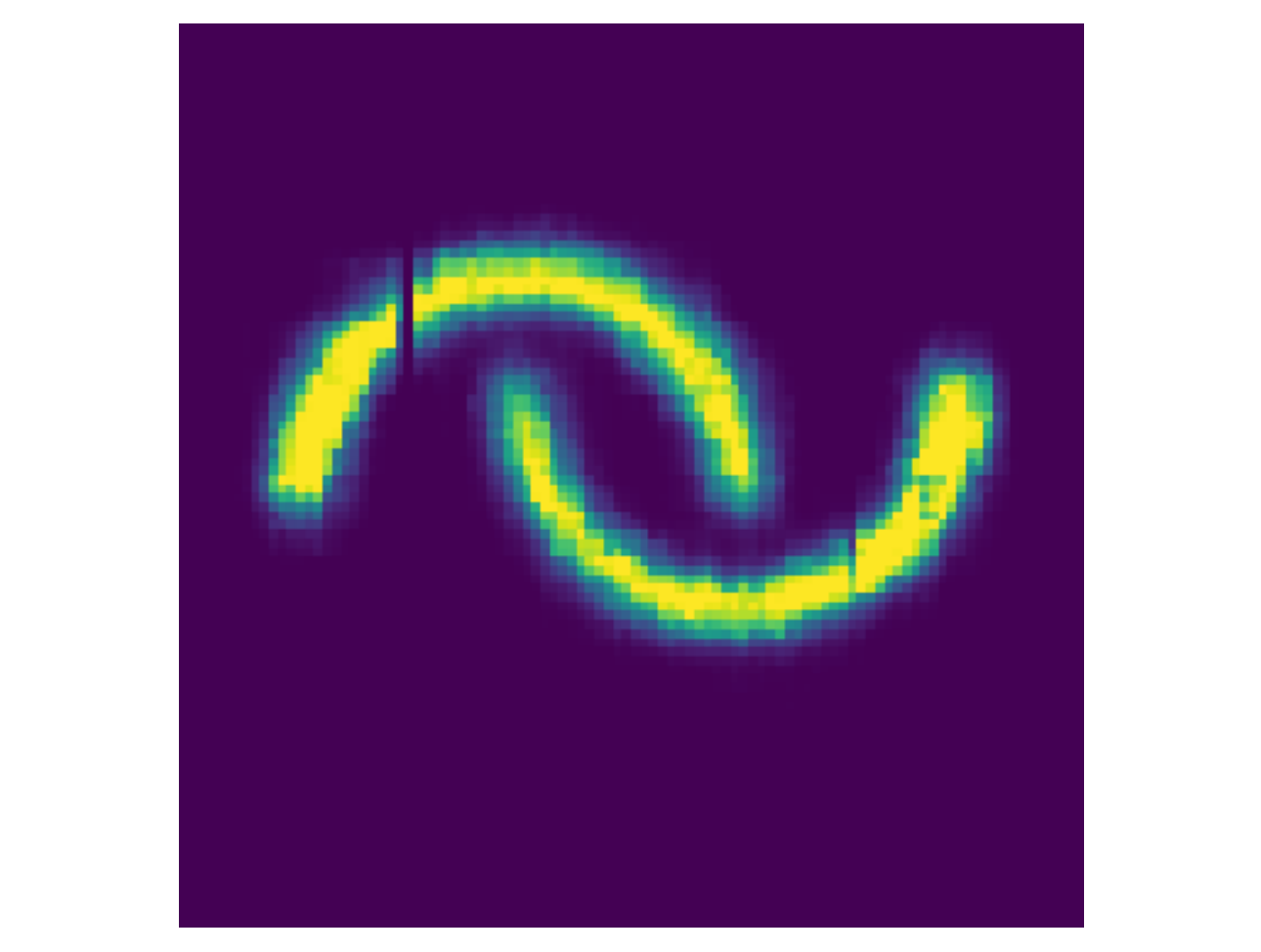}
    \end{minipage}
\hspace{-\hinterval}
    \begin{minipage}{\mpwid\textwidth}
    \centering
    \includegraphics[width=\textwidth]{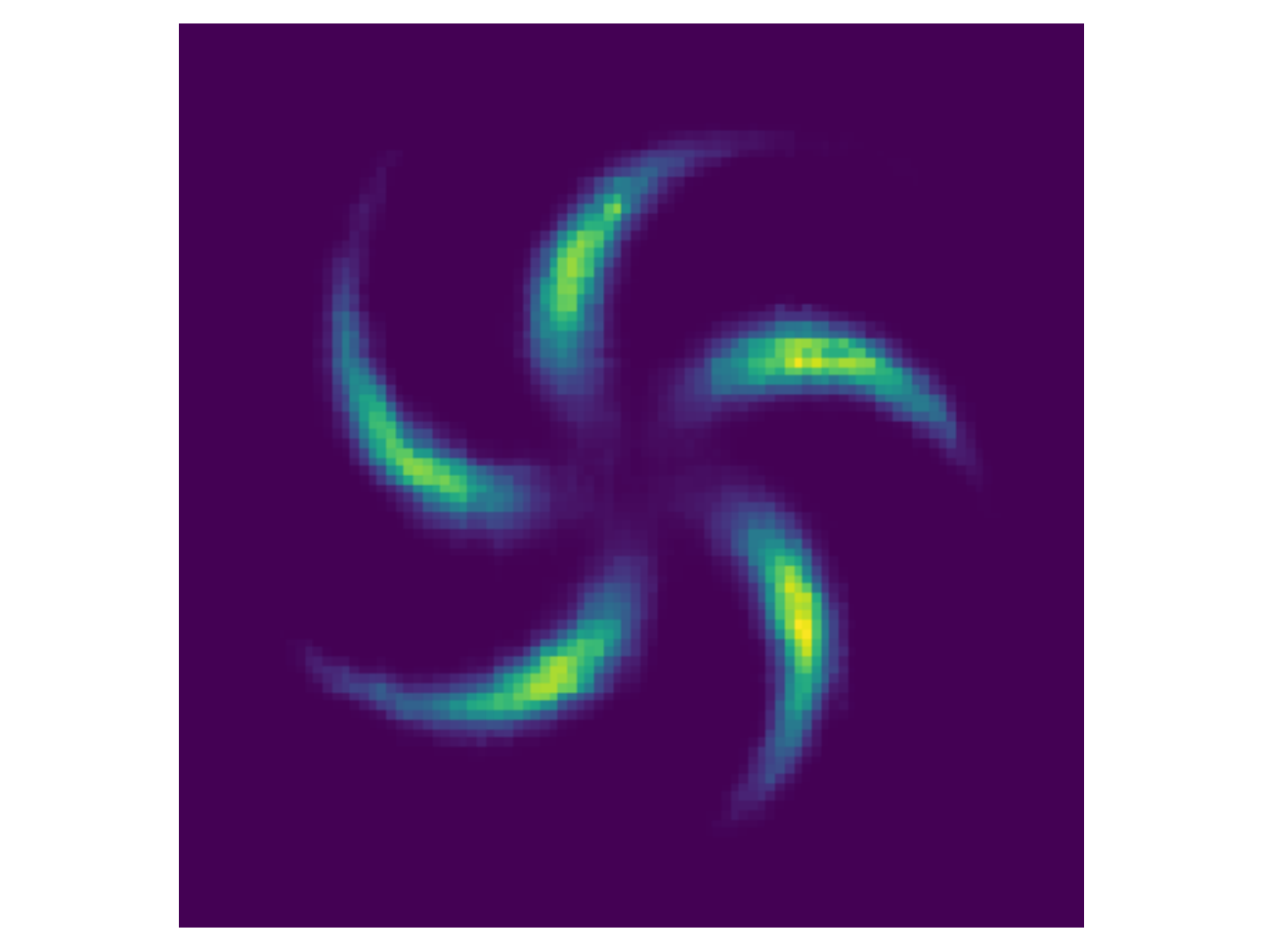}
    \end{minipage}
\hspace{-\hinterval}
    \begin{minipage}{\mpwid\textwidth}
    \centering
    \includegraphics[width=\textwidth]{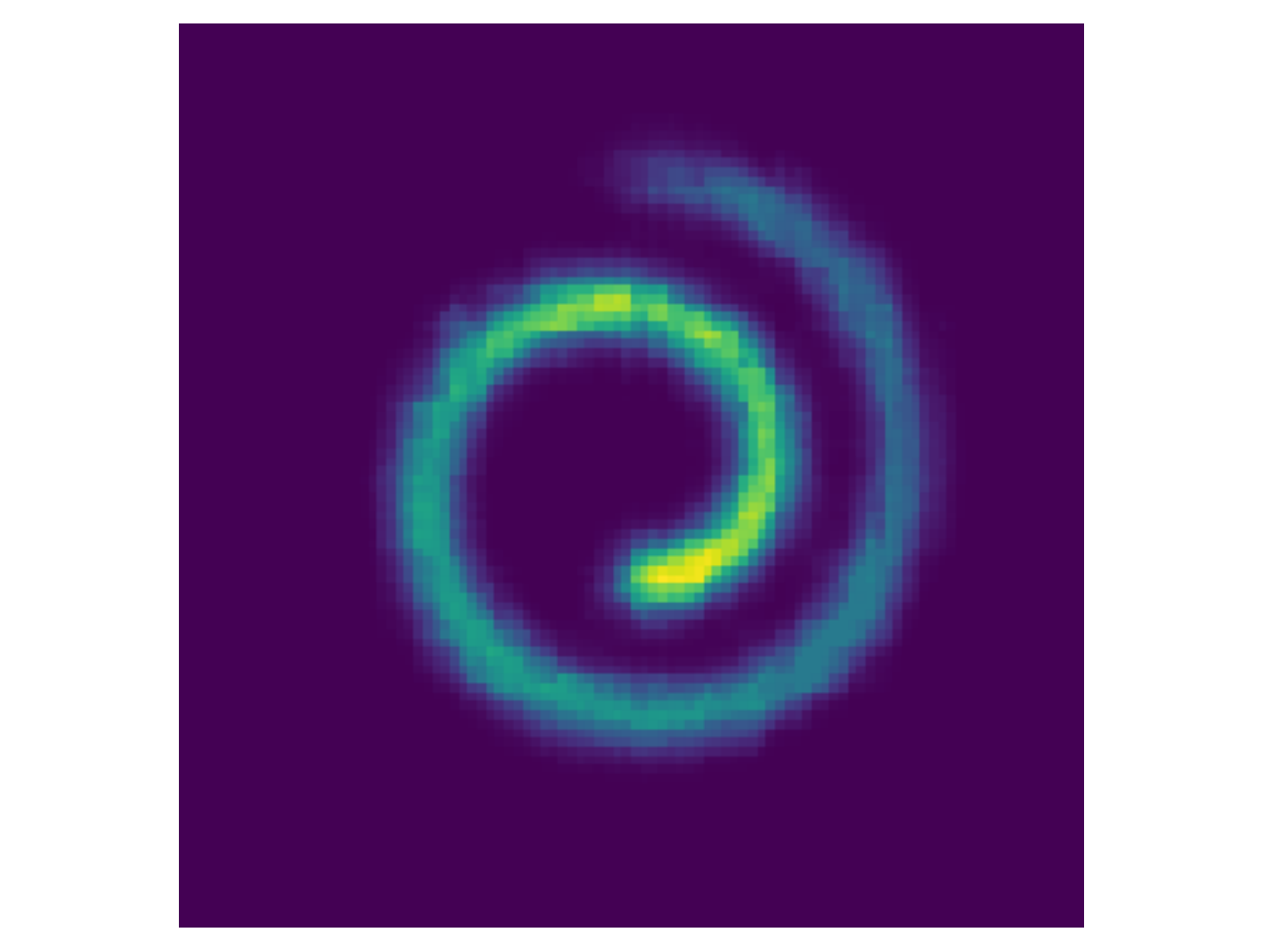}
    \end{minipage}
\hspace{-\hinterval}
    \begin{minipage}{\mpwid\textwidth}
    \centering
    \includegraphics[width=\textwidth]{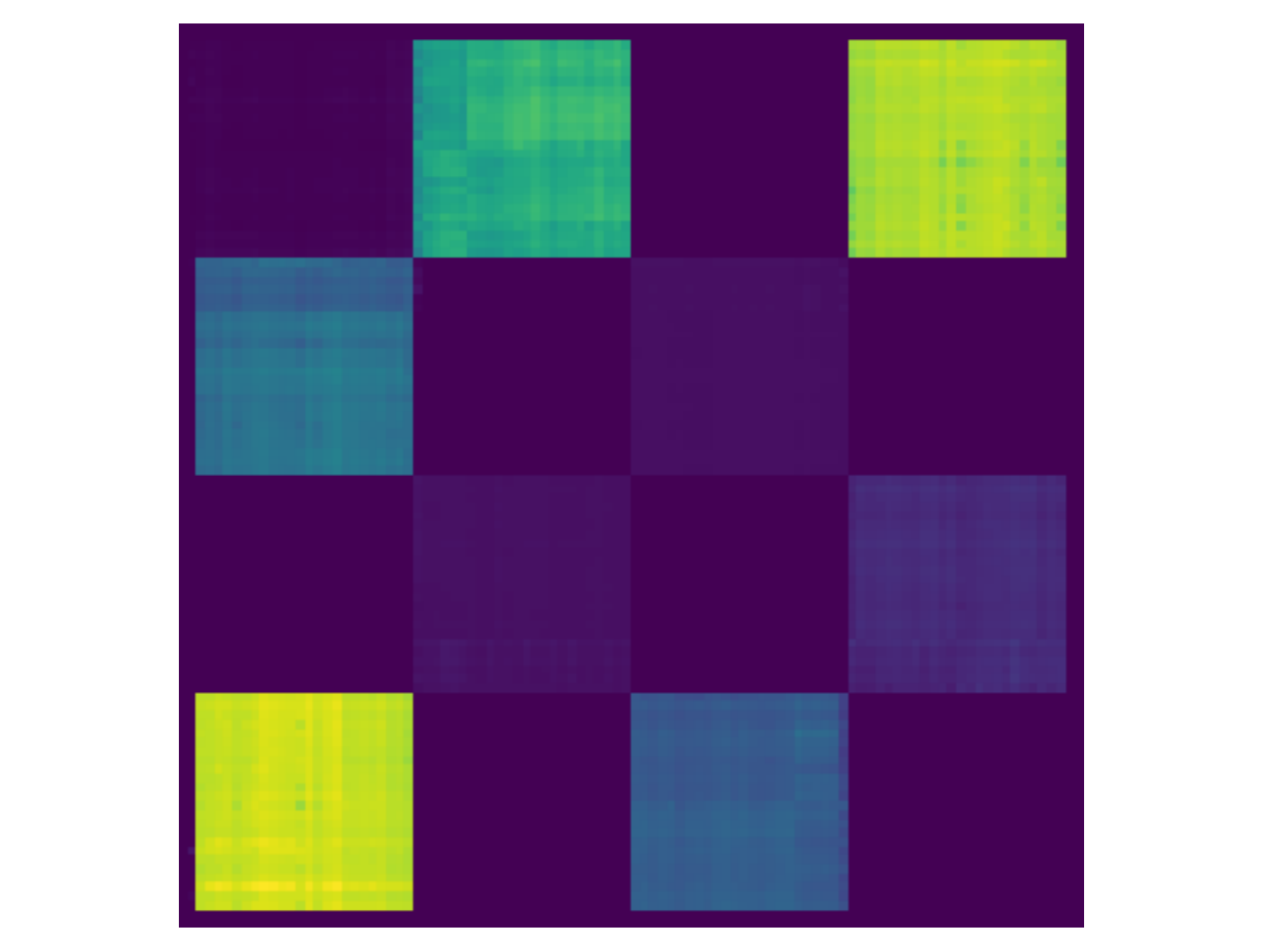}
    \end{minipage}
\hspace{-\hinterval}
\end{minipage}
\begin{minipage}{\textwidth}
\centerline{ALOE}
    \begin{minipage}{\mpwid\textwidth}
    \centering
    \includegraphics[width=\textwidth]{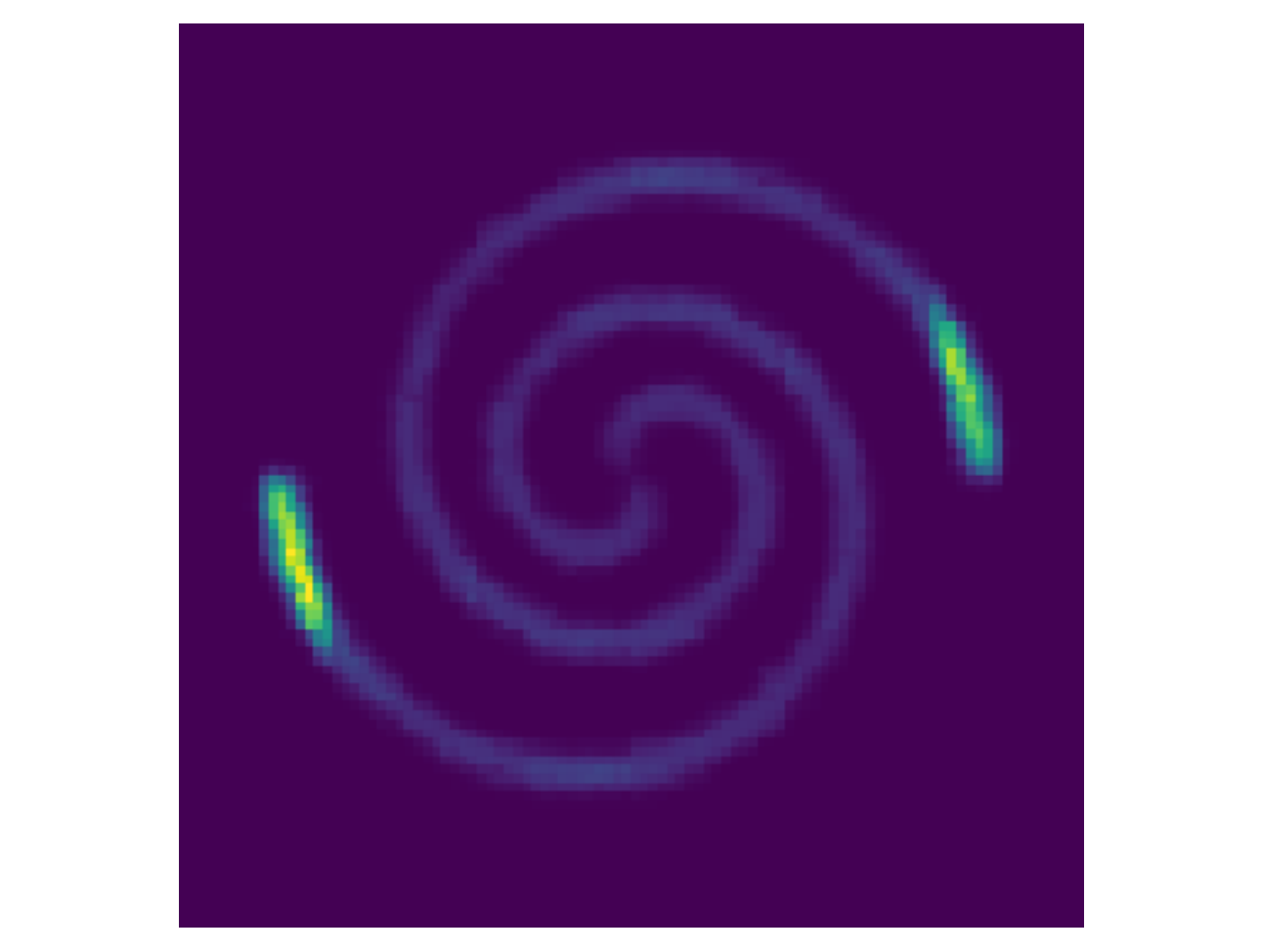}
    \end{minipage}
\hspace{-\hinterval}
    \begin{minipage}{\mpwid\textwidth}
    \centering
    \includegraphics[width=\textwidth]{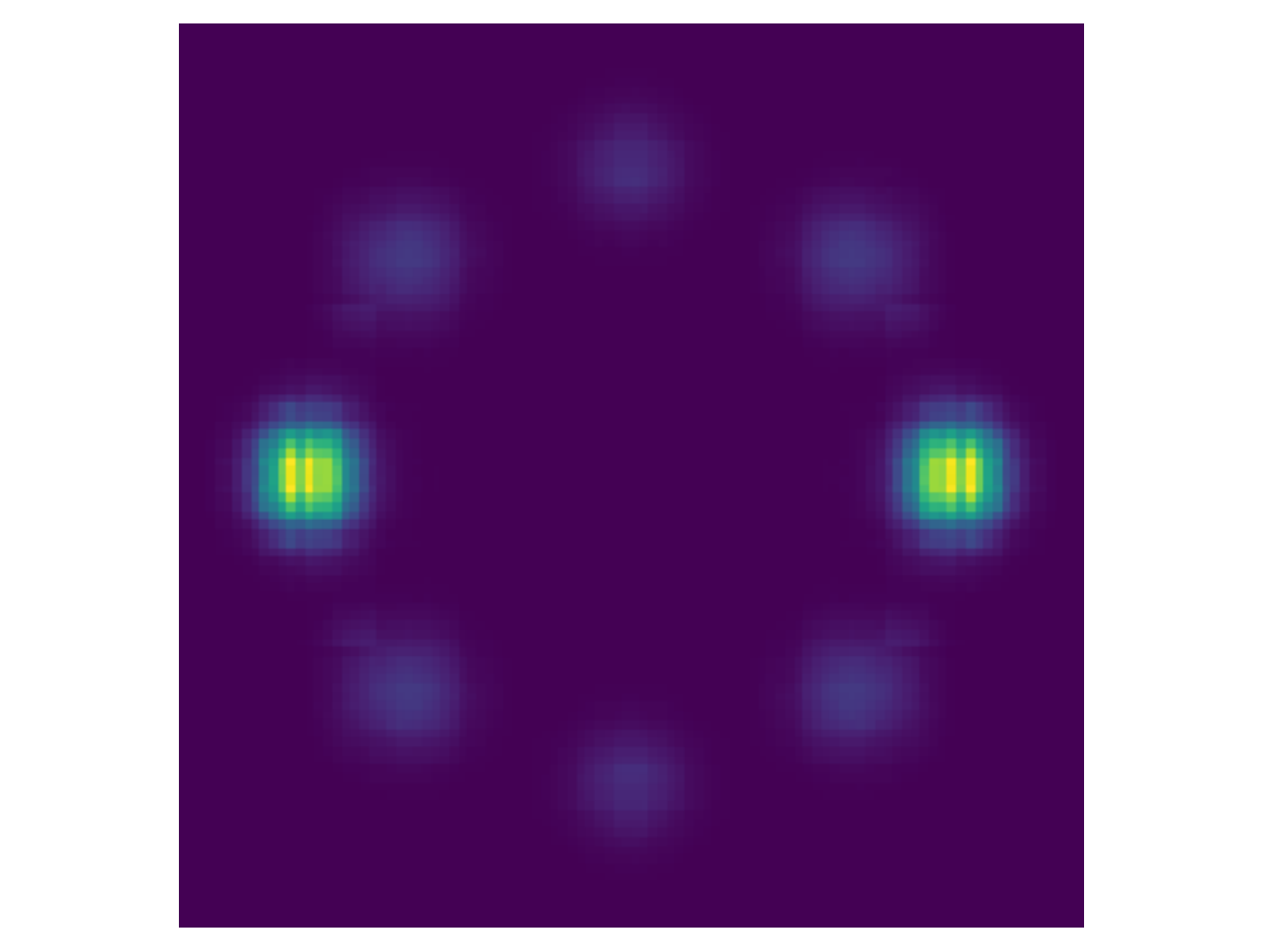}
    \end{minipage}
\hspace{-\hinterval}
    \begin{minipage}{\mpwid\textwidth}
    \centering
    \includegraphics[width=\textwidth]{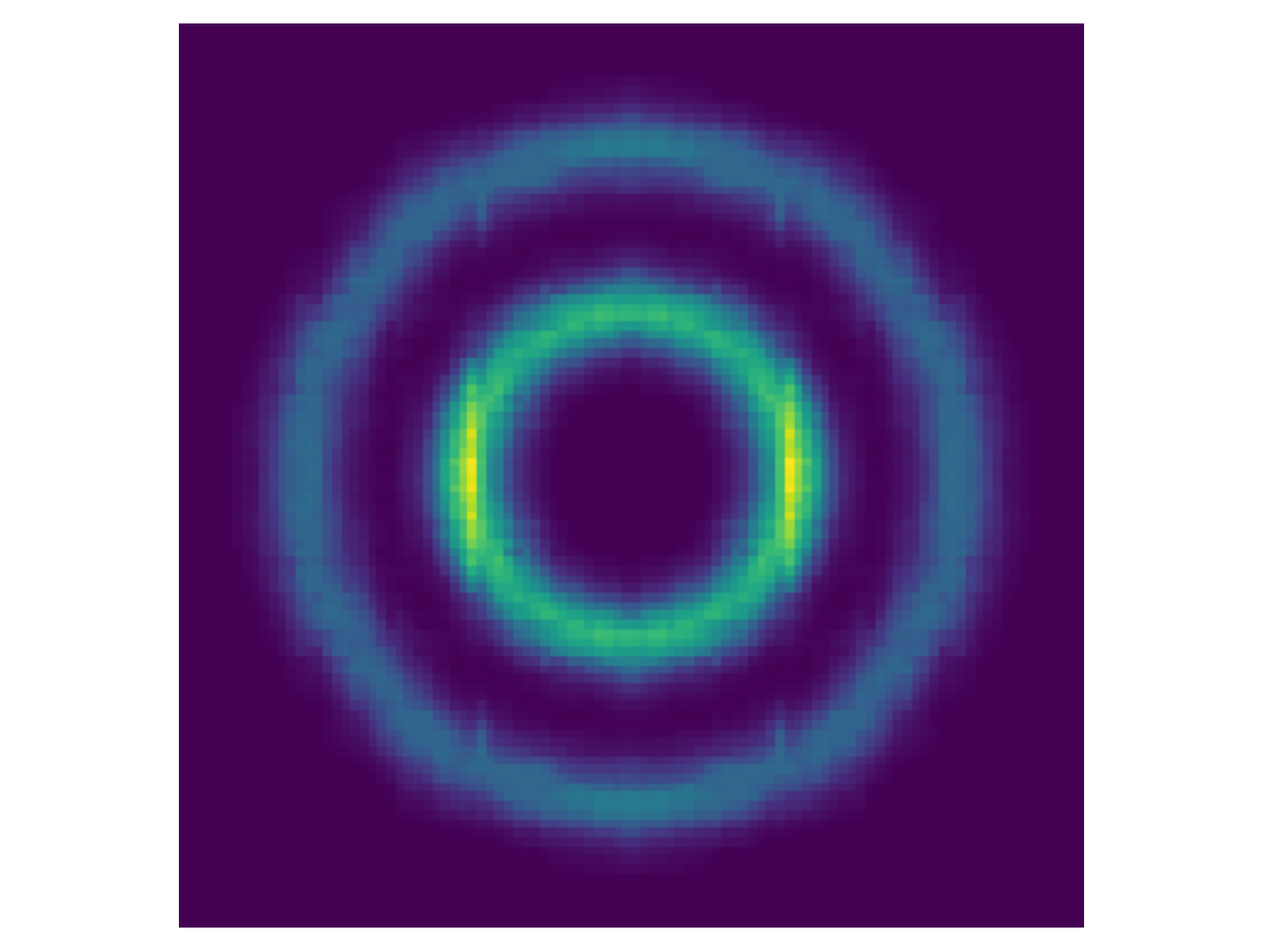}
    \end{minipage}
\hspace{-\hinterval}
    \begin{minipage}{\mpwid\textwidth}
    \centering
    \includegraphics[width=\textwidth]{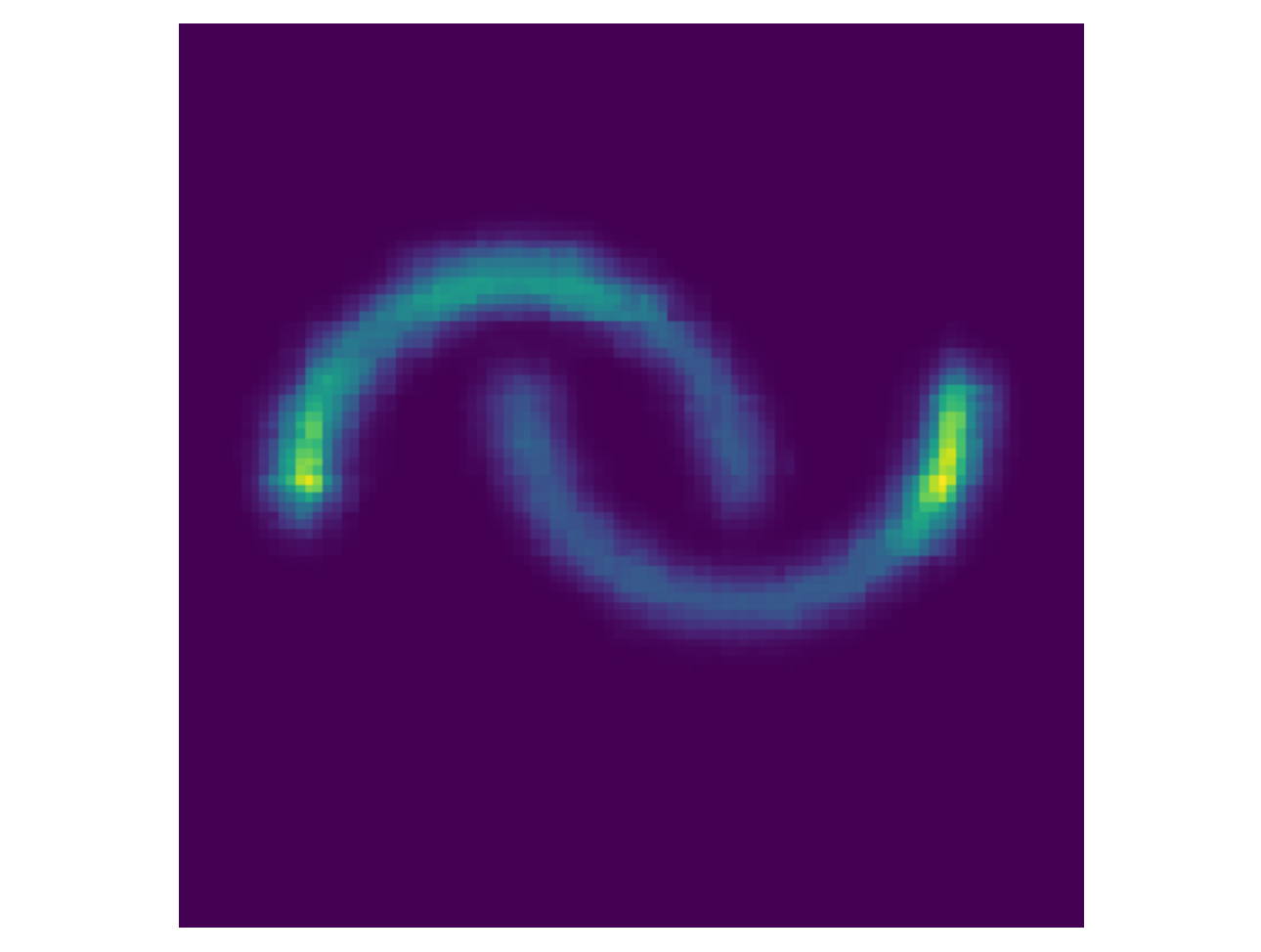}
    \end{minipage}
\hspace{-\hinterval}
    \begin{minipage}{\mpwid\textwidth}
    \centering
    \includegraphics[width=\textwidth]{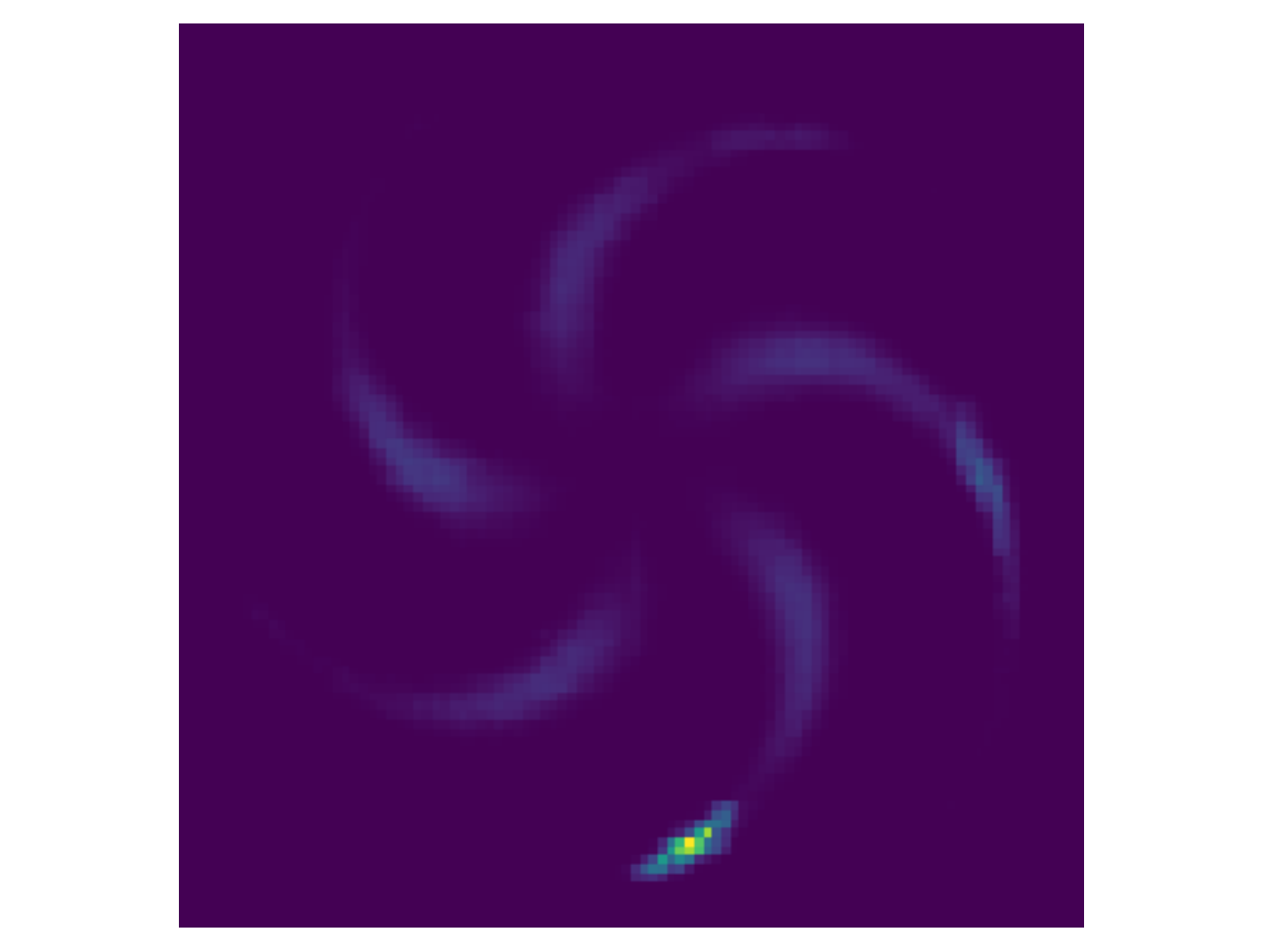}
    \end{minipage}
\hspace{-\hinterval}
    \begin{minipage}{\mpwid\textwidth}
    \centering
    \includegraphics[width=\textwidth]{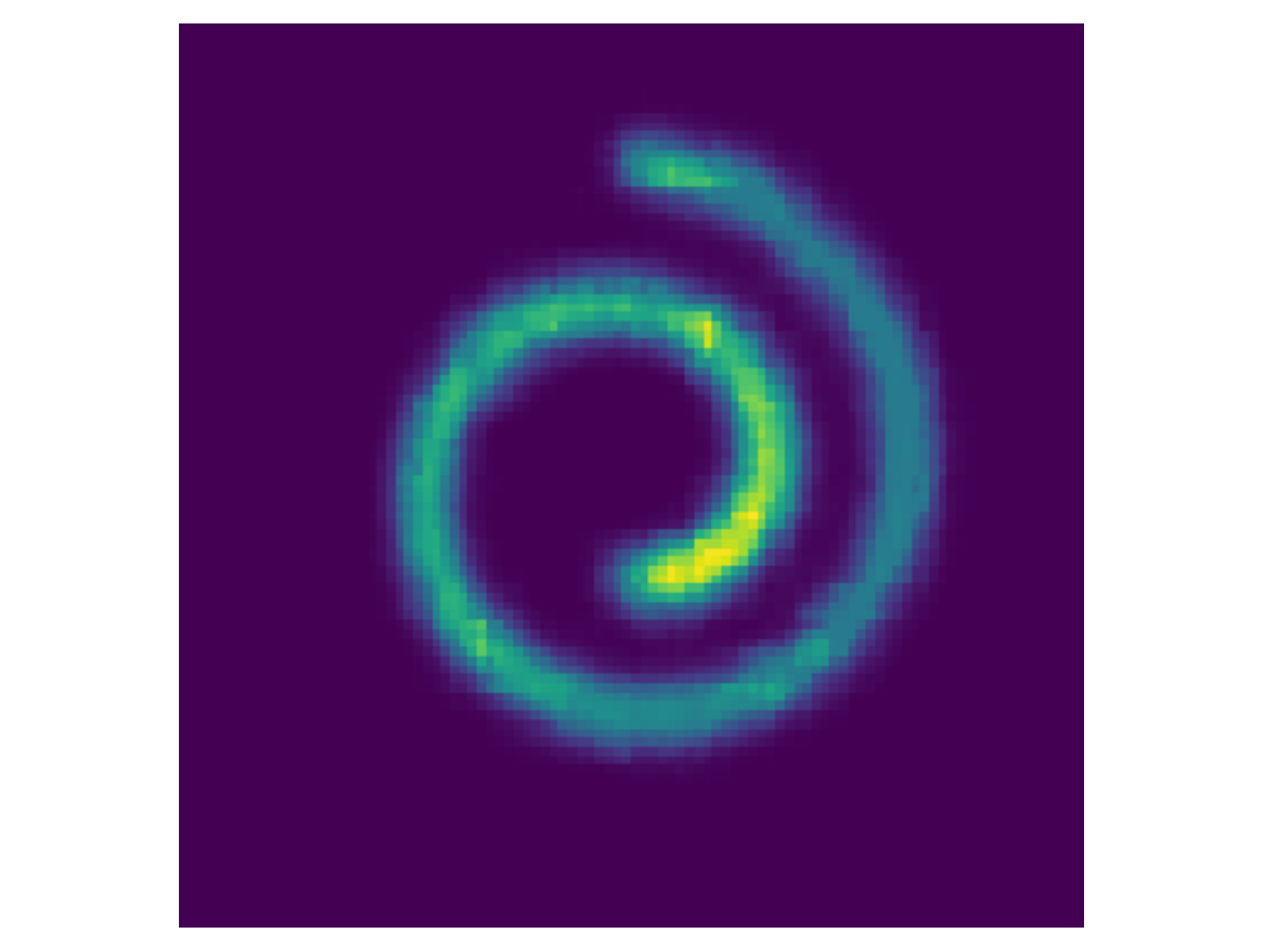}
    \end{minipage}
\hspace{-\hinterval}
    \begin{minipage}{\mpwid\textwidth}
    \centering
    \includegraphics[width=\textwidth]{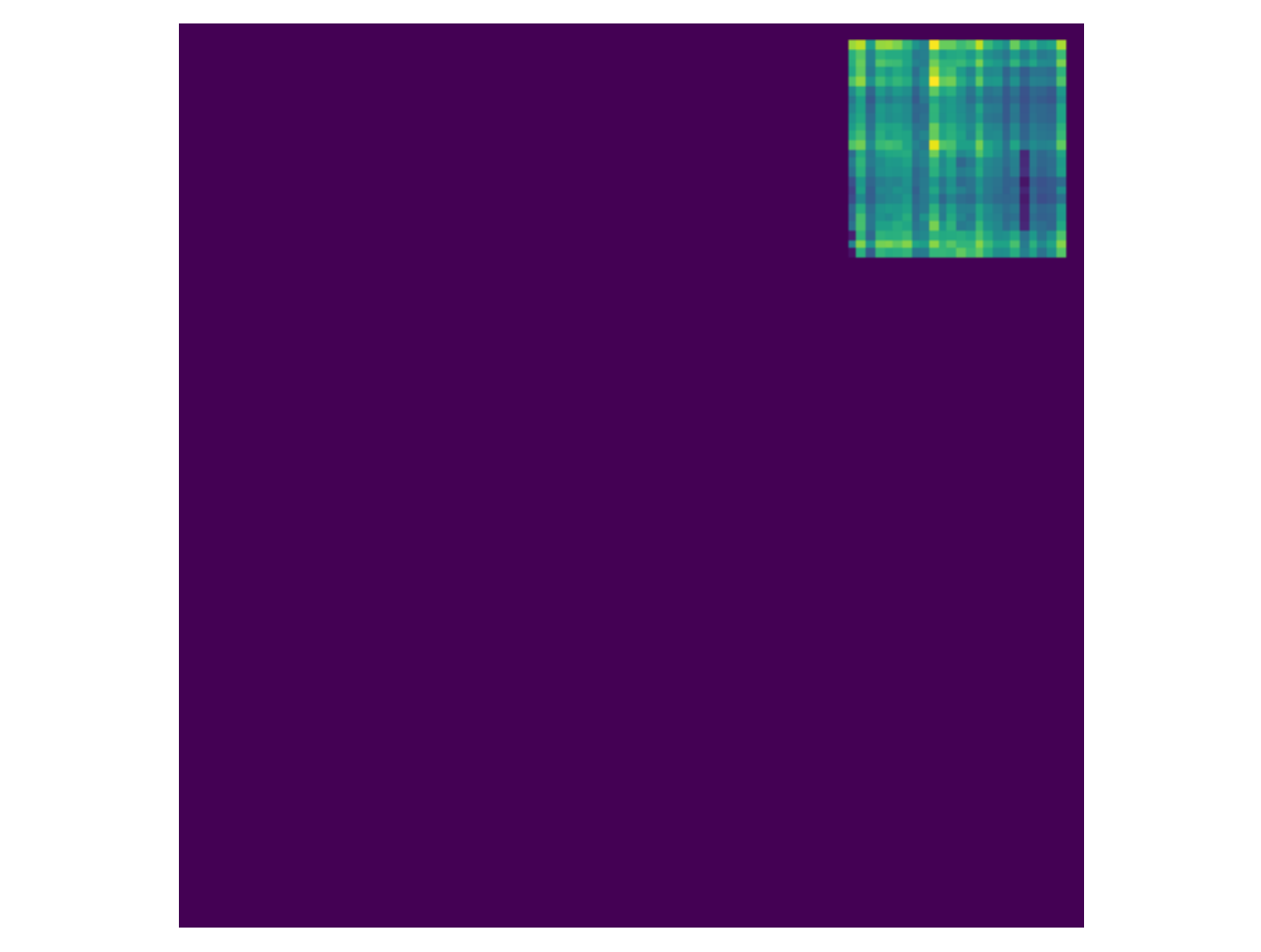}
    \end{minipage}
\hspace{-\hinterval}
\end{minipage}
\begin{minipage}{\textwidth}
\centerline{ALOE+}
    \begin{minipage}{\mpwid\textwidth}
    \centering
    \includegraphics[width=\textwidth]{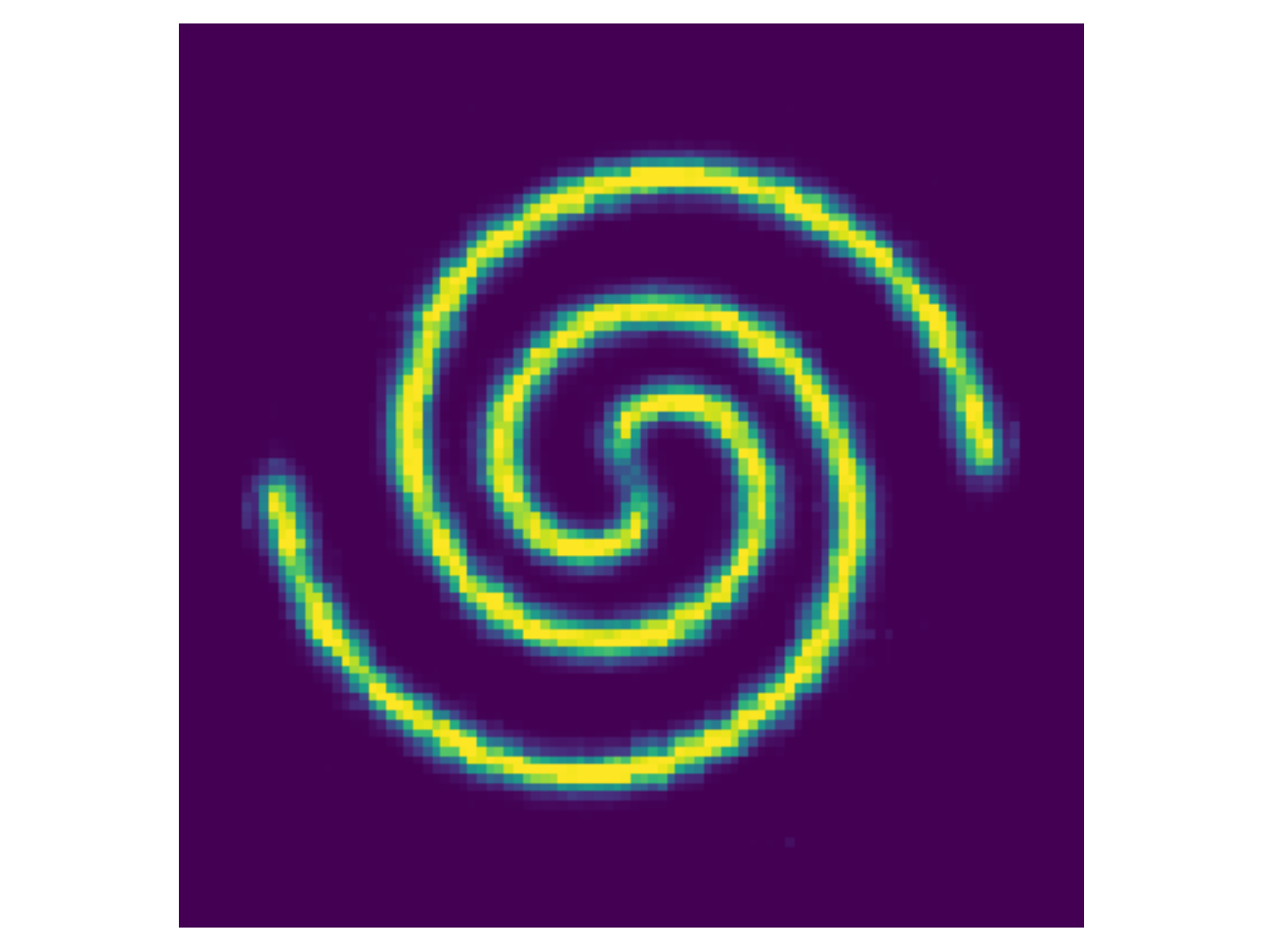}
    2spirals
    \end{minipage}
\hspace{-\hinterval}
    \begin{minipage}{\mpwid\textwidth}
    \centering
    \includegraphics[width=\textwidth]{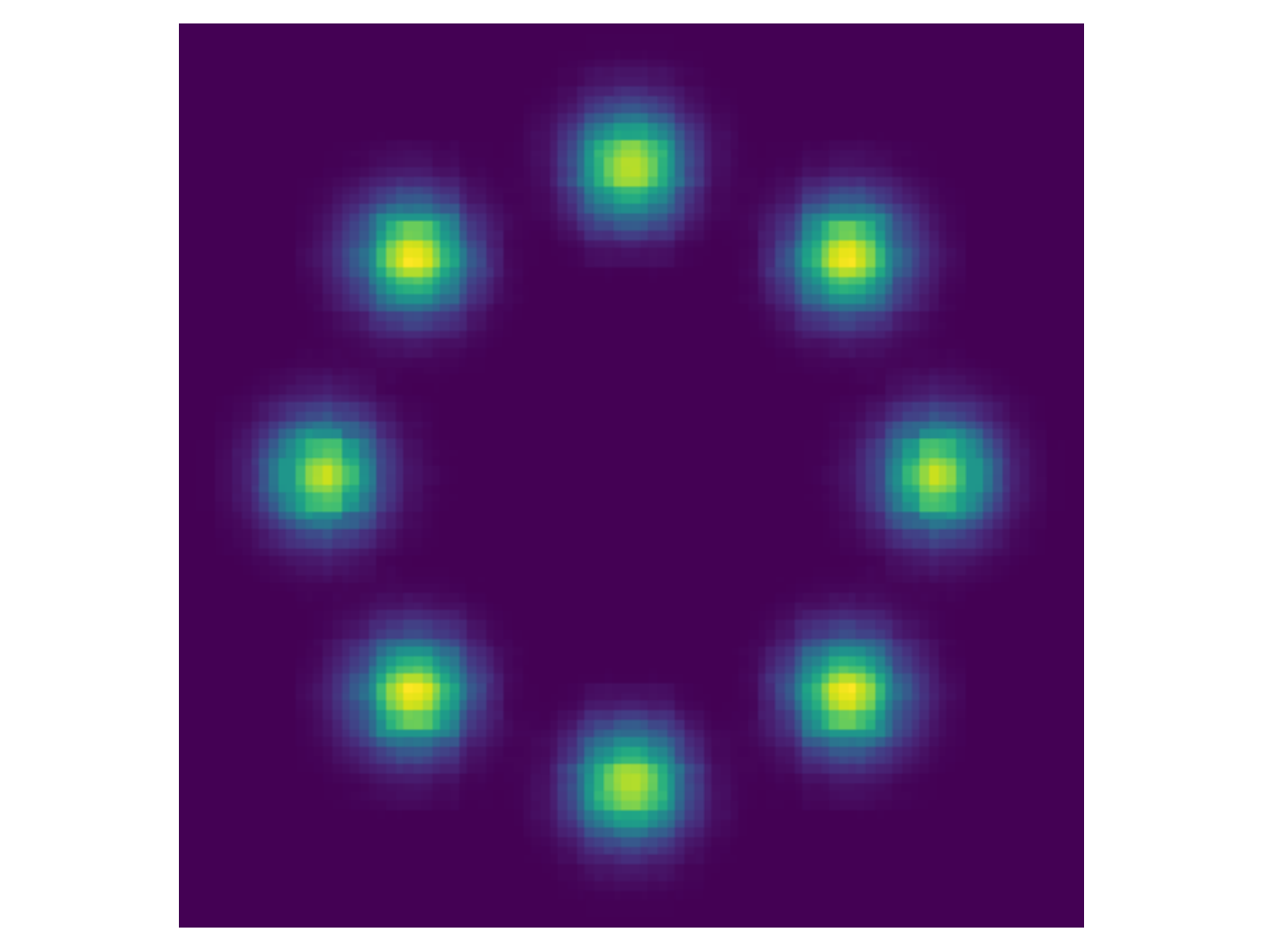}
    8gaussians
    \end{minipage}
\hspace{-\hinterval}
    \begin{minipage}{\mpwid\textwidth}
    \centering
    \includegraphics[width=\textwidth]{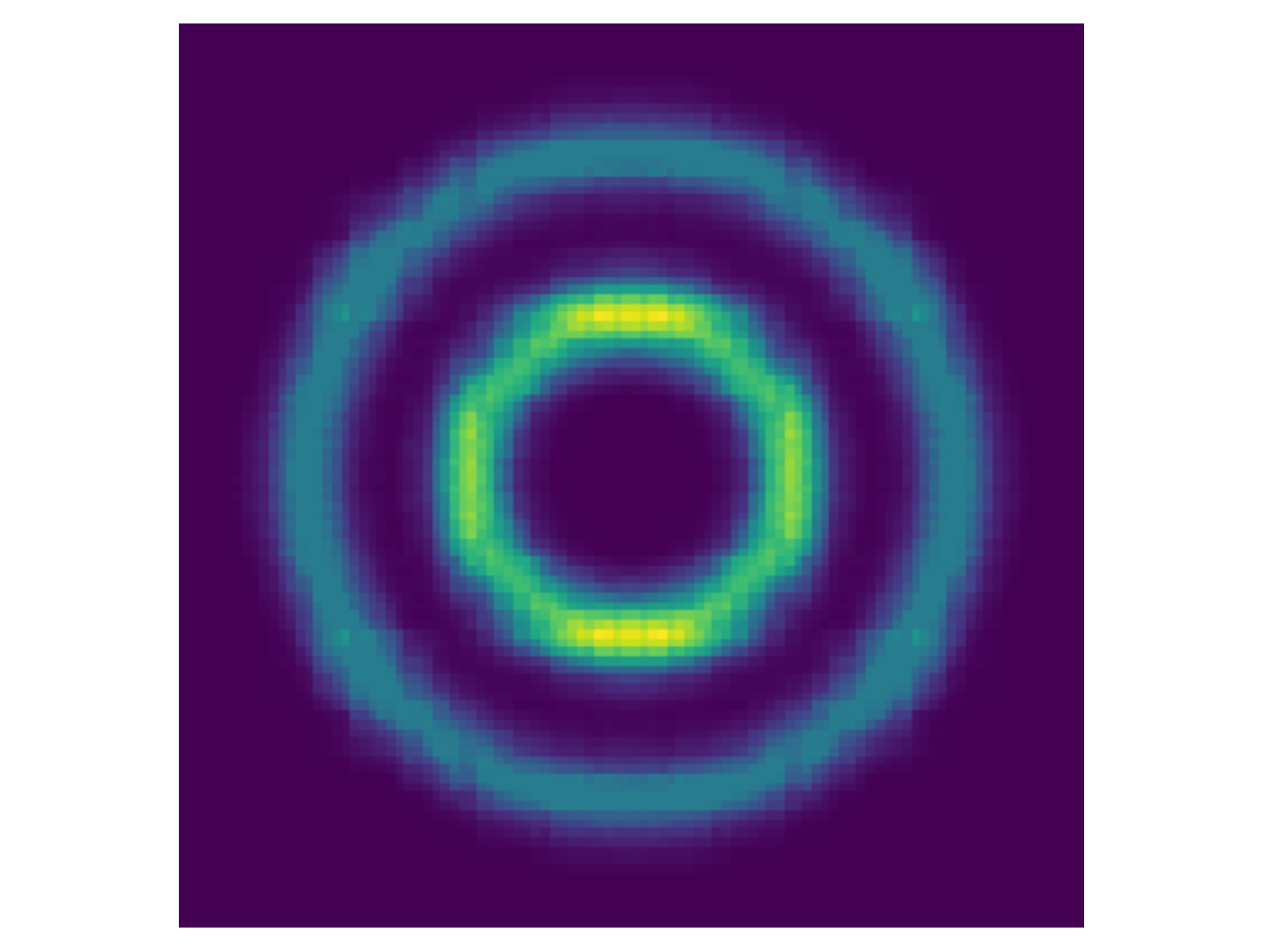}
    circles
    \end{minipage}
\hspace{-\hinterval}
    \begin{minipage}{\mpwid\textwidth}
    \centering
    \includegraphics[width=\textwidth]{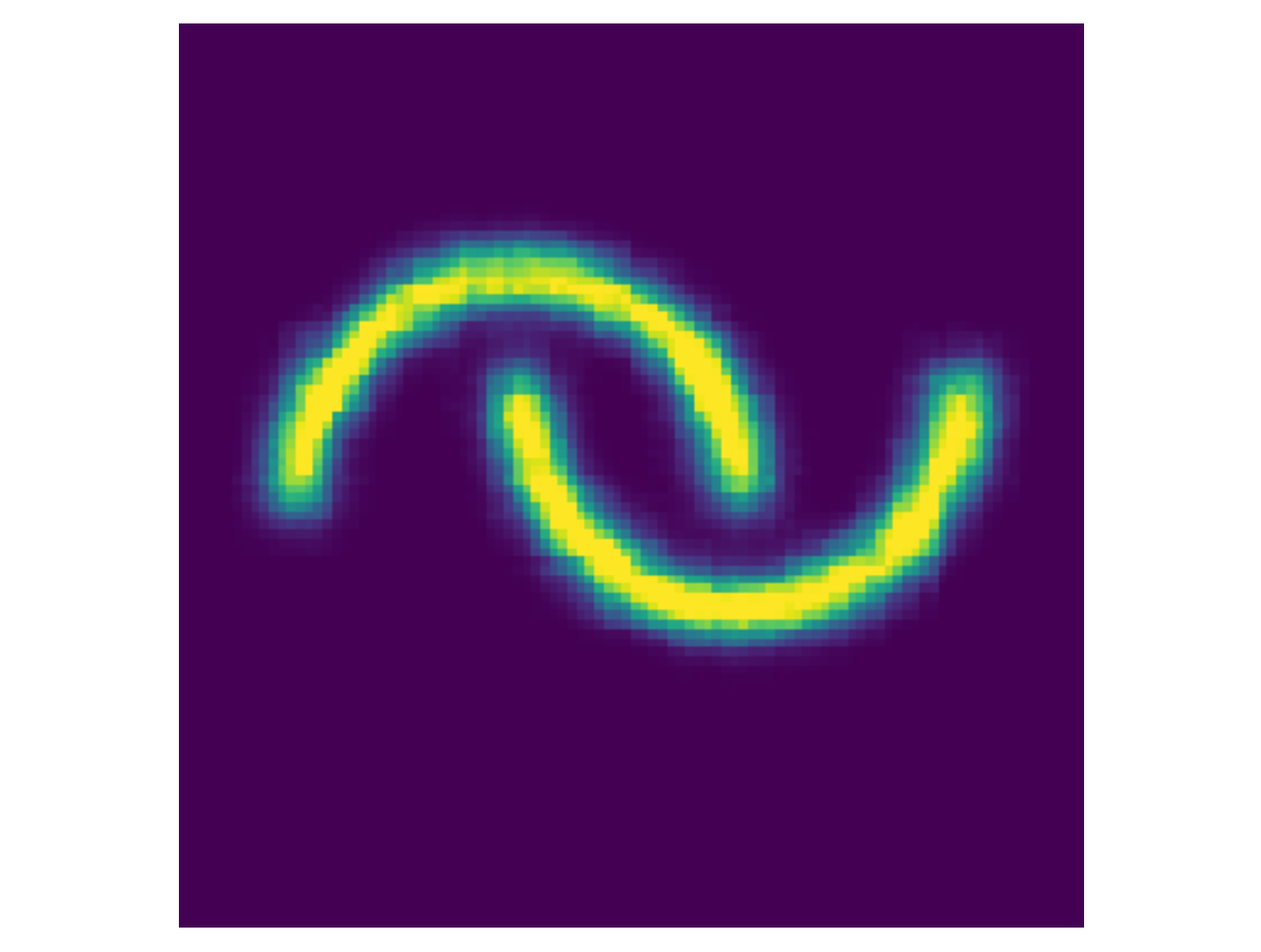}
    moons
    \end{minipage}
\hspace{-\hinterval}
    \begin{minipage}{\mpwid\textwidth}
    \centering
    \includegraphics[width=\textwidth]{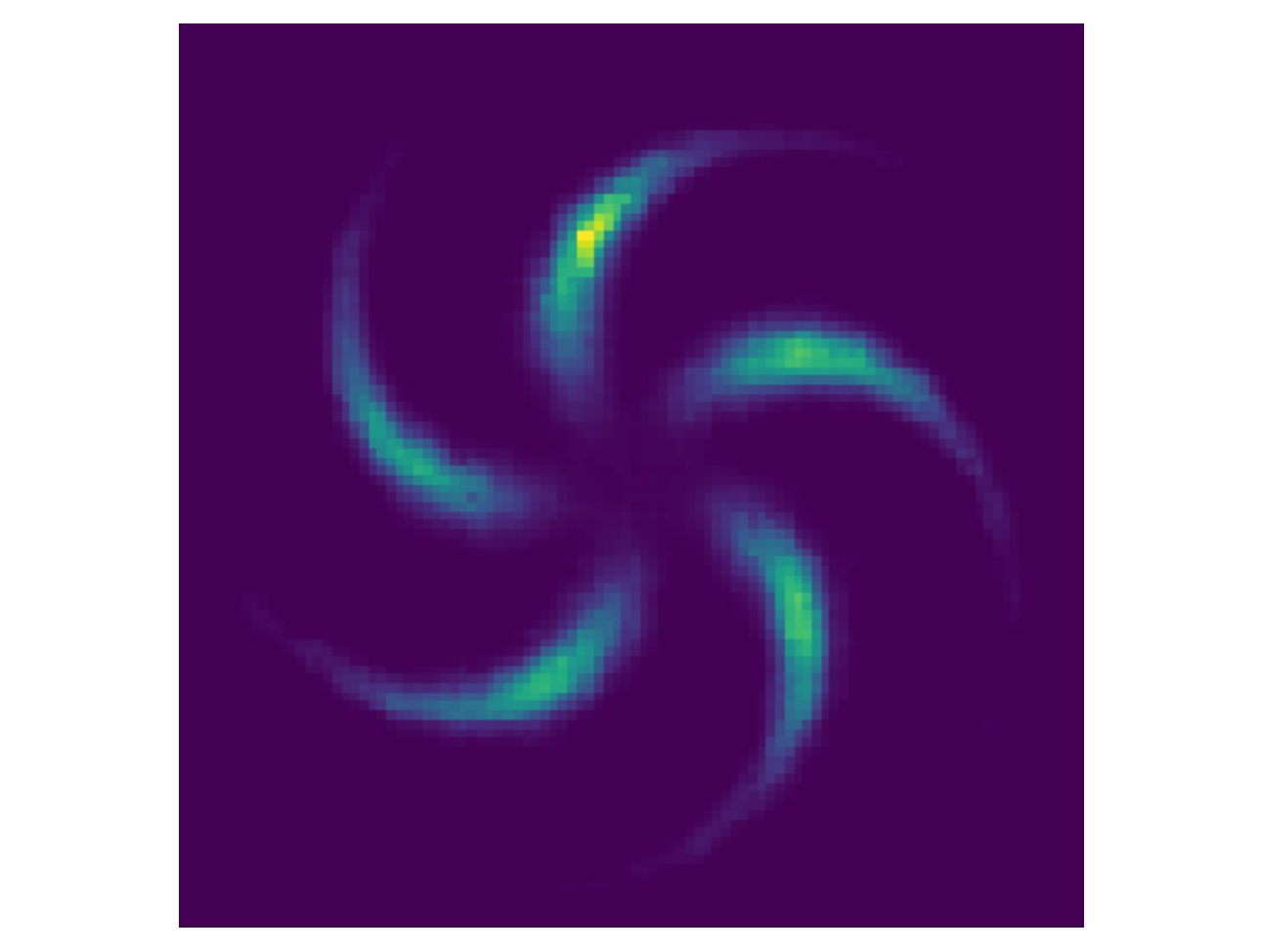}
    pinwheel
    \end{minipage}
\hspace{-\hinterval}
    \begin{minipage}{\mpwid\textwidth}
    \centering
    \includegraphics[width=\textwidth]{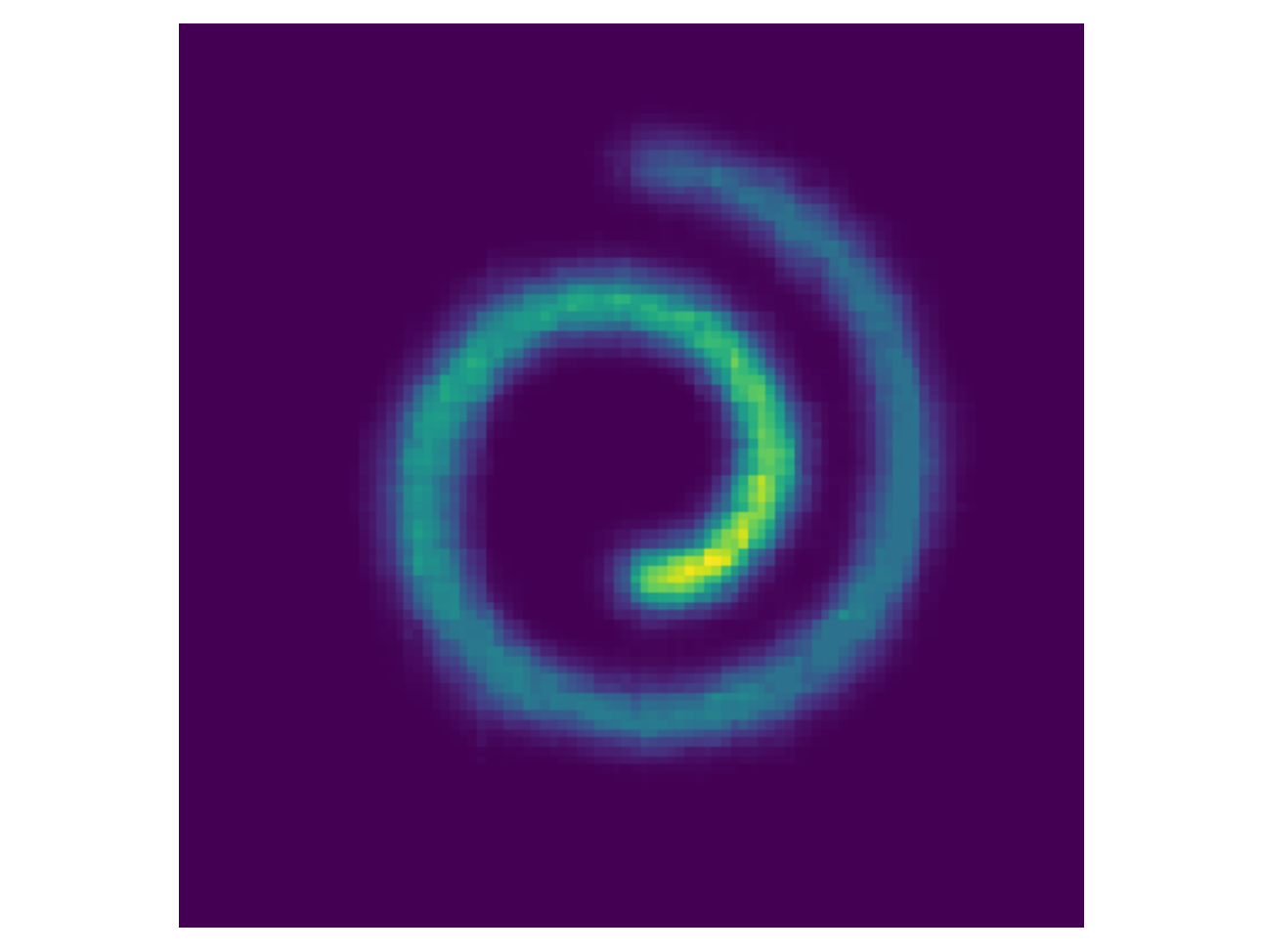}
    swissroll
    \end{minipage}
\hspace{-\hinterval}
    \begin{minipage}{\mpwid\textwidth}
    \centering
    \includegraphics[width=\textwidth]{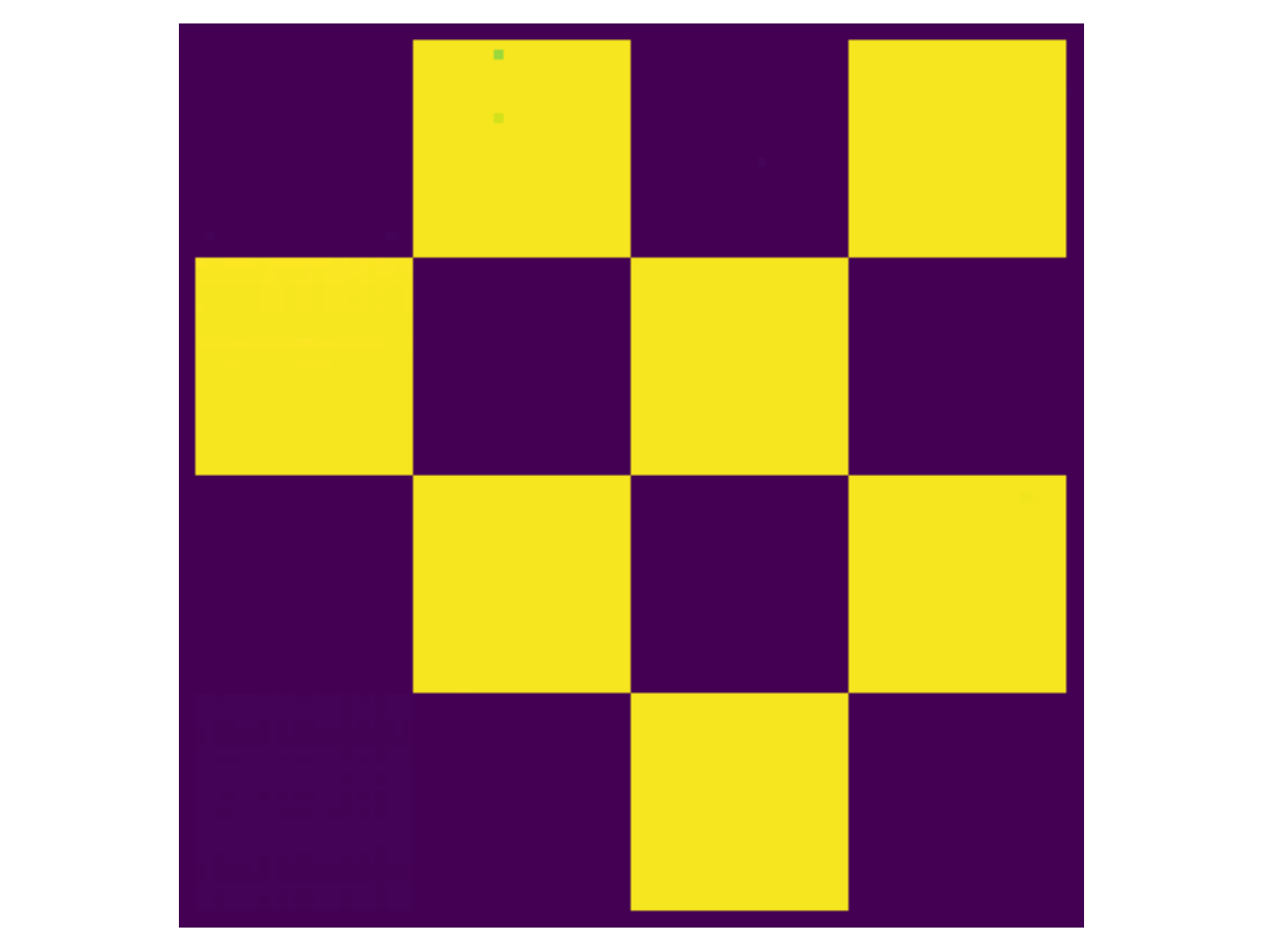}
    checkerboard
    \end{minipage}
\hspace{-\hinterval}
\end{minipage}
\caption{Visualization of the learned energy function from different baseline methods.}
\label{fig:baseline_energy}
\end{figure}

This synthetic task is also adopted by previous works such as \citet{Grathwohl2019FFJORDFC, Dai2020ALOE}.
We keep a consistent setting with \citet{Dai2020ALOE} unless specified.
The data is first generated by an infinite data oracle as 2D floating-points values. They are then turned into 16-bit Gray code.
This problem is challenging itself even without the existence of Gray code transformation, which is highly nonlinear.
We use a $4$ layer MLP with $256$ hidden dimension and ELU activation \citep{Clevert2016FastAA} as the energy function. 
The training of this energy function lasts $10^{5}$ steps.
An Adam optimizer with $1\times 10^{-3}$ learning rate is used to update the EBM. 
The batch size is $128$.
For PCD baseline, we use $10$ steps Gibbs sampling to generate negative samples, and choose the best results from three different replay buffer re-initialization rate: $\{0.05, 0.1, 1\}$.
For ALOE and ALOE+, we keep the same configuration as in \citet{Dai2020ALOE}, where the former uses a naive $32$-dim multinomial initial distribution proposal, and the latter use an autoregressive model which contains $32$ MLPs that each has three layers with $512$ hidden dimension.
Both methods have an editor network and stop policy network, either of which is three layer MLP with $512$ hidden feature.
For the EB-GFN algorithm, the policy network is a similar three layer MLP (the forward policy and backward policy share the same first layers and differ in the last layer),  and the output dimension is $3\times 32=96$.
The GFlowNet is optimized with an Adam optimizer, where the learning rate is $1\times 10^{-3}$.
We use an equal mix of $P_F(\tau)$ and $P_B(\tau|\x)$ to generate training trajectories for trajectory balance objective (\ie, set $\alpha=0.5$ in Algorithm~\ref{alg:gfn_training}).
For the back-and-forth proposal, we set $K$ to linearly increase from $1$ to $D$ through the training process.
The NLL computation of GFlowNet follows the method described in \S\ref{sec:gfn_training}.
We set the value of $M$ to be 100, which is large enough to converge 
(as a reference, 
for checkerboard dataset, the NLL is 
$20.695967$ when $M=10$, 
$20.695692$ when $M=50$, 
$20.695967$ when $M=100$, 
$20.695583$ when $M=500$, 
$20.695490$ when $M=1000$).
The number of samples is set to $10^5$, which is also enough for convergence in a similar sense.

To help better understanding the oracle of this task,
we visualize the ground truth samples in Fig.~\ref{fig:true_synthetic_samples}.
We can see that EB-GFN could generate samples very close to these true data.
We also plot the visualization of the baselines' energy function in Figure~\ref{fig:baseline_energy}.
It demonstrates that ALOE actually has a hard time modeling multimode distribution without the help of a large initial proposal model, as ALOE+ does.

\begin{table*}[t]
\setlength{\tabcolsep}{2.5mm}
\centering
\caption{
Experiment results with seven 2D synthetic problems.
We display the MMD with linear kernel (in units of $1\times 10^{-4}$).
Notice that ALOE+ uses a thirty times larger parametrization than ALOE and EB-GFN.
}
\label{tab:linear_mmd}
\begin{tabular}{c|l|ccccccc}
\toprule
Metric & Method & 2spirals & 8gaussians & circles & moons & pinwheel & swissroll & checkerboard \\
\midrule
\multirow{4}{*}{MMD$\downarrow$} 
& PCD    &  $48.22$&$19.28$&$6.029$&$18.72$&$10.48$&$29.45$& $64.05$ \\
& ALOE   &  $464.8$ & $2240$ & $10.18$ & $557.1$ & $810.4$ & $\textbf{5.001}$ & $1376$ \\
& ALOE+  &  $\textbf{3.802}$ & $\textbf{1.647}$ & $12.54$ & $9.181$ & $32.72$ & $15.01$ & $264.1$ \\
& EB-GFN &  $9.705$ & $10.22$ & $\textbf{4.056}$ & $\textbf{1.899}$ & $\textbf{1.298}$ & $10.23$ & $\textbf{27.18}$\\
\bottomrule
\end{tabular}
\end{table*}

In \citet{Dai2020ALOE}, the authors mentioned using Hamming kernel MMD, while in their public code linear MMD is adopted.
Further, in their public implementation, the MMD result is calculated within a fixed group of 4000 samples.
The variance of such a calculation results in many results in their experimental table being negative (note that MMD is a non-negative metric in theory).
Based on these considerations, we choose to a more commonly adopted exponential Hamming kernel with $0.1$ bandwidth in Table~\ref{tab:synthetic_mmd}.
Besides, we report the average of $10$ repeat results, each with $4000$ samples.
To make a fair comparison, we also report the results given by linear kernel MMD in Table \ref{tab:linear_mmd} which is also used in ALOE public code.
We can see that our algorithm keeps being state-of-the-art, and surpasses both PCD and the basic ALOE method on all datasets except \textit{swissroll}. 

One interesting property of the proposed EB-GFN framework, is that we can get samples either from the resulting GFlowNet (by sampling with the forward policy) or the learned EBM (by sampling with MCMC).
Theoretically, GFlowNet would benefit more from its inductive bias as we discussed in \S\ref{sec:introduction}, but we ideally want both models to achieve good performance.
To this end, we also track the performance of the learned EBM.
We find the learned EBM shares similar performance with the GFlowNet.
The comparison is shown in Table \ref{tab:ebm_of_gfn}.
On average, the EBM expresses slightly worse NLL and MMD than the GFlowNet, but is still very competitive if compared with other baseline methods.
This can also demonstrate the benefit of the GFlowNet prior.
As a result, we 
hypothesize
that 
\emph{the EB-GFN algorithm can achieve a good GFlowNet even with a not-so-good reward function}.
This is because GFlowNet only needs the reward function to be \emph{relatively} accurate with respect the true target distribution.
We point out an interesting analogy to this phenomenon in reinforcement learning: a policy can have good performance even when the agent has learned  a not-so-good Q function \citep{Sutton2005ReinforcementLA, Bengio2020InterferenceAG}.

\begin{table*}[t]
\setlength{\tabcolsep}{2.5mm}
\centering
\caption{
Experiment results of both the EBM and GFlowNet in EB-GFN framework with seven 2D synthetic problems.
We display the NLL and MMD with exponential kernel (in units of $1\times 10^{-4}$).
}
\label{tab:ebm_of_gfn}
\begin{tabular}{c|l|ccccccc}
\toprule
Metric & Method & 2spirals & 8gaussians & circles & moons & pinwheel & swissroll & checkerboard \\
\midrule
\multirow{2}{*}{NLL$\downarrow$} 
& EBM  & $20.061$ & $19.984$ & $20.568$ & $19.736$ & $19.574$ & $20.161$ & $20.683$ \\
& GFlowNet   &  $20.050$ & $19.982$ & $20.546$ & $19.732$ & $19.554$ & $20.146$ & $20.696$ \\
\midrule
\multirow{2}{*}{MMD$\downarrow$} & EBM & $0.585$ & $0.068$ & $0.160$ & $0.186$ & $1.298$ & $0.459$ & $9.138$ \\
& GFlowNet & $0.583$ & $0.531$ & $0.305$ & $0.121$ & $0.492$ & $0.274$ & $1.206$\\
\bottomrule
\end{tabular}
\end{table*}

\textbf{Ablation study on synthetic tasks.}
For completeness, we conduct ablation study to understand the importance of two features in EB-GFN algorithms:
(1) backward trajectory sampling in GFlowNets training distribution mentioned in \S\ref{sec:gfn_training}, and (2) the back-and-forth proposal proposed in \S\ref{sec:joint_training}.
We do experiments on \textit{checkerboard} and \textit{moons} tasks.
We first remove the usage of backward training samples and only use $\tau\sim P_F(\tau)$ to train the GFlowNet.
The GFlowNet NLL on \textit{moons} becomes $19.746$ from $19.732$, and Hamming exponential MMD becomes $0.342$ from $0.121$ (in units of $1\times 10^{-4}$).
For \textit{checkerboard}, the NLL becomes $20.709$ from $20.696$ and the MMD becomes $2.648$ from $1.206$ (in units of $1\times 10^{-4}$).
This shows removing the backward trajectory feature would only do little harm to the performance.
For the second part, once we remove the back-and-forth proposal and always use $K=D$, the training loss of EB-GFN quickly diverges on both tasks.
This indicates the suggested proposal trick is crucial to a reasonable optimization landscape.

\textbf{Understanding the learned backward policy.} In Fig.~\ref{fig:gfn_noise}, we give a visualization to show that the learned erasure policy $P_B(\cdot|\cdot;\vtheta)$ is meaningful. The GFlowNet-damaged samples have a clear visual structure that, interestingly, indicates that the several highest-magnitude bits in the Gray code are the first to be deleted by $P_B$ and, correspondingly, are the last to be generated by a forward policy that minimizes the trajectory balance loss jointly with this $P_B$.

\begin{figure*}[t]
\centering
\includegraphics[width=\textwidth]{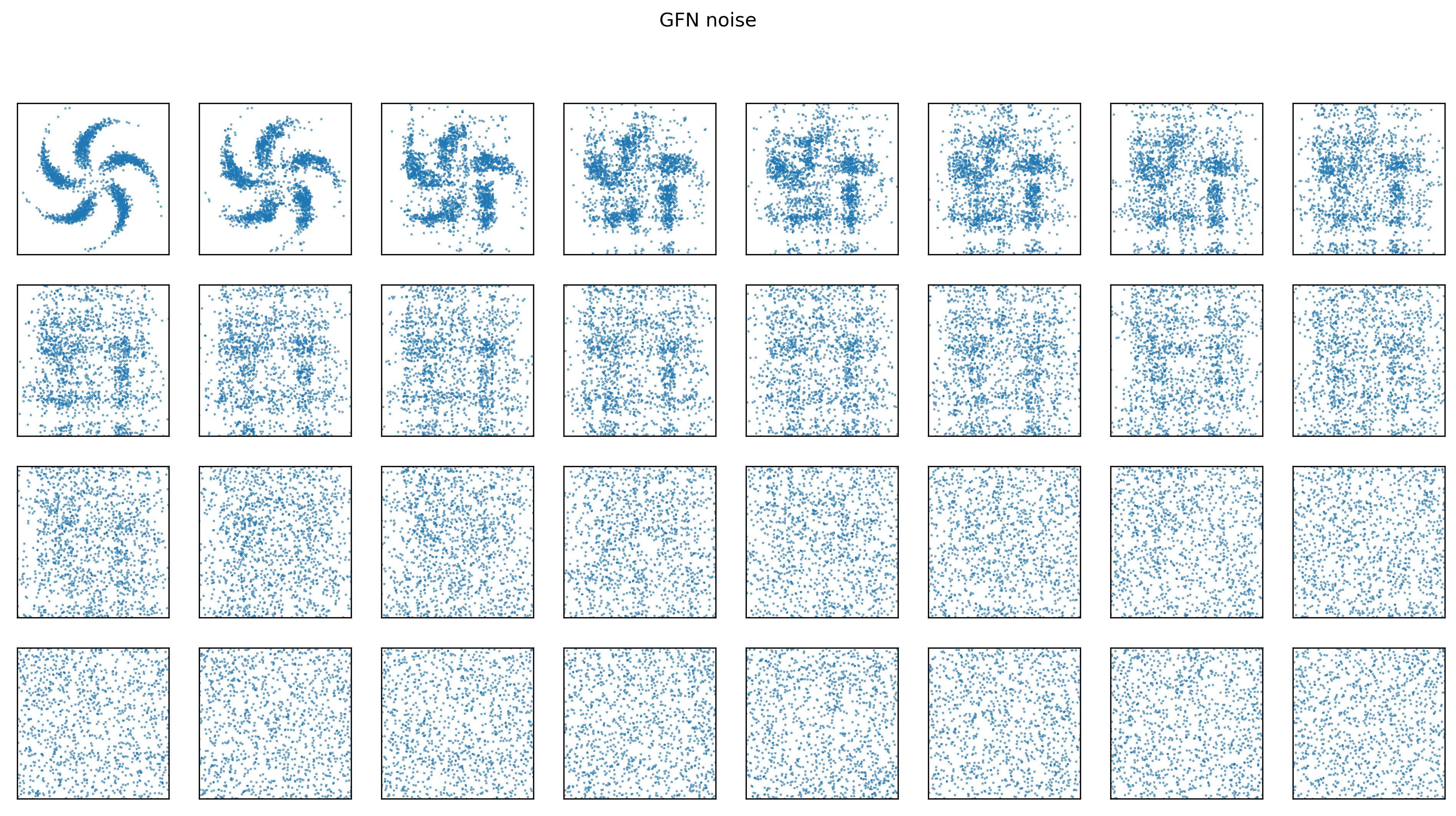}\\
\includegraphics[width=\textwidth]{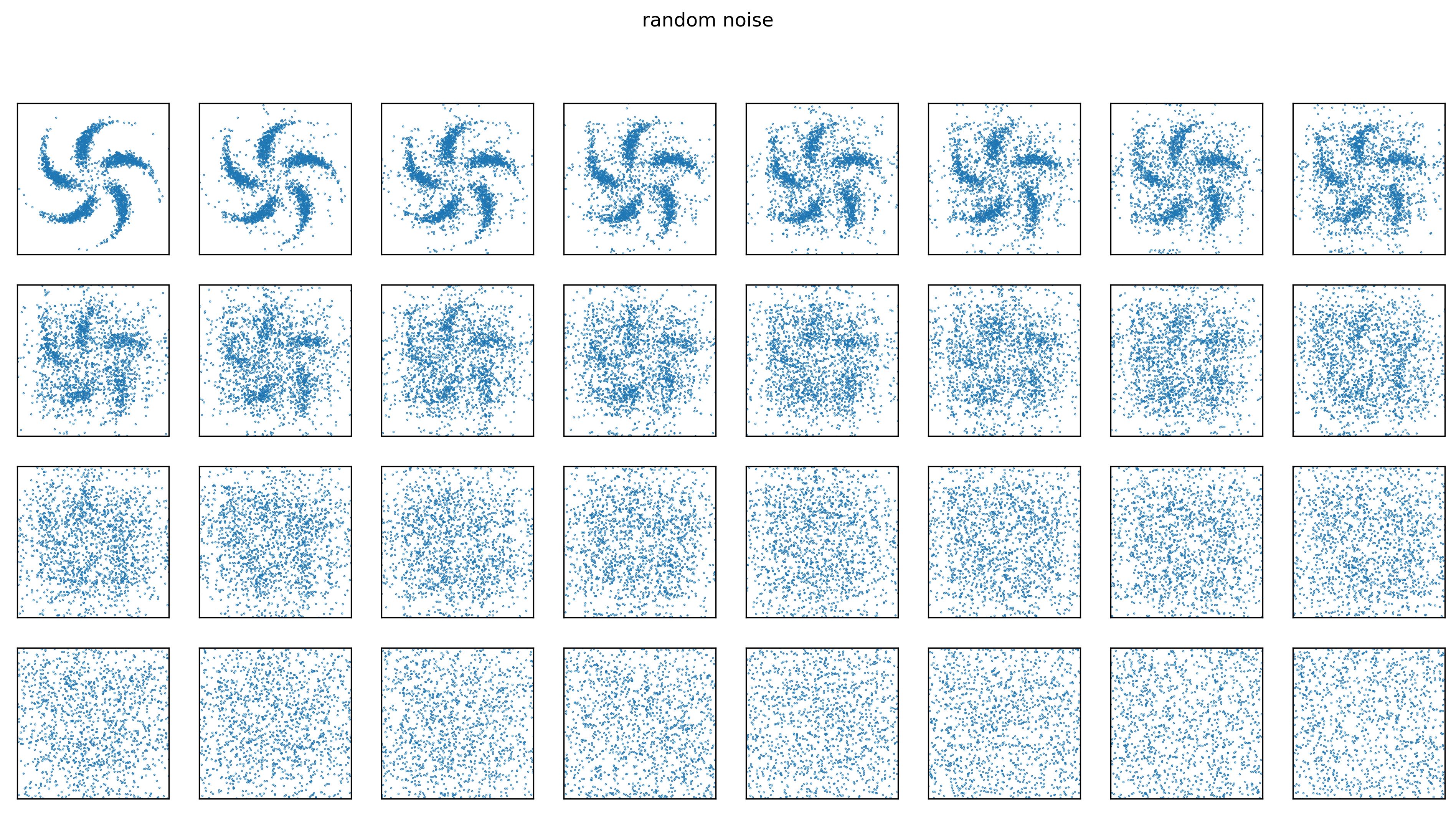}
\caption{We begin with real data from the \textit{pinwheel} dataset (top left) and take increasing numbers of steps (up to 31) from either the learned backward policy $P_B(\cdot|\cdot;\vtheta)$ of a GFlowNet \textit{(top)} or a uniform backward policy $P_B^\circ$ \textit{(bottom)}, then randomly (i.i.d. Bernoulli) pick new 0/1 values for the voided bits and visualize the resulting damaged samples.}
\label{fig:gfn_noise}
\end{figure*}

\subsection{Discrete image modeling}
\label{sec:appendix_image}

We explain the details of our discrete image modeling task here.
Discrete image modeling with EBM is a hard problem, and pure PCD training would diverge if the number of MCMC steps is not large enough or if there is no replay buffer trick to help training.
This is not the case with continuous circumstances \citep{Nijkamp2020OnTA}.
In this part, we follow the settings of \citet{Grathwohl2021OopsIT}, which are stated below.
We use Adam optimizer with $1\times 10^{-4}$ learning rate to update the energy function.
The batch size is set to be $100$ and the training lasts for $5\times 10^4$ steps.
The energy function is an MLP with $256$ hidden units and three hidden layers\footnote{We do not use the ResNet-18 backbone for EBM, because we find that it takes 2 weeks for the training of Gibbs-With-Gradients to finish with the original code of \href{https://github.com/wgrathwohl/GWG_release}{GWG public repo}.}. 
We do not use exponential moving average for simplicity.
The GFlowNet is optimized with an Adam optimizer, where the learning rate is $1\times 10^{-3}$.
The validation of EBM likelihood is achieved with $300000$ step annealed importance sampling (AIS).
GFlowNet is modelled as an MLP with three hidden layers and $512$ hidden units.
In this part we take the canonical design of backward policy, as we find that it could stabilize the training process.
We choose the checkpoint which has the best validation result, and report the corresponding test set performance. 
The same GFlowNet training techniques are utilized as in synthetic tasks:
we use an equal mix of $P_F(\tau)$ and $P_B(\tau|\x)$ to generate training trajectories for trajectory balance objective (\ie, set $\alpha=0.5$ in Algorithm~\ref{alg:gfn_training}).
For the back-and-forth proposal, we set $K$ to linearly increase from $1$ to $D$ through the training process.
We also do ablation study on these features in this task, and we get different results from \S\ref{sec:appendix_synthetic}.
To be precise, in the static mnist experiment, 
we find that these two techniques are both important:
EB-GFN would diverge without either trick.
This partially reflects the difficulty of this task.

We also find that introducing LayerNorm~\citep{Ba2016LayerN} into the forward policy network architecture is of great benefit to the generative modeling performance.
We add a LayerNorm after each linear layer except the last one.
The negative likelihood of this ablation is presented in the following table.

\begin{center}
\begin{tabular}{c|cccc}
\toprule
Method / Dataset & Omniglot & Silhouettes & Static MNIST & Dynamic MNIST \\
\midrule
\midrule
w/o LayerNorm & $112.59$ & $185.57$ & $102.43$ & $105.75$ \\
w/ LayerNorm & $104.88$ & $174.48$ & $89.48$ & $88.76$ \\
\bottomrule
\end{tabular}
\end{center}

\end{document}

%% file: notations.tex
\newcommand{\ie}{\emph{i.e.}}
\newcommand{\eg}{\emph{e.g.}}

\newcommand{\E}{\mathbb{E}}
\renewcommand{\P}{\mathbb{P}}

\newcommand{\R}{\mathbb{R}}
\def\gA{{\mathcal{A}}}

\def\gE{{\mathcal{E}}}
\def\gH{{\mathcal{H}}}
\def\gL{{\mathcal{L}}}
\def\gS{{\mathcal{S}}}
\def\gT{{\mathcal{T}}}
\def\gX{{\mathcal{X}}}
\def\x{{\mathbf{x}}}
\def\s{{\mathbf{s}}}

\def\vtheta{\boldsymbol{\theta}}
\def\vphi{\boldsymbol{\phi}}

\def\ra{{\rightarrow}}
\def\dra{{\dashrightarrow}}
\def\la{{\leftarrow}}